%% file: jcgs-template.tex
\documentclass[12pt,letterpaper]{article}
\pdfoutput=1

\usepackage[utf8]{inputenc}
\usepackage{times}
\usepackage{graphicx} 
\graphicspath{{figures/}}

\usepackage{amsmath}
\usepackage{amsfonts}
\usepackage{amssymb}
\usepackage{amsthm}
\usepackage{enumitem}
\usepackage{mathtools}
\setlist{nolistsep}
\mathtoolsset{showonlyrefs}

\newtheorem{theorem}{Theorem}
\newtheorem{lemma}{Lemma}
\newtheorem{proposition}{Proposition}
\newtheorem{corollary}{Corollary}
\newtheorem{remark}{Remark}

\usepackage{subcaption}
\usepackage[american]{babel}
\usepackage[babel]{csquotes}
\usepackage[natbib,maxbibnames=100,style=authoryear,backend=bibtex,isbn=false,doi=false]{biblatex}
\def\cite{\citep}
\usepackage{algorithm}
\usepackage{algorithmic}
\usepackage{float}
\usepackage{setspace}
\usepackage{hyperref}
\usepackage{titlesec}

\usepackage[final]{changes}

\usepackage[font=small]{caption}
\titlespacing{\paragraph}{%
  0pt}{
  0.5\baselineskip}{
  1em}
  
\DeclareMathOperator{\prox}{\bf prox}
\DeclareMathOperator{\gra}{\bf gra}
\DeclareMathOperator{\zer}{\bf zer}
\DeclareMathOperator{\Fix}{\bf Fix}
\DeclareMathOperator{\dom}{\bf dom}
\DeclareMathOperator*{\argmin}{arg\,min}

\DeclareMathOperator{\expect}{\mathbb{E}}
\DeclareMathOperator{\prob}{\mathbf{Pr}}

\newcommand{\appendixpagenumbering}{
  \break
  \pagenumbering{arabic}
  \renewcommand{\thepage}{S-\arabic{page}}
}

\allowdisplaybreaks

\usepackage{titlesec}
\titlespacing*{\subsection} {0pt}{3.25ex minus 2.5ex}{1.5ex minus 1.2ex}

\newcommand{\blind}{0}

\begin{document}

\def\spacingset#1{\renewcommand{\baselinestretch}%
{#1}\small\normalsize} \spacingset{1}


\if0\blind
{
  \title{\bf Easily parallelizable and distributable \\class of algorithms
		  for structured sparsity,\\ with optimal acceleration}
  \author{Seyoon Ko\\
    Department of Statistics, Seoul National University\\
    and \\
    Donghyeon Yu \\
    Department of Statistics, Inha University\\
    and \\
    Joong-Ho Won\thanks{To whom correspondences should be addressed. E-mail: \href{mailto:wonj@stats.snu.ac.kr}{wonj@stats.snu.ac.kr}
    }\hspace{.2cm} \\
    Department of Statistics, Seoul National University}
  \maketitle
} \fi

\if1\blind
{
  \bigskip
  \bigskip
  \bigskip
  \begin{center}
    {\LARGE\bf Title}
\end{center}
  \medskip
} \fi

\bigskip
\begin{abstract}
\input{sections/abstract}

\end{abstract}

\noindent%
{\it Keywords:}  monotone operator theory;
non-smooth optimization; operator splitting; sparsity; distributed computing; GPU 
\vfill
\newpage
\spacingset{1.45} 
\setlength{\belowdisplayskip}{8pt} \setlength{\belowdisplayshortskip}{8pt}
\setlength{\abovedisplayskip}{8pt} \setlength{\abovedisplayshortskip}{8pt}
\input{sections/intro}
\input{sections/unification}

\input{sections/rates}

\input{sections/acceleration}

\input{sections/rates2}

\input{sections/acceleration_stoc.tex}
\input{sections/rates_stoc.tex}

\input{sections/experiment}
\input{sections/numerical}
\input{sections/numerical_scale}
\input{sections/conclusion}
\subsubsection*{Supplementary material}
\vspace{-0.2cm}
The supplementary material contains an exposition of flexibility of formulation \eqref{eqn:primal} (Appendix \ref{sec:examples}), additional numerical experiments for stochastic optimal acceleration and the latent group lasso (Appendix \ref{sec:numerical:appendix}), a brief summary of monotone operator theory (Appendix \ref{sec:theory}), and the proofs of the theorems, propositions, and lemmas (Appendix \ref{sec:proofs}).





\spacingset{1}
\small
\AtNextBibliography{\fontsize{11}{12}\selectfont}
\printbibliography
\newpage

\appendix
\appendixpagenumbering
\begin{refsection}
\spacingset{1}
\fontsize{11}{12}\selectfont
\numberwithin{equation}{section}
\numberwithin{theorem}{section}
\numberwithin{lemma}{section}
\numberwithin{corollary}{section}
\numberwithin{proposition}{section}
\numberwithin{remark}{section}
\numberwithin{table}{section}
\numberwithin{figure}{section}
\begin{center}
\Large Supplementary Material \\ 
{\bf Easily parallelizable and distributable \\class of algorithms
for structured sparsity,\\ with optimal acceleration}\\
\large by Seyoon Ko, Donghyeon Yu, and Joong-Ho Won
\end{center}
\input{sections/examples.tex}
\input{sections/numerical_stoc.tex}

\input{sections/numerical_latent.tex}

\input{sections/prelim.tex}

\fontsize{10}{12}\selectfont
\input{sections/prelimproof.tex}

\input{sections/originalproof.tex}

\input{sections/optimalproof.tex}
\input{sections/stocproof.tex}

\AtNextBibliography{\fontsize{9}{12}\selectfont}
\printbibliography[heading=subbibliography]
\end{refsection}

%
\end{document}

%% file: sections/abstract.tex
Many statistical learning problems can be posed as minimization of a sum of two
convex functions, one typically a composition of non-smooth and linear
functions. Examples include regression under structured sparsity assumptions.
Popular algorithms for solving such problems, e.g., ADMM, often involve
non-trivial optimization subproblems or smoothing approximation. We consider
two classes of primal-dual algorithms that do not incur these difficulties,
and unify them from a perspective of monotone operator theory. From this
unification we propose a continuum of preconditioned forward-backward operator
splitting algorithms amenable to parallel and distributed computing. For the
entire region of convergence of the whole continuum of algorithms, we
establish its rates of convergence. For some known instances of this
continuum, our analysis closes the gap in theory. We further exploit the
unification to propose a continuum of accelerated algorithms. We show that the
whole continuum attains the theoretically optimal rate of convergence. 
The scalability of the proposed algorithms, as well as their convergence 
behavior, is demonstrated up to 1.2 million variables with a distributed implementation.

%% file: sections/intro.tex
\section{Introduction}\label{sec:intro}
Many statistical learning problems can be formulated as an optimization problem of the form
\begin{align}\label{eqn:primal}
		\min_{x \in \mathbb{R}^p} ~  f(x) + h(Kx),
\end{align}
where $K \in \mathbb{R}^{l \times p}$,
and both $f$ and $h$ are closed, proper, and convex.
In this paper, we assume $f$ is differentiable and its gradient $\nabla f$ is Lipschitz continuous with modulus $L_f$; 
$h$ is not necessarily smooth. 
Under this setting, we show how to solve \eqref{eqn:primal} in a 
fashion that is easy to parallelize or distribute on modern high-performance computing environment such as workstations equipped with multiple graphics processing units (GPUs).

A pinnacle instance of \eqref{eqn:primal} is high-dimensional penalized regression with structured sparsity: \deleted{assumptions, which solves the following optimization problem:}
\begin{align}\label{eqn:objective}
	\min_{x \in \mathbb{R}^p} \quad
\sum_{i=1}^n l_i(a_i^T x, b_i) + H(Dx),
\end{align}
with direct identification $f(x)=\sum_{i=1}^n l_i(a_i^Tx;b_i)$, $H(u)=h(u)$, and $K=D$,
where the set $\{(a_i, b_i): a_i \in \mathbb{R}^p, b_i \in \mathbb{R}, i = 1, \dots, n\}$ constitutes a training sample, $l_i: \mathbb{R}^2 \to \mathbb{R}$ is the loss function that may depend on the sample index, $D \in \mathbb{R}^{l \times p}$ is the structure-inducing matrix, and $H$ is the penalty function, which is typically non-smooth. 
Loss functions with Lipschitz gradients arise in many important problems:  
in linear regression we have $f(x)=(1/2)\|\mathsf{A}x-b\|_2^2$
and the gradient $\nabla f(x)=\mathsf{A}^T(\mathsf{A}x-b)$ is
$\|\mathsf{A}^T\mathsf{A}\|_2$-Lipschitz, 
where $\mathsf{A}=[a_1,\dotsc,a_n]^T$ denotes the data matrix
and $\|\mathsf{A}\|_2$ is the standard operator norm 
with respect to the vector $\ell_2$ norm $\|v\|_2$;
in logistic regression 
$f(x)=-\sum_{i=1}^n\big( b_i(a_i^Tx)+\log(1+e^{a_i^Tx}) \big)$ has 
$(1/4)\|\mathsf{A}^T\mathsf{A}\|_2$-Lipschitz gradients.
%
Choosing the $\ell_1$-penalty $H(z)=\lambda\|z\|_1$ for some $\lambda > 0$
yields the generalized lasso \citep{Tibshirani:TheAnnalsOfStatistics:2011},
which includes 
the fused lasso \cite{Tibshirani:JournalOfTheRoyalStatisticalSocietySeries:2005} as a special case.
For the group lasso \cite{yuan2006model} with $\mathcal{G}$ possibly \emph{overlapping} groups, we can choose
$H(y) = \lambda_1\|y_{[1]}\|_q + \dotsb + \lambda_{\mathcal{G}}\|y_{[\mathcal{G}]}\|_q$ 
for
$y = (y_{[1]}^T, \dotsc, y_{[\mathcal{G}]}^T)^T$, 
where 
$[g] \subset \{1,2,\dotsc,p\}$ is a given set of group indexes
and
$y_{[g]} \in \mathbb{R}^{|[g]|}$ 
for each $g=1,2,\dotsc,\mathcal{G}$;
$\|\cdot\|_q$ denotes the $\ell_q$ norm with $q>1$.
Now set $D$ as a $(|[1]|+\dotsb+|[\mathcal{G}]|)\times p$ binary matrix with a single one (\textit{1}) in each row; the 1 corresponds to the group membership. 
Then,
	$H(D x ) \allowbreak= \lambda_1 \| x_{[1]} \|_q + \dotsb + \lambda_{\mathcal{G}}\|x_{[\mathcal{G}]}\|_q$
as desired;
$D$ has a column with more than a single nonzero entry if and only if there is an overlapping group.
Judicious choices of $f$, $h$, and $K$ in \eqref{eqn:primal} allow more flexibility in solving \eqref{eqn:objective}.
In particular, non-smooth loss functions, such as the hinge loss, can also be handled.  
More complex penalty functions such as the latent group lasso \citep{jacob2009group} are also allowed in \eqref{eqn:objective}
(See Appendix \ref{sec:examples} for details). 
Therefore ability to solve \eqref{eqn:primal} efficiently provides a versatile tool for many important statistical learning problems. 


In spite of its importance, solving \eqref{eqn:primal} is challenging 
because the non-separability of the non-smooth part hampers use of efficient methods. 
If $K=I$ and $h$ is separable, e.g., $h(y)=\lambda\|y\|_1$, then the proximal gradient method \citep{combettes2005signal} is arguably the method of choice, which provides a simple gradient-descent-like iteration
\[
x^{k+1} = \argmin_{x} f(x^{k}) + \langle \nabla f(x^{k}), x-x^{k} \rangle + \frac{1}{2t}\|x-x^{k}\|_2^2 + h(x)
= \prox_{th}(x^{k}-t\nabla f(x^{k})) 
\]
for $0 < t < 2/L_f$,
where $\prox_{\phi}(z) \allowbreak := \allowbreak \argmin_{z' \in \mathbb{R}^n} \phi(z') + \frac{1}{2}\| z' - z \|_2^2$ is the proximity operator for a convex function $\phi$;
$\langle u,v \rangle$ denotes the standard inner product $u^Tv$.
If $h(y)=\lambda\|y\|_1$, then $\prox_{th}$ is an element-wise soft-thresholding operator \citep{Beck:SiamJournalOnImagingSciences:2009}. 
However, for general $K$ and other choices of $h$, e.g., group lasso, 
proximal gradient involves
evaluating $\prox_{t h\circ K}(\cdot)$, 
which is nontrivial even for tractable cases
\citep{Friedman:TheAnnalsOfAppliedStatistics:2007,liu2010efficient,xin2014efficient,yu2015high}.
While approximating $h$ by a smooth function has been considered \cite{Nesterov:MathematicalProgramming:2004,Chen:TheAnnalsOfAppliedStatistics:2012},
this approach introduces an additional smoothing parameter that is difficult to choose in practice.
The popular alternating directions method of multipliers \citep[ADMM; see, e.g.,][]{Boyd:FoundationsAndTrendsInMachineLearning:2010} can be applied to solve \eqref{eqn:primal} as well, which yields an iteration
\begin{subequations}\label{eqn:ADMM}
\begin{align}
	x^{k+1} &= \argmin_x f(x) + (t/2)\|Kx-\tilde{x}^{k}+(1/t)y^{k}\|_2^2 \label{eqn:ADMM:x} \\
	\tilde{x}^{k+1} &= \prox_{(1/t)h}(K x^{k+1}+(1/t)y^{k})\label{eqn:ADMM:b}\noeqref{eqn:ADMM:b} \\
	y^{k+1} &= y^{k} + t (Kx^{k+1}-\tilde{x}^{k+1})\label{eqn:ADMM:c}\noeqref{eqn:ADMM:c}
\end{align}
\end{subequations}
The $x$-update \eqref{eqn:ADMM:x} is an \emph{inner minimization subproblem} and is potentially expensive to compute. For example, if $f$ is a loss function for a generalized linear model, 
then the corresponding update involves solving a linear equation of the form $(\mathsf{A}^TW\mathsf{A} + t K^TK)x=r$, $W$ diagonal, \emph{iteratively}.
While $K$ is structured and known \textit{a priori}, the data matrix $\mathsf{A}$ is hardly structured.
A similar problem arises in medical imaging reconstruction problems, such as undersampled multi-coil MRI reconstruction \citep{nien2015fast} or sparse-view CT reconstruction \citep{sidky2012convex} using the total variation penalty \cite{rudin1992nonlinear,goldstein2009split}. 
In this case the ``measurement matrix'' $\mathsf{A}$ is large and unstructured.
Hence avoiding inner minimization subproblem
is crucial in both statistical learning and imaging problems where the problem dimensions are ever increasing. 
Primal-dual hybrid gradient method \citep[PDHG;][]{zhu2008efficient,Esser:SiamJournalOnImagingSciences:2010,chambolle2011first,He:SiamJournalOnImagingSciences:2012,Chambolle:MathematicalProgramming:2015,zhu2015augmented} and linearized alternating directions method \citep[LADM;][]{lin2011linearized} add an additional regularization term to \eqref{eqn:ADMM:x} in order to avoid the costly inner minimization subproblem.
However, these methods often involve evaluating $\prox_f(\cdot)$, 
which may lead to another inner minimization subproblem in the presence of
$\mathsf{A}$.

The goal of this paper is to introduce to the statistical community a class of algorithms that does require neither smoothing nor quadratic minimization. This class of algorithms only involve evaluation of the gradient $\nabla f(x)$, matrix-vector multiplications and simple proximity operators. Thus it is simple to implement and attractive for parallel and distributed computation.
We begin with introducing two known algorithms. 
One is due to  \citet{loris2011generalization}, later studied by \citet{Chen:InverseProblems:2013}, and \citet{drori2015simple}: 
\begin{align*}\tag{Algorithm LV}
\begin{split}\label{eqn:LV}\noeqref{eqn:LV}
    \tilde{x}^{k+1} &= x^k - \tau \left(\nabla f(x^k) + K^T y^k\right) \\
	y^{k+1} &= (1-\rho_k) y^k + \rho_k \prox_{\sigma h^{*}}(y^k + \sigma K \tilde{x}^{k+1}) \\
	x^{k+1} &= (1-\rho_k) x^k + \rho_k (\tilde{x}^{k+1} - \tau K^T(y^{k+1}-y^{k})),
\end{split}
\end{align*}
and the other is due to \citet{Condat:JournalOfOptimizationTheoryAndApplications:2012} and \citet{Vu2013}:
\begin{align*}\tag{Algorithm CV}
\begin{split}\label{eqn:CV}\noeqref{eqn:CV}
    \bar{x}^{k+1} &= x^k - \tau(\nabla f (x^k) + K^T y^k)\\
    \tilde{x}^{k+1} &= 2x^{k+1} - \bar{x}^{k+1} \\
    x^{k+1} &= (1-\rho_k) x^k + \rho_k \bar{x}^{k+1} \\
    y^{k+1} &= (1-\rho_k) y^k + \rho_k \prox_{\sigma h^{*}} (y^k + \sigma K \tilde{x}^{k+1}),
\end{split}
\end{align*}
where
$h^{*}(v) = \sup_{u \in \mathbb{R}^l} \langle u, v \rangle - h(u)$
is the convex conjugate of $h$.
Choices of the sequence $\{\rho_k\}$ and the step size parameters $(\sigma,\tau)$ for convergence of these algorithms are discussed in Section \ref{sec:unify}.
As can be seen, 
the proximity operator employed by both algorithms depends only on $h^*$ but not $K$.
Moreover, $\prox_{\sigma h^*}(\cdot)$ can be evaluated by using Moreau's decomposition
$\prox_{\sigma h^*}(y) = y - \sigma \prox_{\sigma^{-1}h} (\sigma^{-1}y)$.
Thus they are simple to implement and attractive for parallel and distributed computation
as long as either $\prox_{h^*}(\cdot)$ or $\prox_{h}(\cdot)$ is simple (``proximable'').
Table \ref{tab:separable} illustrates the proximity operators for
popular choices of $h$. 
Once the conditions for convergence is understood, the rate of convergence and acceleration of the algorithm are the next interest.

%

In this regard,
the contributions of this paper are as follows.
First, we connect Algorithms LV and CV from a perspective of monotone operator theory to show that they are essentially the \emph{same} preconditioned forward-backward splitting algorithm \citep[see, e.g.,][]{combettes2005signal} sharing a common preconditioner. 
Second, from this connection we propose a new, broader family of preconditioners that generates an entire \emph{continuum} of forward-backward algorithms.
Third, by a unified analysis, we show that this continuum of algorithms enjoys common ergodic and non-ergodic rates of convergence over the entire region of convergence. 
Prior to our connection the rates of the above two algorithms have been available under much more stringent conditions than that for convergence; we close this gap. 
Fourth, we proceed further to \emph{accelerate the whole continuum of algorithms to achieve the theoretically optimal rate of convergence.} 
Only an optimal acceleration of Algorithm CV has been known \citep{chen2014optimal}, and acceleration of LV has remained an open problem. 
Finally, we demonstrate the scalability of the studied algorithms by implementing them on a distributed computing environment in case that data do not fit in the memory of a single device.

\paragraph{\textit{Organization.}}
In Section \ref{sec:unify}, we examine the relation between Algorithms LV and CV and unify them to propose a broader class of algorithms. The rates of convergence of this class of algorithms is also analyzed. 
In Section \ref{sec:optimal}, we develop an accelerated variant of the new class of algorithms achieving the optimal rate. 
Its stochastic counterpart, also possessing the optimal rate, is discussed in Section \ref{sec:stoc}. 
Section \ref{sec:numerical} demonstrates the convergence behavior and scalability of the new algorithms through their multi-GPU implementations. 
Discussion and conclusion follow thereafter in Section \ref{sec:conclusion}. 
All the proofs of our results can be found in the supplementary material. 


\paragraph{\textit{Notation.}}
That a symmetric matrix $M$ is positive (semi)definite is denoted by $M \succ 0$ ($M \succeq 0$); $L \succ M$ refers to $L-M \succ 0$, etc.
For $M \succ 0$, we define its associated inner product and norm by $\langle x,x' \rangle_M=\langle Mx,x'\rangle$ and $\|x\|_M=\sqrt{\langle x,x\rangle_M}$, respectively.
For a symmetric matrix $M$, $\lambda_{\max}(M)$ and $\lambda_{\min}(M)$ respectively denote the maximum and minimum eigenvalues.

\begin{table}[!htb]
\caption{Convex conjugates and proximity operators for selected choices of $h$.
Function $\delta_S$ denotes the indicator function for set $S$ so that $\delta_S(u)=0$ if $u\in S$ and $\delta_S(u)=+\infty$ otherwise;
$P_S$ denotes the projection onto set $S$, which is unique if $S$ is closed and convex;
$\sigma_j(M)$ denotes the $j$th largest singular value of matrix $M$.
All $\min$, $\max$ operations are elementwise. 
In $\ell_{1,q}$-norm, ${1}/{q}+{1}/{s}=1$.
}
\label{tab:separable}
\medskip
\centering
\footnotesize
\begin{tabular}{llll} \hline
Name  & $ h(y) $ & $h^*(z)$ & ${\bf prox}_{h^*}(z)$  \\ \hline
$\ell_1$-norm & $\lambda \|y\|_1$ & $\delta_{\mathcal{B}_{\infty}}(z)$, $\mathcal{B}_{\infty}=\{z:\|z\|_{\infty}\le\lambda\}$  & $\min\{\max\{z, -\lambda\}, \lambda\}$ \\
$\ell_2$-norm & $\lambda \|y\|_\infty$ & $\delta_{\mathcal{B}_2}(z)$, $\mathcal{B}_2=\{z:\|z\|_2 \le \lambda\}$ & $P_{\mathcal{B}_2}(z)$ \\ 
$\ell_\infty$-norm & $\lambda \|y\|_\infty$ & $\delta_{\mathcal{B}_1}(z)$, $\mathcal{B}_1=\{z:\|z\|_1 \le \lambda\}$ & $P_{\mathcal{B}_1}(z)$ \\ 
$\ell_{1,q}$-norm &
$\sum_{g=1}^{\mathcal{G}} \lambda_g \|y_{[g]}\|_q$ 
& $\delta_{\mathcal{B}_s^1\times\dotsb \times\mathcal{B}_s^{\mathcal{G}}}(z)$, $\mathcal{B}_s^g=\{z:\|z_{[g]}\|_s \le \lambda_g\}$
& $\big(P_{\mathcal{B}_s^1}(z_{[1]}), \dotsc, P_{\mathcal{B}_s^{\mathcal{G}}}(z_{[\mathcal{G}]})\big)$
\\ 
nuclear norm & $\lambda \sum_{i=1}^{\text{rank}{(Y)}} \sigma_i (Y)$ & 
$\delta_{\mathcal{B}_{*}}(Z)$, $\mathcal{B}_{*}=\{Z:\|Z\|_2\le\lambda\}$ &
$U\min\{\Sigma,\lambda I\}V^T$, $Z=U\Sigma V^T$ \\
hinge loss & $\sum_{i=1}^l\max\{1 - y_i, 0\}$ & $\sum_{i=1}^l\big(z_i - \delta_{[0,1]}(-z_i)\big)$ & $\min\{z+1, \max\{z, 1\} \}$ 
\\
\hline
\end{tabular} 
\end{table}

%% file: sections/unification.tex
\section{Unification}\label{sec:unify}
In this section we provide a unified treatment to Algorithms LV and CV from the perspective of monotone operator theory. 
For a brief summary of monotone operator theory, see Appendix \ref{sec:theory}.

\subsection{Relation between Algorithms LV and CV}\label{sec:unify:relation}
It can be shown that both Algorithms LV and CV are instances of preconditioned forward-backward splitting. 
To be specific, 
note the first-order optimality condition for \eqref{eqn:primal} is given by
\begin{subequations}\label{eqn:optimality}
\begin{align}
	0 &= \nabla f(x^{\star}) + K^T y^{\star}, \label{eqn:optimality1}\noeqref{eqn:optimality1} \\
	y^{\star} &\in \partial h(Kx^{\star}). \label{eqn:optimality2}
\end{align}
\end{subequations}
where 
$\partial h(y) = \{ g \in \mathbb{R}^l: h(y') \ge h(y) + \langle g, y'-y \rangle, ~ \forall y' \in \mathbb{R}^l \}$
is the subdifferential of the convex function $h$ at $y$, which is a set-valued operator.
Since $h$ is closed and proper, condition \eqref{eqn:optimality2} is equivalent to $Kx^{\star} \in (\partial h)^{-1}(y^{\star}) = \partial h^{*}(y^{\star})$ \citep{Bertsekas:ConvexOptimizationTheory:2009},
thus \eqref{eqn:optimality} can be equivalently written as an inclusion problem
\begin{align}\label{eqn:inclusion}
	 \begin{bmatrix} 0 \\ 0 \end{bmatrix} \in
	\begin{bmatrix} \nabla f  & K^T \\ -K & \partial h^{*} \end{bmatrix}
	\begin{bmatrix} x^{\star} \\ y^{\star} \end{bmatrix}
	=: T(z^{\star}),
	\quad
	z^{\star} = (x^{\star},y^{\star}).
\end{align}
Under a mild condition 
\citep[Theorem 19.1 and Proposition 19.18]{Bauschke:ConvexAnalysisAndMonotoneOperatorTheoryIn:2011}; see also \citet{Condat:JournalOfOptimizationTheoryAndApplications:2012},
\eqref{eqn:inclusion} has a solution. 
If $(x^\star,y^\star)$ is solution, then it is a saddle point for 
the saddle point formulation of \eqref{eqn:primal}:
\begin{align}\label{eqn:saddlepoint}
		\min_{x\in\mathbb{R}^p} \max_{y\in\mathbb{R}^l} \mathcal{L}(x,y)
\end{align}
where
$\mathcal{L}(x,y) = f(x) + \langle Kx, y \rangle - h^{*}(y)$
is the saddle function.
Also the strong duality holds: 
$x^\star$ is a primal solution to \eqref{eqn:primal}, 
and $y^\star$ is a solution to the associated dual 
\begin{align}\label{eqn:dual}
	\max_{y\in\mathcal{Y}}\left(-f^*(-K^Ty)-h^*(y)\right)
\end{align}
\citep[Theorem 19.1 and Proposition 19.18]{Bauschke:ConvexAnalysisAndMonotoneOperatorTheoryIn:2011}; see also \citet{Condat:JournalOfOptimizationTheoryAndApplications:2012}.
In the sequel, we assume that \eqref{eqn:inclusion} has a solution.

The set-valued operator $T$ is split into $T=F+G$, where
\begin{align}\label{eqn:operatorsplit}
	F = \begin{bmatrix} 0 & K^T \\ -K & \partial h^{*} \end{bmatrix}
	\quad \text{and} \quad
	G = \begin{bmatrix} \nabla f & 0 \\ 0 & 0 \end{bmatrix}.
\end{align}
The operator $F$ is maximally monotone and $G$ is $1/L_f$-cocoercive
\citep{Bauschke:ConvexAnalysisAndMonotoneOperatorTheoryIn:2011}.
A preconditioned forward-backward splitting for solving \eqref{eqn:inclusion} is
\begin{align}\label{eqn:preconditionedforwardbackward_relax}
\begin{split}
	\tilde{z}^{k} &= (I+M^{-1}F)^{-1}(I-M^{-1}G)(z^k) \\
    z^{k+1} &= (1-\rho_k)z^k + \rho_k \tilde{z}^{k},
\end{split}
\end{align}
for $z^k=(x^k,y^k)$, $\tilde{z}^k=(\tilde{x}^k,\tilde{y}^k)$, and $M \succ 0$.
If the modulus of cocoercivity of $M^{-1}G$ denoted by $\gamma$
\citep[cocoercivity of $G$ is preserved; see][]{Davis:SiamJOptim:2015},
then \eqref{eqn:preconditionedforwardbackward_relax}
converges if $\gamma>1/2$ and for a sequence $\{\rho_k\}\subset[0,\delta]$ such that $\sum_{k=0}^{\infty}\rho_k(\delta-\rho_k)=\infty$ with $\delta \allowbreak = \allowbreak 2-1/(2\gamma)$. 
Note $\rho_k\equiv 1$ is allowed which yields a simple iteration $z^{k+1}=(I+M^{-1}F)^{-1}(I-M^{-1}G)z^k$. 
The inverse operator $(I+M^{-1}F)^{-1}$ is single-valued due to maximal monotonicity of $M^{-1}F$ \citep[Theorems 25.8 and 24.5]{Bauschke:ConvexAnalysisAndMonotoneOperatorTheoryIn:2011}.
(For instance, $(I+\partial \phi)^{-1}(z) = \argmin_{z' \in \mathbb{R}^n} \phi(z') + \frac{1}{2}\| z' - z \|_2^2= \prox_{\phi}(z)$.)
In particular,
the preconditioners for Algorithms LV and CV are respectively given by
\citet{combettes2014forward, Condat:JournalOfOptimizationTheoryAndApplications:2012, Vu2013}:
\[
	M = M_{\mathsf{LV}} := \begin{bmatrix} \frac{1}{\tau}I &  \\  & \frac{1}{\sigma}I-\tau KK^T \end{bmatrix}
	\quad \text{and} \quad
	M = M_{\mathsf{CV}} := \begin{bmatrix} \frac{1}{\tau}I & -K^T \\ -K  & \frac{1}{\sigma}I \end{bmatrix}
	.
\]

Now we are ready to see that Algorithms LV and CV are essentially the same algorithm.
The ``LDL'' decomposition of $M_{\mathsf{CV}}$ reveals that
\begin{align*}
	M_{\mathsf{CV}} 
	&= \begin{bmatrix} I &   \\ -\tau K & I \end{bmatrix}
	  \begin{bmatrix} \frac{1}{\tau}I &  \\  & \frac{1}{\sigma}I-\tau KK^T \end{bmatrix}
	  \begin{bmatrix} I & -\tau K^T  \\  & I \end{bmatrix} 
	= L M_{\mathsf{LV}} L^T .  \stepcounter{equation}\tag{\theequation}\label{eqn:M1M2}
\end{align*}
It is clear both $M_{\sf LV}$ and $M_{\sf CV}$ are positive definite if and only if $1/(\tau\sigma) > \|K\|_2^2$.
Also it is easy to see that Algorithm CV, i.e., \eqref{eqn:preconditionedforwardbackward_relax} with $M=M_{\mathsf{CV}}$,
is equivalent to
\begin{align}\label{eqn:CPasLV}
	 L^T z^{k+1} = (1-\rho_k)L^Tz^k + \rho_k(I + M_{\mathsf{LV}}^{-1}\tilde{F})^{-1}(I - M_{\mathsf{LV}}^{-1}\tilde{G})(L^Tz^k),
\end{align}
where $\tilde{F} = L^{-1}FL^{-T}$ and $\tilde{G} = L^{-1} G L^{-T}$.
Letting $w=L^Tz$, \emph{we see that Algorithm CV is in fact Algorithm LV applied to the linearly transformed variable $w$ by
splitting the similarly transformed operator $L^{-1}TL^{-T}$ into $\tilde{F}$ and $\tilde{G}$.}
The cocoercivity constant of $M_{\sf{LV}}^{-1}\tilde{G}$ is found by the following proposition.
\begin{proposition}\label{prop:Gtildecocoercive}
 $M_{\mathsf{LV}}^{-1}\tilde{G}$ is $(1/\tau-\sigma\|K\|_2^2)/L_f$-cocoercive with respect to $\|\cdot\|_{M_{\mathsf{LV}}}$.
\end{proposition}
Thus from the discussion below \eqref{eqn:preconditionedforwardbackward_relax}
we have
$\gamma=(1/\tau-\sigma\|K\|_2^2)/L_f$
and
$\delta = 2 - \frac{L_f}{2}\cdot\frac{1}{1/\tau - \sigma\|K\|_2^2}$.
Then Algorithm CV converges if
\begin{align}\label{eqn:CPparamrange2}
 	\frac{1}{\tau} > \frac{L_f}{2}
 	\quad \text{and} \quad
 	\left(\frac{1}{\tau}-\frac{L_f}{2}\right)\frac{1}{\sigma} > \|K\|_2^2
\end{align}
With respect to the untransformed sequence $\{z^{k}\}$,
 observe that $M_{\mathsf{CV}}^{-1}G$ is also $(1/\tau-\sigma\|K\|_2^2)/L_f$-cocoercive (with respect to $\|\cdot\|_{M_{\mathsf{CV}}}$).
 In light of \eqref{eqn:CPasLV}, it is natural to measure convergence using the metric $\|L^T\cdot\|_{M_{\mathsf{LV}}}$, 
 and this metric coincides with $\|\cdot\|_{M_{\mathsf{CV}}}$.
On the other hand, it is easy to see $M_{\sf LV}^{-1}G$ is $1/(\tau L_f)$-cocoercive with respect to $\|\cdot\|_{M_{\sf LV}}$, hence Algorithm LV has $\gamma=1/(\tau L_f)$ and $\delta=2-\tau L_f/2$. 
It converges if
\begin{align}\label{eqn:LVparamrange}
	1/\tau > L_f/2 \quad \text{and} \quad 1/(\tau\sigma) > \|K\|_2^2.
\end{align}
Both \eqref{eqn:CPparamrange2} and \eqref{eqn:LVparamrange} recover the known convergence regions in the literature \citep{Condat:JournalOfOptimizationTheoryAndApplications:2012,Chen:InverseProblems:2013}.

\subsection{Unified algorithm class}
The relation between the two algorithms suggests a more general family of preconditioners, namely
\begin{align}\label{eqn:generalM}
	M = \tilde{L} M_{\mathsf{LV}} \tilde{L}^T = \arraycolsep=4pt\begin{bmatrix} \frac{1}{\tau}I & C^T \\ C & \frac{1}{\sigma} I + \tau (C C^T - K K^T)\end{bmatrix},
\end{align}
where $\tilde{L}$ replaces $(2,1)$ block of $L$ in \eqref{eqn:M1M2} by $\tau C$.
In particular, if $CK^T=KC^T$, then 
\eqref{eqn:preconditionedforwardbackward_relax} yields the following iteration (for simplicity we set $\rho_k\equiv 1$):
\begin{align}\label{eqn:generaliteration}
\begin{split}
	y^{k+1} &= \prox_{\sigma h^{*}}( \sigma K x^{k} + \sigma\tau(C-K)\nabla f(x^{k}) + (I+\sigma \tau K(C-K)^T)y^{k}) \\
	x^{k+1} &= x^{k} - \tau(\nabla f(x^{k}) - C^Ty^{k} + (C+K)^T y^{k+1} ). 
\end{split}
\end{align}
Condition $CK^T=KC^T$ is satisfied 
if and only if $C=US\Sigma^{-1}V^T + N\bar{V}^T$, where $U$, $V$, and $\Sigma$ are from the reduced singular value decomposition of $K=U\Sigma V^T$ so that $\Sigma$ is an $r \times r$ positive diagonal matrix where $r=\textrm{rank}(K)$;
$\bar{V}$ is 
such that $\tilde{V}=[V, \bar{V}]$ is orthogonal;
$S$ is symmetric, and $N$ is arbitrary.
A simple choice is $S=\kappa \Sigma^2$ for some $\kappa \in \mathbb{R}$ and $N=0$, yielding $C=\kappa K$.
Choosing $\kappa=0$ and $-1$ respectively recovers Algorithms LV and CV; for $\kappa=1$, we have
\begin{align*}
	y^{k+1} &= \prox_{\sigma h^{*}}( \sigma K x^{k} + y^{k} ) \\
	x^{k+1} &= x^{k} - \tau\nabla f(x^{k}) - \tau K^T(2y^{k+1}-y^{k}),
\end{align*}
which is the dual version of Algorithm CV \citep[][Algorithm 3.2]{Condat:JournalOfOptimizationTheoryAndApplications:2012}.
Another choice is to set $S=\pm\Sigma^2$ and $N$ so that $NN^T$ is diagonal.
In this case $CC^T-KK^T$ reduces to a diagonal matrix, $C=[\bar{K}, N]\tilde{V}$ where $\bar{K}$ is the first $r$ columns of $K\tilde{V}$.
If the eigenspace of $K^TK$ is well-known and multiplication with $\bar{V}$ can be computed fast, e.g., the discrete cosine transform matrix for the fused lasso on a regular grid \cite{lee2017large}, 
this choice can be useful.

\subsection{Convergence analysis}

\subsubsection*{Region of convergence}
A condition for \eqref{eqn:preconditionedforwardbackward_relax} with general $M$ to converge is 
\begin{align}\label{eqn:fullrange}
M \succ \begin{bmatrix} \frac{L_f}{2}I &  \\  & 0 \end{bmatrix},
\end{align}
which follows from Theorem \ref{thm:nonergodic} and Proposition \ref{prop:Mcondition} later in this section. Thus with $M$ in \eqref{eqn:generalM} the following region of convergence is obtained.
\begin{proposition}\label{prop:unifiedconvergence}
Algorithm \eqref{eqn:generaliteration} converges for $(\sigma,\tau)$ such that
\begin{align}\label{eqn:generalparamrange}
	\frac{1}{\tau} > \frac{L_f}{2} 
	\quad \text{and} \quad
	\left(\frac{1}{\tau}-\frac{L_f}{2}\right)\left(\frac{1}{\sigma} - \tau \|K\|_2^2\right) > \frac{\tau L_f}{2}\|C\|_2^2.
\end{align}
\end{proposition}
Note that \eqref{eqn:generalparamrange} reduces to  \eqref{eqn:LVparamrange} for Algorithm LV and to \eqref{eqn:CPparamrange2} for CV.
In general for $C=\kappa K$, $\kappa \in [-1,1]$,
the region of convergence shrinks gradually from $|\kappa|=0$ (LV) to $1$ (CV); see Figure \ref{fig:paramrange_}.
This extends the observation made in Section \ref{sec:unify:relation} regarding convergence conditions \eqref{eqn:LVparamrange} and \eqref{eqn:CPparamrange2} to a continuum of algorithms between LV and CV.

\begin{figure}[h]
\vskip -0.1in
\begin{center}
\includegraphics[width=0.35\columnwidth]{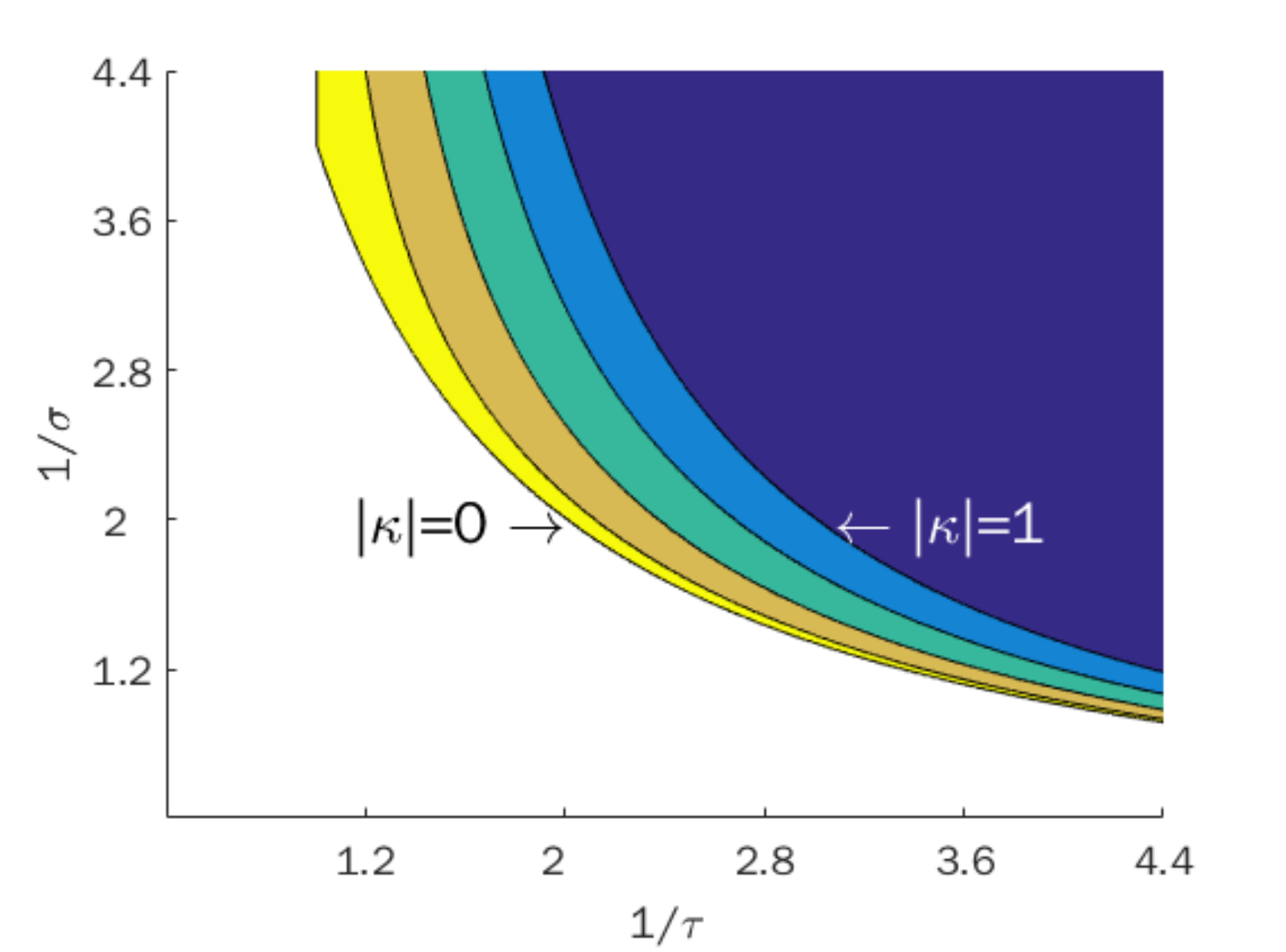}
\caption{Region of convergence in $(1/\sigma, 1/\tau)$. Boundaries correspond to $|\kappa|$ = 0, 0.25, 0.5, 0.75, 1.}
\label{fig:paramrange_}
\end{center}
\vskip -0.1in
\end{figure}

\begin{remark}\label{remark:CV+}
	\citet{Condat:JournalOfOptimizationTheoryAndApplications:2012} also considers an extension of \eqref{eqn:primal}, which minimizes the three-function sum $f(x)+g(x)+h(Kx)$, with $g$ convex closed proper (not necessarily smooth). In this case, the second term of the first line of Algorithm CV is replaced by $\prox_{\tau g}(x^{k}-\tau(\nabla f(x^k)+K^Ty^k))$.
	We call this extension Algorithm CV+. This algorithm is still a preconditioned forward-backward splitting one with preconditioner $M_{\mathsf{CV}}$, where the zero in the (1,1) block of operator $F$ is replaced by $\partial g$, and converges under \eqref{eqn:CPparamrange2}. For this extended $F$, \eqref{eqn:generaliteration} generates a feasible algorithm only when $C=\pm K$, i.e., Algorithm CV+ or its dual.
   Nevertheless, for Algorithm LV, there is a three-function extension \cite{Chen:FixedPointTheoryAndApplications:2016}. 
\end{remark}


%% file: sections/rates.tex
\subsubsection*{Rates of convergence}

We now analyze the rates of convergence of the preconditioned forward-backward splitting algorithm \eqref{eqn:preconditionedforwardbackward_relax}
for the preconditioner matrices $M$ of \eqref{eqn:generalM}. 
A pre-duality gap function $\mathcal{G}(\tilde{z},z):=\mathcal{L}(\tilde{x},y)-\mathcal{L}(x,\tilde{y})$, where $z=(x,y)$ and $\tilde{z}=(\tilde{x},\tilde{y})$, is used to measure the convergence of
the objective value, because the duality gap $\mathcal{G}^{\star}(\tilde{z}):=\sup_{z\in Z}\mathcal{G}(\tilde{z},z)$, $Z\subset\mathbb{R}^p\times\mathbb{R}^l$, guarantees that the pair  $\tilde{z}=(\tilde{x},\tilde{y})$ is a primal-dual solution to \eqref{eqn:saddlepoint} if $\mathcal{G}^{\star}(\tilde{z})\le 0$.
The rate of convergence of a gap function is typically analyzed in terms of an averaged solution sequence $\bar{z}^N=\sum_{k=0}^N\alpha_k z^k/\sum_{k=0}^N\alpha_k$ for some positive sequence $\{\alpha_k\}$, yielding an \emph{ergodic rate}.
Ergodic rates are widely studied in the literature \cite{loris2011generalization,Chen:InverseProblems:2013,Bot2015b,chambolle2011first,Chambolle:MathematicalProgramming:2015},
partly due to ease of analysis.
%
Sometimes the unaveraged (last) solution sequence $\{z_k\}$ or $\{\tilde{z}_k\}$ is preferred as it tends to preserve the desired structural properties better than the ergodic counterpart. 
Analysis based on the unaveraged sequence yields the \emph{non-ergodic rate}
\citep{Davis:SiamJOptim:2015}.

First we establish an $O(1/N)$ ergodic convergence rate of the pre-duality gap evaluated for
an average of the first $N$ terms of the sequence 
$\{(\tilde{x}^{k},\tilde{y}^{k})\}$:
\begin{theorem}\label{thm:ergodic}
	In iteration \eqref{eqn:preconditionedforwardbackward_relax},
	let $\mu$ be a constant such that 
		$\| (x,0) \|_{M^{-1}}^2 \le (1/\mu)\|x\|_2^2$, for all $x \in \mathbb{R}^p$.
	Let $\alpha=(2\mu)/(4\mu-L_f)$
	and denote $z^{k}=(x^{k},y^{k})$,
	$\tilde{z}^{k}=(\tilde{x}^{k},\tilde{y}^{k})$.
	Define $\bar{z}^N=(\bar{x}^N,\bar{y}^N)$ with $\bar{x}^{N}=\sum_{k=0}^N\rho_k\tilde{x}^{k}/\sum_{k=0}^N\rho_k$ and $\bar{y}^{N}=\sum_{k=0}^N\rho_k\tilde{y}^{k}/\sum_{k=0}^N\rho_k$.
	Also let $\bar\rho = \sup_{k\ge 0}\rho_k$. 
	If $\mu > L_f/2$ and $\{\rho_k\}$ is chosen so that $0 < \rho_k < 1/\alpha$ for all $k$, 
	then the following holds for all $z=(x,y) \in \mathbb{R}^{p}\times\mathbb{R}^{l}$:
	\begin{align*}
		\mathcal{G}(\bar{z}^N,z)
		\le \textstyle  \frac{1}{2\sum_{k=0}^N\rho_k}\left( \|z^{0}-z\|_M^2 + \frac{\alpha L_f}{(1-\alpha\bar{\rho})\lambda_{\min}(M)}\|z^{0}-z^{\star}\|_M^2\right),
	\end{align*}
	where $z^{\star}=(x^{\star},y^{\star})$ is a solution to \eqref{eqn:saddlepoint}. 
\end{theorem}
The key observation in proving Theorem \ref{thm:ergodic} is the following lemma, also used in the proof of Theorem \ref{thm:nonergodic}.
\begin{lemma}\label{lemma:pregap}
	For $\rho \in (0,2)$,
	consider a relation $z^{+}=(I+M^{-1}F)^{-1}(I-M^{-1}G)z^{-}$, $z_{\rho} = (1-\rho)z^{-}+\rho z^{+}$.
	Write $z_{\rho}=(x_{\rho},y_{\rho})$, $z^{+}=(x^{+},y^{+})$, $z^{-}=(x^{-},y^{-})$, all in $\mathbb{R}^p\times\mathbb{R}^l$.
	Then, 
	$$
		2\rho \; \mathcal{G}(z^{+},z)
		\le  \|z^{-}-z\|_M^2 - \|z_{\rho}-z\|_M^2
		 + (1-2/\rho)\|z^{-}-z_{\rho}\|_M^2  
		 + (L_f/\rho) \|x^{-}-x_{\rho}\|_2^2, 
		 \quad \forall z=(x,y).
	$$
\end{lemma}

Now let $\mathcal{F}(x)=f(x)+h(Kx)$ be the primal objective function and $\mathcal{F}^{\star}$ be the primal optimal value.
For an important class of penalty functions $h$ including those for the generalized and group lasso, the following rate for primal suboptimality holds.
\begin{corollary}\label{cor:ergodic}
Assume the conditions for Theorem \ref{thm:ergodic}.
If $\dom(h)=\mathbb{R}^l$, i.e., $h$ does not take the value $+\infty$, then there exists a constant $C_1$ independent of $N$ such that for all $N$,
\begin{align*}
0 \le \mathcal{F}(\bar{x}^N) - \mathcal{F}^{\star} \le C_1/(\textstyle\sum_{k=0}^N\rho_k). 
\end{align*}
\end{corollary}
Thus if $\{\rho_k\}$ is chosen so that  $\inf_{k \ge 0} \rho_k > 0$, we obtain $O(1/N)$ convergence of the primal suboptimality.


The following theorem establishes the non-ergodic counterpart of Theorem \ref{thm:ergodic}.
\begin{theorem}\label{thm:nonergodic}
	For some $\nu > L_f/2$ and $\epsilon > 0$,
	suppose $M$ in iteration \eqref{eqn:preconditionedforwardbackward_relax} satisfies
	\begin{align}\label{eqn:Mbound}
		M \succeq \begin{bmatrix} \nu I &  \\   & \epsilon I \end{bmatrix}
	\end{align}
	Let $\alpha=2\nu/(4\nu-L_f)$ and write $z^{k}=(x^{k},y^{k})$, $\tilde{z}^{k}=(\tilde{x}^{k},\tilde{y}^{k})$.
	If $\{\rho_k\}$ is chosen so that $0 < \rho_k < 1/\alpha$ for all $k$ and  
	$\tau=\inf_{k\ge0}\rho_k(1-\alpha\rho_k) > 0$,
	then the following holds:
	\begin{align*}
		 \mathcal{G}(\tilde{z}^k,z) 
		 \le ~ {\|z^{0}-z^{\star}\|_M(\|z^{0}-z^{\star}\|_M+\|z^{\star}-z\|_M)}/({\sqrt{\tau(k+1)}}),
		 \quad
		 \forall z=(x,y)\in \mathbb{R}^{p}\times\mathbb{R}^{l},
	\end{align*}
	and additionally, 
	$\mathcal{G}(\tilde{z}^k,z) = o(1/\sqrt{k+1})$.
	Furthermore, if $\dom(h)=\mathbb{R}^l$, 
	then there exists a constant $C_2$ independent of $k$ such that $0 \le \mathcal{F}(\tilde{x}^k) - \mathcal{F}^{\star} \le C_2/\sqrt{k+1}$ for all $k$ and $\mathcal{F}(\tilde{x}^k) - \mathcal{F}^{\star}=o(1/\sqrt{k+1})$.
\end{theorem}
\begin{remark}
The little-o result suggests that the non-asymptotic upper bound of the gap function may be conservative and the gap may diminish faster than the $1/\sqrt{k+1}$ rate. 
The outcomes of the numerical experiments in Section \ref{sec:numerical} also suggest that the bound is not tight.
\end{remark}

\subsubsection*{Closing the gap}
Here we describe how our results close the gap in the literature between the conditions for convergence and those for the rate.
The following fact helps
understanding the conditions for Theorems \ref{thm:ergodic} and \ref{thm:nonergodic}:
\begin{proposition}\label{prop:Mcondition}
	For $M \succ 0$ and a given $L_f>0$, the following are equivalent.
	\begin{enumerate}
	\item\label{prop:Mcondition:1} For all $x \in \mathbb{R}^p$, there exists $\mu>L_f/2$ such that
		$\| (x,0) \|_{M^{-1}}^2 \le (1/\mu)\|x\|_2^2$.
	\item\label{prop:Mcondition:2} The condition \eqref{eqn:fullrange} holds.
	\item\label{prop:Mcondition:3} There exist $\nu>L_f/2$ and $\epsilon>0$ such that $M \succeq \begin{bmatrix} \nu I &  \\  & \epsilon I \end{bmatrix}.$
	\end{enumerate}
\end{proposition}
That is, the conditions for Theorems \ref{thm:ergodic} and \ref{thm:nonergodic} are both equivalent to \eqref{eqn:fullrange}.
This implies that the rates of convergence results in this section hold 
for $M$ in \eqref{eqn:generalM} satisfying \eqref{eqn:generalparamrange}.
	\emph{Thus, for the entire range of $(\sigma, \tau)$ for which \eqref{eqn:generaliteration} converges, we have established an $O(1/N)$ ergodic and an $o(1/\sqrt{k+1})$ non-ergodic convergence rates for the objective values.}

For Algorithm LV ($M=M_{\mathsf{LV}}$),
	\citet{loris2011generalization} obtain an $O(1/N)$ ergodic convergence rate for $f(x)=\frac{1}{2}\|\mathsf{A}x-b\|_2^2$.
	For general $f$, \citet{Chen:InverseProblems:2013} show that Algorithm LV converges under \eqref{eqn:LVparamrange}, but the rate is given only for strongly convex $f$ and full row rank $K$. 
	This special case is not very interesting in statistical learning applications in which $f$ is almost always not strongly convex. 
	To the best of our knowledge, our result for the rates of convergence for Algorithm LV and its variants (including the optimal accelerated one in the next subsection) without this impractical assumption is novel.
	For Algorithm CV ($M=M_{\mathsf{CV}}$),
our result extends the region of parameters for which ergodic converge rate is known from $(1/\tau-\sigma\|K\|^2)/L_f \ge 1$ \cite[][Theorems 1 and 2]{Chambolle:MathematicalProgramming:2015} to the full range $(1/\tau-\sigma\|K\|^2)/L_f \ge 1/2$ of \eqref{eqn:CPparamrange2}. 
	Therefore we close the gap between the conditions for convergence and those for the rate.


\begin{remark}
An inspection of the proof of Lemma \ref{lemma:pregap} asserts that the results of this section also holds for the extended $F$ (see Remark \ref{remark:CV+}). Thus we close the gap for Algorithm CV+, the three-function extension, as well.
\end{remark}
\begin{remark}
\citet[Proposition 5.3]{Davis:SiamJOptim:2015} analyzes both ergodic and non-ergodic rates for general $F$ and $G$, under the condition $M \succeq \underline{\lambda} I$ for some $\underline{\lambda}>0$. When applied to \eqref{eqn:preconditionedforwardbackward_relax}, this analysis results in a convergence region
smaller than that is allowed by \eqref{eqn:fullrange}. 
Here we exploit the special structure of $G$ in \eqref{eqn:operatorsplit}. 
\end{remark}
%



%% file: sections/acceleration.tex
\section{Optimal acceleration}\label{sec:optimal}
It is well known that first-order methods can be accelerated by introducing some ``inertia'' \cite{nesterov2004introductory,Beck:SiamJournalOnImagingSciences:2009,Chen:TheAnnalsOfAppliedStatistics:2012}. For the saddle-point problem of the form \eqref{eqn:saddlepoint}, the optimal rate of convergence is known to be $O(L_f/N^2 + \|K\|_2/N)$ in terms of the duality gap $\mathcal{G}^\star$, where $N$ is the total number of iterations \cite{Nesterov:MathematicalProgramming:2004,chen2014optimal}. A natural question arises regarding whether the same optimal rate can be attained for the entire continuum \eqref{eqn:generaliteration} of algorithms. In this section, we show that the answer is affirmative. 

\subsection{Algorithms}

\citet{chen2014optimal} devise an accelerated variant of Algorithm CV that achieves the theoretically optimal rate of convergence $O(L_f/N^2+\|K\|_2/N)$, where $N$ is the total number of iterations:
\begin{subequations}\label{eqn:chen}
\begin{align}
\bar{x}^k &= \tilde{x}^k + \theta_k (\tilde{x}^k - \tilde{x}^{k-1}) \label{eqn:chen:inertia} \\
x_{md}^k &= (1-\rho_k) x^k + \rho_k \tilde{x}^k \label{eqn:chen:mid} \\
\tilde{y}^{k+1} &= \prox_{\sigma_k h^*} (\tilde{y} + \sigma_k K \bar{x}^k) \label{eqn:chen:c}\noeqref{eqn:chen:c}\\
\tilde{x}^{k+1} &= \tilde{x}^k - \tau_k (\nabla f (x_{md}^k) + K^T \tilde{y}^{k+1}) \label{eqn:chen:d}\noeqref{eqn:chen:d} \\
x^{k+1} &= (1 - \rho_{k}) x^{k} + \rho_{k}\tilde{x}^{k+1} \label{eqn:chen:e}\noeqref{eqn:chen:e}\\ 
y^{k+1} &= (1 - \rho_{k}) y^{k} + \rho_{k}\tilde{y}^{k+1}. \label{eqn:chen:f}\noeqref{eqn:chen:f}
\end{align}
\end{subequations}
Note an extrapolation step \eqref{eqn:chen:inertia} with a parameter $\theta_k$,
and a ``middle'' relaxation step \eqref{eqn:chen:mid} are introduced.
For \eqref{eqn:generaliteration}, we consider the following generalization:
\begin{subequations}\label{eqn:accgen}
\begin{align}
\bar{u}^{k} &= K\tilde{x}^{k} - \theta_{k} A (\tilde{x}^{k} - \tilde{x}^{k-1}) \label{eqn:accgen:a}\noeqref{eqn:accgen:a} \\
\bar{v}^{k} &= K^T\tilde{y}^{k} + \theta_{k} \left(\tau_{k}^{-1}\tau_{k-1} (K+B)^T - B^T \right)(\tilde{y}^{k} - \tilde{y}^{k-1}) \label{eqn:accgen:b}\noeqref{eqn:accgen:b} \\
x_{md}^{k} &= (1-\rho_{k})x^{k} +\rho_{k}\tilde{x}^{k} \label{eqn:mid} \\
\tilde{u}^{k+1} &= \bar{u}^{k} - \tau_{k} (K+A)(\nabla f(x_{md}^{k}) + \bar{v}^{k}) \label{eqn:accgen:d}\noeqref{eqn:accgen:d}\\
\tilde{y}^{k+1} &= \prox_{\sigma_{k} h^*} (\tilde{y}^{k} + \sigma_{k}  \tilde{u}^{k+1}) \label{eqn:accgen:e}\noeqref{eqn:accgen:e}\\
\tilde{v}^{k+1} &= K^T \tilde{y}^{k+1} + B^T (\tilde{y}^{k+1} - \tilde{y}^{k})- \theta_{k} B^T(\tilde{y}^{k}-\tilde{y}^{k-1}) \label{eqn:accgen:f}\noeqref{eqn:accgen:f} \\
\tilde{x}^{k+1} &= \tilde{x}^{k} - \tau_{k} (\nabla f(x_{md}^{k}) +  \tilde{v}^{k+1}) \label{eqn:accgen:g}\noeqref{eqn:accgen:g}\\
x^{k+1} &= (1 - \rho_{k}) x^{k} + \rho_{k}\tilde{x}^{k+1} \label{eqn:xag} \\ 
y^{k+1} &= (1 - \rho_{k}) y^{k} + \rho_{k}\tilde{y}^{k+1}. \label{eqn:yag}
\end{align}
\end{subequations}
Step sizes $(\sigma_k,\tau_k)$ are allowed to depend on the iteration count $k$.
This algorithm reduces to \eqref{eqn:generaliteration} (hence to Algorithms LV, CV, and in between) if $A=-C$, $B=C$, $\rho_k \equiv 1$, $\theta_k \equiv 0$, $\sigma_k\equiv\sigma$, and $\tau_k\equiv\tau$,
and  to \citet{chen2014optimal} for $A=-K$ and $B=0$. 
The optimal rate of convergence of \eqref{eqn:accgen} is established in Section \ref{sec:accconv}.
In particular, the optimal acceleration of Algorithm LV is new.

%% file: sections/rates2.tex
\subsection{Convergence analysis}\label{sec:accconv}

We first consider the case in which the bounds for $\{x^k\}$, $\{y^k\}$ is known \emph{a priori}. In this case we can assume that the search space is $Z=X\times Y$, where $X\subset \mathbb{R}^p$, $Y\subset \mathbb{R}^l$ are both closed and bounded. Under this assumption, we have the following bound for the duality gap:
\begin{theorem}\label{thm:bddgen}
Let $\{z^k\}=\{(x^k,y^k)\}$ be the sequence generated by \eqref{eqn:accgen}.
Assume for some $\Omega_X$, $\Omega_Y >0$, 
\begin{align}\label{eqn:bdd}
	\textstyle \sup_{x, x' \in X} \|x - x'\|_2^2 \le 2\Omega_X^2, \;\; \sup_{y, y' \in Y} \|y - y' \|_2^2 \le 2\Omega_Y^2,
\end{align}
and the parameter sequences $\{\rho_k\}$, $\{\theta_k\}$, $\{\tau_k\}$, and $\{\sigma_k\}$ satisfy $\rho_1 = 1$ and
\begin{subequations}\label{eqn:bddcondgen}
\begin{align}
\textstyle &\rho^{-1}_{k+1}-1 = \rho^{-1}_k \theta_{k+1}, \label{eqn:rhocondgen}\\
\textstyle &\frac{1-q}{\tau_k} - L_f \rho_k - \frac{1}{r}\|A\|_2^2 \sigma_k \ge 0, \label{eqn:condgen1} \\
\textstyle &\frac{1-r}{\sigma_k} - \tau_k \left(2\|K+A\|_2\|K+B\|_2 + \frac{1}{q}\|B\|_2^2 \right) \ge 0\label{eqn:condgen2}
\end{align}
\end{subequations}
for some $q \in (0,1)$, $r \in (0,1)$.
Further suppose that
\begin{align}
0 < \theta_{k} \le \min({\tau_{k-1}}/{\tau_k}, {\sigma_{k-1}}/{\sigma_k}), & \max({\tau_{k-1}}/{\tau_k}, {\sigma_{k-1}}/{\sigma_k}) \le 1. \label{eqn:thetacondbddgen} 
\end{align}
\deleted{for some {$r \in (0,1)$, $s \in (0,1)$},
and $\rho_1 = 1$} 
Then for all $k \ge 1$, 
\begin{align}\label{eqn:bddupper}
		\mathcal{G}^{\star}(z^{k+1}) & \le 
		\textstyle \frac{\rho_k}{\tau_k}\Omega_X^2 + \frac{\rho_k}{\sigma_k} \Omega_Y^2.
\end{align}
\end{theorem}
For the following choice of the algorithm parameters, we obtain the claimed optimal convergence rate. 
\begin{corollary}\label{cor:bddgen}
If $\|A\|_2 \le a\|K\|_2$, $\|B\|_2 \le b\|K\|_2$, $\|K+A\|_2 \le c\|K\|_2$, and $\|K+B\|_2 \le d\|K\|_2$ for some $a$, $b$, $c$, and $d >0$, and the parameters
are set to 
\begin{align}\label{eqn:bddparams}
	\textstyle \rho_k = \frac{2}{k+1}, \;\; \theta_k = \frac{k-1}{k}, \;\; \tau_k = \frac{k}{2P L_f + k Q \|K\|_2 \Omega_Y/\Omega_X}, \;\; \sigma_k = \frac{\Omega_Y}{\|K\|_2 \Omega_X}, \text{where}
\end{align}
\begin{align}\label{eqn:bddparamcond}
P = \textstyle\frac{1}{1-q} \quad \text{and} \quad Q = \max\left\{\frac{1}{(1-q)r}a^2, \frac{1}{1-r}(2cd+b^2/q)\right\},
\end{align}
then
\begin{align}\label{eqn:bddoptimal}
		\mathcal{G}^{\star}(z^k) \le \textstyle \frac{4 P \Omega_X^2}{k(k-1)} L_f + \frac{2\Omega_X \Omega_Y(Q+1)}{k}\|K\|_2 , 
		\quad \forall k \ge 2.
\end{align}
\end{corollary}

\begin{remark}
For $A=-K$, $B=0$, \eqref{eqn:bddcondgen} recovers the condition for \citet[Theorem 2.1]{chen2014optimal} by putting $r \rightarrow 1$ and $q \rightarrow 0$. 
For 
$A=-\kappa K=-B$, 
we obtain $(1-|\kappa|q)/\tau_k \ge L_f \rho_k + |\kappa| \|K\|_2^2 \sigma_k/r$ and $(1-|\kappa| r)/\sigma_k \ge \|K\|_2^2 \tau_k \left(2(1-\kappa^2) + |\kappa|/q\right)$. 
In particular for Algorithm LV ($\kappa=0$), we have $1/\tau_k \ge L_f \rho_k$ and $1/(\tau_k\sigma_k) \ge 2\|K\|_2^2$ regardless of $q$ and $r$; 
this condition resembles \eqref{eqn:LVparamrange}.
\end{remark}

Now suppose the bounds for $\{x^k\}$, $\{y^k\}$ are unavailable. In this case the duality gap $\sup_{z \in Z} \mathcal{G}(\tilde{z}, z)$, $Z = \mathbb{R}^p \times \mathbb{R}^l$, may be unbounded above. 
Instead, we define a perturbed gap function: 
\begin{align}\label{eqn:pgap}
		\tilde{\mathcal{G}}(\tilde{z}, v) := \sup_{z \in Z} \mathcal{G}(\tilde{z}, z) - \langle v, \tilde{z}-z \rangle.
\end{align}
There always exists a perturbation vector $v$ such that \eqref{eqn:pgap} is \replaced[remark={Reviewer 1, Point 11}]{finite}{well-defined} \citep{monteiro2011complexity}. Thus we want to find a sequence of perturbation vector\added{s} $\{v^k\}$ that makes $\tilde{\mathcal{G}}(\tilde{z}^k, v^k)$ small. 

\begin{theorem}\label{thm:unbddgen}
Suppose that $\{z^k\} = \{(x^k, y^k)\}$ are generated by Algorithm \eqref{eqn:accgen}. 
If the parameter sequences $\{\rho_k\}$, $\{\theta_k\}$, $\{\tau_k\}$, and $\{\sigma_k\}$ satisfy \eqref{eqn:bddcondgen}
and 
\begin{align}\label{eqn:thetacondunbdd}
\theta_{k} = {\tau_{k-1}}/{\tau_k}={\sigma_{k-1}}/{\sigma_k} \le 1
\end{align}
for some $0<q<1$, $0<r<1/2$.  
Then there exists a vector $v^{k+1}$ such that for any $k \ge 1$, 
\begin{align}\label{eqn:unbddupper}
\tilde{\mathcal{G}}(z^{k+1}, v^{k+1}) & \le 
	\textstyle\frac{\rho_k}{\tau_k} \left(2 +\frac{q}{1-q} + \frac{2r+1}{1-2r}\right) R^2 =: \epsilon_{k+1},\text{ and}
\end{align}
\begin{align}
\|v^{k+1}\|_2 & \le \textstyle \left(\frac{\rho_k}{\tau_k} \|\hat{x}-\tilde{x}^1\|_2 + \frac{\rho_k}{\sigma_k} \|\hat{y}-\tilde{y}^1\|_2 \right) \label{eqn:vbound}\\ 
& \quad + \textstyle \left(\frac{\rho_k}{\tau_k} (\mu +\frac{\tau_1}{\sigma_1}\nu) + 2\rho_k(\mu\|A\|_2 + \nu\|B\|_2) + 2\tau_k\rho_k\nu \|K+A\|_2\|K+B\|_2 \right) R, \nonumber
\end{align}
where $(\hat{x}, \hat{y})$ is a pair of solutions to problem \eqref{eqn:saddlepoint}, and
\begin{equation}\label{eqn:Rdefgen}  
	\textstyle R = \sqrt{\|\hat{x}-\tilde{x}^1\|_2^2 + \frac{\tau_1}{\sigma_1}\|\hat{y}-\tilde{y}^1\|_2^2}, \quad 
\mu = \sqrt{\frac{1}{1-q}}, \quad 
\nu = \sqrt{\frac{2\sigma_1}{\tau_1 (1-2r)}}. 
\end{equation}
\end{theorem}
For the following choice of the algorithm parameters, we obtain the claimed optimal convergence rate. 
\begin{corollary}\label{cor:unbddgen}
If $\|A\|_2 \le a\|K\|_2$, $\|B\|_2 \le b \|K\|_2$, $\|K+A\|_2 \le c\|K\|_2$, and $\|K+B\|_2 \le d \|K\|_2$ for some $a, b, c, d > 0$, 
$N$ is given, and the parameters are set to 
\begin{align}\label{eqn:unbddparams}
	\textstyle \rho_k = \frac{2}{k+1},\;\; \theta_k = \frac{k-1}{k},\;\;\tau_k = \frac{k}{2PL_f + Q N \|K\|_2}, \;\; \sigma_k = \frac{k}{N\|K\|_2}, \text{ where}
\end{align}
\begin{align}\label{eqn:unbddparamcond}
	\textstyle P = \frac{1}{1-q}, \;\; Q = \max\left\{\frac{a^2}{(1-q)r}, \frac{2cd + b^2/q}{1-r}, 1 \right\},
\end{align}
then
\begin{align}
    \textstyle \epsilon_{N+1} & \textstyle \le \left(\frac{4PL_f}{N^2} + \frac{2Q\|K\|_2}{N}\right) \left[2+\frac{q}{1-q} + \frac{r+1/2}{1/2-r}\right] R^2, \text{ and } \label{eqn:epsfbound}\\
    \textstyle \|v^{N+1}\|_2 & \textstyle \le \frac{4PL_f}{N^2} \left[\left( \|\hat{x} - \tilde{x}^1\|_2 + \|\hat{y} - \tilde{y}^1\|_2\right) + R\left(\mu + \frac{\tau_1}{\sigma_1} \nu\right) \right] \label{eqn:vfbound}\\
\omit\span \textstyle\; + \frac{\|K\|_2}{N} \left[2Q\left(\left( \|\hat{x} - \tilde{x}^1\|_2 + \|\hat{y} - \tilde{y}^1\|_2\right) + R\left(\mu +\frac{\tau_1}{\sigma_1}\nu\right)\right) + 4R(a\mu + b\nu) + \frac{4Rcd\nu}{Q} \right]. \nonumber
\end{align}
\end{corollary}

This result can be interpreted as follows.
Theorem \ref{thm:unbddgen} and Corollary \ref{cor:unbddgen} state that for every pair of positive scalars $(\rho,\varepsilon)$, Algorithm \eqref{eqn:accgen} generates $(v^N,\epsilon_N)$ such that $\|v^N\|\le\rho$ and $\epsilon_N \le\varepsilon$ (see \eqref{eqn:unbddupper}, \eqref{eqn:vbound}, \eqref{eqn:epsfbound}, and \eqref{eqn:vfbound}) for a sufficiently large $N$. The associated pair $(x^N,y^N)$ is called a $(\rho,\varepsilon)$-saddle point of the unperturbed saddle point problem \eqref{eqn:saddlepoint} \citep[][Definition 3.10]{monteiro2011complexity}. 
With this notion, the following proposition can be stated.
\begin{proposition}\label{prop:approxopt}
Under the assumptions of Theorem \ref{thm:unbddgen} and Corollary \ref{cor:unbddgen}, there exists a vector $w^N=(w_x^N,w_y^N)$ such that $w^N \in T_{\epsilon^N}(x^N,y^N)$ and $\|w^N\|\le\rho+\sqrt{4L\varepsilon}$ for some constant $L>0$,  
where 
\[
T_{\varepsilon} = \begin{bmatrix}
		\nabla f & K^T \\
		-K  & \partial_{\varepsilon} h^*
	\end{bmatrix}.
\] 
Here, $\partial_{\varepsilon}h^*$ is the $\varepsilon$-subgradient of $h^*$ defined as
$\partial_{\varepsilon}h^*(y) = \{ g: h^*(y') \ge h^*(y) + 
		\langle y' - y, g \rangle - \varepsilon, \forall y' \in \mathbb{R}^l \}, \quad \forall y \in \mathbb{R}^l$.
\end{proposition}

The condition $w^N \in T_{\epsilon_N}(x^N,y^N)$ in Proposition \ref{prop:approxopt} can be written as the following two inequalities
\begin{subequations}\label{eqn:epsoptimality}
\begin{align}
     0 & \ge -\langle \nabla f(x^N)+K^T y^N,x-x^N \rangle + \langle w_x^N,x-x^N \rangle - \epsilon_N,\;\;\forall x, \label{eqn:epsoptimality:a}\\
     h^*(y) &\ge h^*(y^N) + \langle K x^N,y-y^N \rangle + \langle w_y^N,y-y^N \rangle - \epsilon_N,\;\;\forall y.\label{eqn:epsoptimality:b} 
\end{align}
\end{subequations}
Comparing with the optimality conditions \eqref{eqn:optimality} for the unperturbed saddle point problem \eqref{eqn:saddlepoint}:
\begin{align*}
     0 &\ge -\langle \nabla f(x^\star)+K^T y^\star,x-x^\star \rangle, \quad\forall x, \\
     h^*(y) &\ge h^*(y^\star) + \langle K x^\star,y-y^\star \rangle, \quad\forall y,
\end{align*}
we see that the sum of the last two terms in each right-hand side of \eqref{eqn:epsoptimality:a} and \eqref{eqn:epsoptimality:b} is the error of the approximate solution $(x^N,y^N)$. Indeed, in the unit ball centered at $(x^N,y^N)$, each error is bounded by $\rho+\sqrt{4L\varepsilon}+\varepsilon$, which can be made arbitrarily small since the choice of $(\rho,\varepsilon)$ is free. In this sense, for large $N$, $(x^N,y^N)$ is a ``nearly optimal'' primal-dual solution.

%% file: sections/acceleration_stoc.tex
\section{Stochastic optimal acceleration}\label{sec:stoc}
\subsection{Algorithm}
In large-scale (``big data'') applications, it is often the case that even the first-order information on the objective of \eqref{eqn:primal} or \eqref{eqn:saddlepoint} cannot be obtained exactly. Such settings can be modeled by a stochastic oracle, which provides unbiased estimators of the first-order information. To be precise, at the $k$-th iteration suppose the oracle returns the stochastic gradient $(\hat{\mathcal{F}}(\tilde{x}^k), \hat{\mathcal{K}}_x (\tilde{x}^k), \hat{\mathcal{K}}_y (\tilde{y}^k))$ independently from the previous iteration, such that 
\begin{align}\label{eqn:expect}
\begin{split}
\expect[\hat{\mathcal{F}}(\tilde{x}^k)] = \nabla f(\tilde{x}^k), \;\; \expect \left[\begin{pmatrix}- \hat{\mathcal{K}}_x (\tilde{x}^k) \\ \hat{\mathcal{K}}_y (\tilde{y}^k)\end{pmatrix} \right] = \begin{pmatrix}-K\tilde{x}^k \\ K^T \tilde{y}^k \end{pmatrix}, \\
\expect[\hat{\mathcal{A}}(\tilde{x}^k)] = A\tilde{x}^k,\;\; \text{and} \;\; \expect[\hat{\mathcal{B}}(\tilde{y}^k)] = B^T \tilde{y}^k.
\end{split}
\end{align}
We further assume that the variance of these estimators are uniformly bounded, i.e., 
\begin{align}\label{eqn:A1}
\begin{split}
\expect[\|\hat{\mathcal{F}}(\tilde{x}^k) - \nabla f(\tilde{x}^k) \|^2] \le \chi_{x, f}^2, \;\; \expect[\|\hat{\mathcal{K}}_x(\tilde{x}^k)-K \tilde{x}^k\|^2] \le \chi_y^2, \; \expect[\|\hat{\mathcal{K}}_y(\tilde{y}^k) - K^T \tilde{y}^k\|^2] \le \chi_{x, K}^2, \\
\expect[\|\hat{\mathcal{A}}(\tilde{x}^k)-A \tilde{x}^k\|^2] \le \chi_A^2 \text{ and } \expect[\|\hat{\mathcal{B}}(\tilde{y}^k)-B^T \tilde{y}^k\|^2] \le \chi_B^2.
\end{split}
\end{align}
For notational convenience, we define $\chi_x := \sqrt{\chi_{x, f}^2 + \chi_{x, K}^2}$. 

We consider the following stochastic variant of \eqref{eqn:accgen}:
\begin{align}\label{eqn:stocgen}
\begin{split}
\overline{u}_k &= \hat{\mathcal{K}}_x (\tilde{x}^k) - \theta_k \hat{\mathcal{A}}(\tilde{x}^k - \tilde{x}^{k-1}) \\
\overline{v}_k &= \hat{\mathcal{K}}_y (\tilde{y}^k+\frac{\theta_k\tau_{k-1}}{\tau_k}) + \hat{\mathcal{B}}\left(\left(\frac{\tau_{k-1}}{\tau_k}-1 \right) (\tilde{y}^k - \tilde{y}^{k-1})\right)\\
\tilde{x}^k_{md} &= (1-\rho_k)x^k +\rho_k\tilde{x}^k  \\
\tilde{u}^{k+1} &= \overline{u}_k - \tau_k(\hat{\mathcal{K}_x} + \hat{\mathcal{A}}) (\hat{\mathcal{F}}(\tilde{x}^k_{md}) + \overline{v}_{k}) \\
\tilde{y}^{k+1} &= \prox_{\sigma_k h^*} (\tilde{y}^k + \sigma_k \tilde{u}^{k+1}) \\
\tilde{v}^{k+1} &= \hat{\mathcal{K}}_y (\tilde{y}^{k+1}) +\hat{\mathcal{B}}(\tilde{y}^{k+1}-\tilde{y}^k - \theta_k (\tilde{y}^k - \tilde{y}^{k-1})) \\
\tilde{x}^{k+1} &= \tilde{x}^k - \tau_k (\hat{\mathcal{F}}(\tilde{x}^k_{md}) +  \tilde{v}^{k+1}) \\
x^{k+1} &= (1 - \rho_k) x^k + \rho_k\tilde{x}^{k+1} \\
y^{k+1} &= (1 - \rho_k) y^k + \rho_k\tilde{y}^{k+1},
\end{split}
\end{align}
which can be considered a generalization of the stochastic variant of \eqref{eqn:chen} by \citet{chen2014optimal}. 
The optimal rate of convergence of solving \eqref{eqn:saddlepoint} stochastically is known to be $O \left(\frac{L_f}{N^2} + \frac{\|K\|_2}{N} + \frac{\chi_x +\chi_y}{\sqrt{N}} \right)$ in terms of the expected duality gap $\expect[\mathcal{G}^\star]$ \citep{chen2014optimal}.  
In the sequel, we show that Algorithm \eqref{eqn:stocgen} achieves this rate. 

%% file: sections/rates_stoc.tex
\subsection{Convergence analysis}

We obtain the following results for Algorithm \eqref{eqn:stocgen} when $Z$ is bounded. 
Note part \ref{thm:stocbdd:2} of Theorem \ref{thm:stocbdd} is strengthened under the tail assumption
\begin{align}\label{eqn:A2}
\begin{split}
\expect \left[\exp(\|\nabla f(x)-\hat{\mathcal{F}}(x) \|^2/ \chi_{x, f}^2) \right] &\le \exp(1) \\
\expect \left[\exp (\|Kx - \hat{\mathcal{K}}_x (x)\|^2/\chi_y^2) \right] &\le \exp(1) \\
\expect \left[\exp (\|K^T y - \hat{\mathcal{K}}_y (y)\|^2/ \chi_{x, K}^2 )\right] & \le \exp(1).
\end{split}
\end{align}
Observe that \eqref{eqn:A2} implies \eqref{eqn:A1} by Jensen's inequality.

\begin{theorem}\label{thm:stocbdd}
Assume that \eqref{eqn:bdd} holds, for some $\Omega_X$, $\Omega_Y$ $>0$. Also suppose that for all $k \ge 1$, the parameters $\rho_k$, $\theta_k$, $\tau_k$, and $\sigma_k$ in \eqref{eqn:stocgen} satisfy \eqref{eqn:rhocondgen}, \eqref{eqn:thetacondbddgen}, 
\begin{subequations}\label{eqn:stocparam}
\begin{align}
\frac{s-q}{\tau_k} - L_f \rho_k - \frac{\|A\|_2^2 \sigma_k}{r} & \ge 0, \label{eqn:stoccond1}\\ 
\frac{t-r}{ \sigma_k}- \tau_k\left(2 \|K+A\|_2\|K+B\|_2 + \frac{\|B\|_2^2}{q}\right) & \ge 0 \label{eqn:stoccond2}
\end{align}
\end{subequations}
for some $q$, $r$, $s$, $t \in (0,1)$. Then the following holds.
\begin{enumerate}[label=(\roman*)]
\item Under \eqref{eqn:A1}, we have
$\expect [\mathcal{G}^\star(z^{k+1})] \le \mathcal{Q}_0(k)$
for all $k \ge 1$,  
where 
\begin{align}\label{thm:stocbdd:1}
\mathcal{Q}_0 (k) & := \frac{\rho_k}{\gamma_k} \left(\frac{2\gamma_k}{\tau_k} \Omega_X^2 + \frac{2 \gamma_k}{\sigma_k} \Omega_Y^2 \right) \\
& \quad + \frac{\rho_k}{2\gamma_k} \sum_{i=1}^k \left(\frac{(2-s) \tau_i \gamma_i}{1-s} (\chi_x^2 + \chi_B^2) 
+\frac{(2-t) \sigma_i \gamma_i}{1-t}(\chi_y^2 + \chi_A^2 + \tau_k^2 \|K+A\|_2^2 (\chi_x^2 + \chi_B^2)) \right) \nonumber
\end{align}
\item Suppose $A=-K$ and $B=bK$, then under the assumption \eqref{eqn:A2}, 
we have
\begin{align}\label{thm:stocbdd:2} 
\prob (\mathcal{G}^\star(z^{k+1})>\mathcal{Q}'_0(k) + \lambda \mathcal{Q}_1(k)) \le 3 \exp(-\lambda^2/3)+ 3 \exp(-\lambda), 
\end{align}
for all $\lambda>0$ and $t \ge 1$, 
where
\begin{align}
\mathcal{Q}'_0 (k) & := \frac{\rho_k}{\gamma_k} (\frac{2\gamma_k}{\tau_k} \Omega_X^2 + \frac{2 \gamma_k}{\sigma_k} \Omega_Y^2 ) + \frac{\rho_k}{2 \gamma_k} \sum_{i=1}^k (\frac{(2-s) \tau_i \gamma_i}{1-s} \chi_x^2 +\frac{(2-t) \sigma_i \gamma_i}{1-t}\chi_y^2  ), \\
\mathcal{Q}_1 (k) &:= \frac{\rho_k}{\gamma_k} (\sqrt{2}\chi_x \Omega_X+ \chi_y \Omega_Y ) \sqrt{2 \sum_{i=1}^k \gamma_i^2 } + \frac{\rho_k}{2 \gamma_k}\sum_{i=1}^k (\frac{(2-s) \tau_i \gamma_i}{1-s}\chi_x^2  + \frac{(2-t) \sigma_i \gamma_i}{1-t} \chi_y^2).
\end{align}
\end{enumerate}
\end{theorem}

\begin{corollary}\label{cor:stocbdd}
Assume condition \eqref{eqn:bdd} holds. In Algorithm \eqref{eqn:stocgen}, if $N \ge 1$ is given, $A=-K$, $\|B\|_2 \le b\|K\|_2$,   and the parameters are set to 
\begin{align}\label{eqn:stocbddparams}
\rho_k & = \frac{2}{k+1},\;\;\theta_k = \frac{k-1}{k},\;\; \tau_k = \frac{\Omega_X k}{2P L_f \Omega_X + Q \|K\|_2 \Omega_Y (N-1) + \chi_x N \sqrt{N-1}}, \\
\sigma_k &= \frac{\Omega_Y k }{\|K\|_2 \Omega_X (N-1) + \chi_y N \sqrt{N-1}}
\end{align}
where $P$ and $Q$ satisfies 
\begin{align}\label{eqn:stocbddparamcond}
P = \frac{1}{s-q}, \;\; Q \ge \max\left\{\frac{1}{r(s-q)}, \frac{b^2/q}{t-r}\right\}, 
\end{align}
the following holds.

\begin{enumerate}[label=(\roman*)]
\item Under assumption \eqref{eqn:A1}, we have
		$\expect[\mathcal{G}^\star(z^N)] \le \mathcal{C}_0(N)$,
 where 
\begin{align}
\begin{split}\label{eqn:C0}
\mathcal{C}_0(N) &= \frac{8 P L_f \Omega_X^2}{N(N-1)} + \frac{4 \|K\|_2 \Omega_X \Omega_Y (Q+1)}{N} + \frac{4 \chi_x \Omega_X + 4 \chi_y \Omega_Y}{\sqrt{N-1}} 
\\
& \quad 
+ \frac{(2-r) \Omega_X \chi_x}{3 (1-r)  \sqrt{N-1}} + \frac{(2-s) \Omega_Y \chi_y}{3(1-s) \sqrt{N-1}}.
\end{split}
\end{align}
\item Under assumption \eqref{eqn:A2}, then we have
\begin{align}
P(\mathcal{G}^\star(z^N) > \mathcal{C}_0(N) + \lambda \mathcal{C}_1(N)) \le 3 \exp(-\lambda^2/3) + 3 \exp(- \lambda),
\end{align}
for all $\lambda>0$, 
where 
\begin{align}
\begin{split}\label{eqn:C1}
\mathcal{C}_1(N) =\left( \frac{4}{\sqrt{3}} + \frac{2-r}{3(1-r)}\right) \frac{\Omega_X \chi_x}{\sqrt{N-1}} + \left(\frac{2\sqrt{2}}{\sqrt{3}} + \frac{2-s}{3 (1-s)}\right) \frac{\Omega_Y \chi_y}{\sqrt{N-1}}.
\end{split}
\end{align}
\end{enumerate}
\end{corollary}


When $Z$ is unbounded, we have the following theorem. 

\begin{theorem}\label{thm:stocunbdd}
Assume that $\{z^k\} = \{(x^k, y^k)\}$ is the sequence generated by \eqref{eqn:stocgen}. Further assume that the parameters $\beta_k$, $\theta_k$, $\tau_k$, and $\sigma_k$ in \eqref{eqn:stocgen} satisfy  \eqref{eqn:rhocondgen}, \eqref{eqn:thetacondunbdd}, and \eqref{eqn:stocparam}.
for all $k \ge 1$ and some $q$, $r$, $s$, $t \in (0,1)$. Then there is a perturbation vector $v^{k+1}$ satisfying 
\begin{align}
\expect [\tilde{\mathcal{G}}(z^{k+1}, v^{k+1})] \le \frac{\rho_k}{\tau_k} \left[ \left( 6 +\frac{4q}{1-q} +\frac{4(r+1/2)}{1/2-r}\right) R^2 + \left( \frac{5}{2} + \frac{2q}{1-q} + \frac{2(r+1/2)}{1/2-r} \right) S^2\right]
\end{align}
for all $k \ge 1$. Furthermore, 
\begin{align}
\expect[\|v^{k+1}\|] & \le \frac{2 \rho_k \|\hat{x} - x^1\|}{\tau_k}+ \frac{2 \rho_k \|\hat{y}-y^1\|}{\sigma_k} + \sqrt{2R^2+S^2} [\frac{\rho_k(1+\mu)}{\tau_k} +(\nu + \sqrt{\frac{\sigma_1}{\tau_1}}) \frac{\rho_k}{\sigma_k} \\
&+ 2 \rho_k (\|A\|_2\mu + \|B\|_2\nu ) + 2 \tau_k \rho_k \|K+A\|_2\|K+B\|_2 \nu ] =: \epsilon_{k+1}
\end{align}
where $(\hat{x}, \hat{y})$ is a pair of solutions for \eqref{eqn:saddlepoint}, $R$, $\mu$, and $\nu$ are as defined in \eqref{eqn:Rdefgen}, and 
\begin{align}\label{eqn:Sdef}
S &:= \sqrt{\sum_{i=1}^k \frac{(2-s)\tau_i^2 (\chi_x^2+\chi_B^2)}{1-s} + \sum_{i=1}^k \frac{(2-t)\tau_i \sigma_i (\chi_y^2 + \chi_A^2 + \tau_k^2 \|K+A\|_2^2 (\chi_x^2 + \chi_B^2))}{1-t}}.
\end{align}
\end{theorem}

\begin{corollary}\label{cor:stocunbdd}
In Algorithm \eqref{eqn:stocgen}, if $N$ is given, $A=-K$, $B=bK$, and the parameters are set to 
\begin{align}\label{eqn:stocunbddparams}
\rho_k = \frac{2}{k+1}, \;\; \theta_k = \frac{k-1}{k}, \;\; \tau_k = \frac{k}{\tau}, \;\; \sigma_k = \frac{k}{\|K\|_2 (N-1) + N \sqrt{N-1} \chi/\tilde{R}},
\end{align}
where 
\begin{align}
\tau = 2 P L_f +  Q \|K\|_2 (N-1) +  N \sqrt{N-1} \chi/\tilde{R}
\end{align}
for some $\tilde{R}>0$, where $\chi$ is defined by $\chi = \sqrt{\frac{2-s}{1-s} \chi_x^2 +  \frac{2-t}{1-t} \chi_y^2}$. 
Then for $P$ and $Q$ satisfying 
\begin{align}\label{eqn:stocunbddparamcond}
P = \frac{1}{s-q},\;\; Q \ge \max \left\{ \frac{1}{r(s - q)}, \frac{b^2}{q(t - r)}, 1\right\},
\end{align}
we have 
\begin{align*}
\expect[\|v^N\|] & \le \left(\frac{4PL_f}{N(N-1)} + \frac{2Q\|K\|_2}{N} + \frac{2 \chi/\tilde{R}}{\sqrt{N-1}} \right) \left(4R + \left(\sqrt{2}R + \frac{\tilde{R}}{\sqrt{3}}\right) (2 + \mu' + \nu')\right)\\ 
& \quad + \frac{2\|K\|_2}{N} (\sqrt{2}R + \tilde{R}/\sqrt{3}) (2 \mu' + 2b \nu'),
\end{align*}
and 
\begin{align*}
\epsilon_N & \le \left(\frac{4PL_f}{N(N-1)} + \frac{2 Q \|K\|_2}{N}+ \frac{2 \chi /\tilde{R}}{\sqrt{N-1}} \right) \left((6 + \frac{4q}{1-q}+\frac{4(r+1/2)}{1/2-r}) R^2 + \frac{(\frac{5}{2}+\frac{2q}{1-q}+\frac{2(r+1/2)}{1/2-r}) \tilde{R}^2}{3}\right). 
\end{align*}
\end{corollary}
Therefore we obtain the desired order for both $\epsilon_N$ and $\expect[\|v_N\|]$.

%% file: sections/experiment.tex

%% file: sections/numerical.tex
\section{Numerical experiments}\label{sec:numerical}
In this section, we illustrate the actual convergence behavior of the 
algorithms generated by \eqref{eqn:generaliteration} and their accelerated variant \eqref{eqn:accgen}. 
In addition, we demonstrate the scalability of these algorithms by implementing a distributed version of \eqref{eqn:generaliteration}. 
The experiment was conducted on a system with two Intel Xeon CPUs (E5-2680 v2 @2.80GHz) with eight Nvidia GTX 1080 GPUs with 8 GB of RAM each. 

\subsection{Model problems}\label{sec:experiment}
\paragraph{Overlapping group lasso.}
We consider an overlapping group lasso problem with a quadratic loss  
$$
\min_{x} \frac{1}{2}\|b- \mathsf{A} x\|_2^2 + \lambda \sum_{j=1}^R \sqrt{|g_j|}\|x_{g_j}\|_2, 
$$
where $\mathsf{A} = [a_1, \cdots, a_n]^T$ is the data matrix, and $b = (b_1, \cdots, b_n)$ is the response vector. 
We generated a test dataset based on the methods in \citet{Chen:TheAnnalsOfAppliedStatistics:2012}.
We defined $R$ groups of $S$ adjacent variables, with 10 overlaps of adjacent groups. i.e., $g_j = \{90(j-1)+1, \dotsc, 90j+10\}$, thus $p=R(S-10)+10$\deleted{, with $n=5,000$}.
 We set $x_j = (-1)^j \exp (-(j-1)/100)$ for $j = 1, \dotsc, p$. We sampled each element of $\mathsf{A}$ from the standard normal distribution, and added Gaussian noise $\epsilon \sim \mathcal{N}(0,1)$ to $\mathsf{A}x$ to generate $b = \mathsf{A}x + \epsilon$. For the convergence experiments, we used $R=100$ and $S=100$, so that the dimension is given by $p=9010$. For the scalability experiment, we used $S=130$ and $R=1000, 5000, 8000, 10000$ so that the dimensions are $p=120010, 600010, 960010, 1200010$. For all experiments, we set $n=5000$ and $\lambda = R/100$.

\paragraph{Graph-guided fused lasso.}
The graph-guided fused lasso problem we consider is given by 
$$
\min_x \frac{1}{2}\|b-\mathsf{A}x\|_2^2 + \lambda \|D x\|_1,
$$
where $D$ is the difference matrix imposed by the network structure. 
The dataset for the graph-guided fused lasso experiments was generated following the transcription factor (TF) model of \citet{zhu2015augmented}. This is a simple gene network model with $J$ fully connected subnetworks of size $T$, where each subgroup has one TF with $T-1$ regulatory target genes.
Variables corresponding to TFs are sampled independently from $\mathcal{N}(0,1)$.  Variables for target genes are sampled so that each target gene and the corresponding TF has a bivariate normal distribution with correlation 0.7,  and these variables are conditionally independent given the TF. For $j$-th subnetwork, we chose
$$
x_{i} = \begin{cases}
    (-1)^{j+1} \left\lfloor\frac{j+1}{2}\right\rfloor & \text{if $j=1, \dotsc, J_a$} \\
    0 & \text{otherwise}
    \end{cases}, \quad i = (j-1)r+1, \dotsc, jr,
$$
where $J_a$ is the number of active groups. Response $b_i$ is sampled so that $b_i = \mathsf{A}x + \epsilon_i$, with $\epsilon_i \overset{\text{\tiny i.i.d.}}{\sim} \mathcal{N}(0, 100^2)$.
In addition to the edges comprised of fully-connected subnetworks, we added random edges between the active variables and the inactive variables. For each active variable, we added  edges connecting this variable and $J-1$ distinct inactive variables. For the convergence experiments, we used $T=10$, $J_a=20$, $J=1000$ so that the dimension $p$ is 10000. 
For the scalability experiment, we set $T=12$, and $J_a=20$. We selected $J = 10000, 50000, 80000, 100000$ to generate the dataset with $p = 120000, 600000, 960000, 1200000$, respectively. 
For all experiments, we set $n=5000$ and $\lambda = 1$.

\subsection{Convergence behavior}
We applied the algorithms to the
overlapping group lasso and graph-guided fused lasso with a quadratic loss, as described in Section \ref{sec:experiment}.
%
For the forward-backward (FB) splitting \eqref{eqn:generaliteration}, we set $C=\kappa K$, $|\kappa |\le 1$. We set $\rho_k= 0.9 \left(2 - \frac{\tau L_f}{2}\frac{1-(1-\kappa^2)\tau\sigma \|K\|_2^2}{1-\tau\sigma\|K\|_2^2} \right)$. 
Step sizes were chosen as $\tau = 0.9 \frac{2}{L_f}$ and $\sigma = 0.9 \frac{1}{\tau} \frac{1-\tau L_f/2}{1-(1-\kappa^2) \tau L_f/2}$, so that \eqref{eqn:generalparamrange} is satisfied.
For the acceleration \eqref{eqn:accgen}, we tested four cases: Algorithm LV ($A=B=0$), CV ($A=-K$, $B=K$), their ``midpoint'' ($A=-0.5K$, $B=0.5K$), and \citet{chen2014optimal} ($A=-K$, $B=0$). Number of iterations $N$ is set to 10000. For bounded (Corollary \ref{cor:bddgen}) and unbounded (Corollary \ref{cor:unbddgen}) cases, 
we found $(q, r)$ that minimizes $\frac{4 P \Omega_X^2}{k(k-1)}L_f + \frac{2 \Omega_X \Omega_Y (Q+1)}{N} \|K\|_2$ in \eqref{eqn:bddoptimal} and $\left(\frac{4PL_f}{N^2}+\frac{2Q\|K\|_2}{N}\right)\left(2+\frac{q}{1-q}+\frac{r+1/2}{1/2-r}\right)$ in \eqref{eqn:epsfbound}, respectively. Those minimizers were found using sequential least squares programming. 
As a benchmark, we also applied an inertial version of the forward-backward-forward (FBF) algorithm \citep{Combettes2012} 
as described in \citet{Bot2016}:
\begin{align}\label{eqn:fbf}
\begin{split}
\tilde{x}^{k+1} &= x^k - \tau \left(\nabla f(x^k) + K^T y^k\right) + \alpha_1 (x^k-x^{k-1}) \\
\tilde{y}^{k+1} &= \prox_{\tau h^*}(y^k + \tau K x^k + \alpha_1 (y^k - y^{k-1})\\
y^{k+1} &= \tilde{y}^{k+1} +\tau K (\tilde{x}^{k+1}-x^{k}) + \alpha_2 (y^k -y^{k-1}) \\
x^{k+1} &= \tilde{x}^{k+1} - \tau K^T (\tilde{y}^{k+1}-y^k) + \alpha_2 (x^k -x^{k-1}).
\end{split}
\end{align}
With $\alpha_1=\alpha_2=0$, \eqref{eqn:fbf} resembles Algorithm LV, but requires one more step per iteration; its convergence rate has not been established.
 
Figures \ref{fig:conv}(a), \ref{fig:conv}(b), \ref{fig:conv}(d), and \ref{fig:conv}(e) show 
the  convergence of the FB \eqref{eqn:generaliteration} 
with respect to the averaged sequence $\{(\bar{x}^N,\bar{y}^N)\}$, and the convergence of the accelerated FB algorithms \eqref{eqn:accgen} with respect to $\{(x^N,y^N)\}$. 
We plot the
gap between the primal objective value at $x^k$ and the ``optimal''
objective value versus iteration count $k$. Following \citet{loris2011generalization}, the reference ``optimal'' value was computed by running the accelerated LV algorithm with bounded parameters for 100000 iterations; this obtained the minimal value up to the point that the machine precision allows.
Figures \ref{fig:conv}(a) and \ref{fig:conv}(d) used parameters given by \eqref{eqn:bddparams}, which assumes $x^k$ and $y^k$ are bounded. 
This is true as long as $\|x^k\|_2<\Omega_X/\sqrt{2}$ and $\|y^k\|_2<\Omega_Y/\sqrt{2}$; 
we chose $\Omega_X=12$ and $\Omega_Y = 15$ for group lasso, and $\Omega_X = 141.4$ and $\Omega_Y = 305.9$ for graph-guided fused lasso.
The resulting iterates respected these bounds. Figures \ref{fig:conv}(b) and \ref{fig:conv}(e) used parameters given by \eqref{eqn:unbddparams}, which does not require $\Omega_X$ and $\Omega_Y$. 
The oscillation in the later part of Figures \ref{fig:conv}(a) and \ref{fig:conv}(b) are due to the machine precision of the GPUs.     
Since the reference optimal value was an order of $10^4$, the values in the oscillating region correspond to the 7th or 8th significant decimal digit of the objective value.

We observe that Theorems \ref{thm:ergodic} and \ref{thm:unbddgen} faithfully describes the convergence behavior. 
The convergence rates of the accelerated ones were close to $O(1/N^2)$, 
because in this experiment $L_f \gg \|K\|_2$. 
On the other hand, the base FB algorithms appear very close to the $O(1/N)$ line.
All of the optimal acceleration settings exhibit a very similar convergence behavior, which suggests that we 
have a good degree of freedom in choosing an optimal primal-dual algorithm. 
%

Figures \ref{fig:conv}(c) and \ref{fig:conv}(f) compare the non-ergodic convergence with respect to $\{(\tilde{x}^k,\tilde{y}^k)\}$ of the FB and FBF. 
The FB algorithms behave like $O(1/k)$ initially, and then converges faster than $O(1/k^2)$. 
This behavior is much faster than what is predicted by Theorem \ref{thm:nonergodic}.
On the contrary, the FBF algorithm stalls after a few hundred iterations. 
\begin{figure}[h!]
\centering
\begin{subfigure}[b]{0.42\textwidth}
 \includegraphics[width=\textwidth]{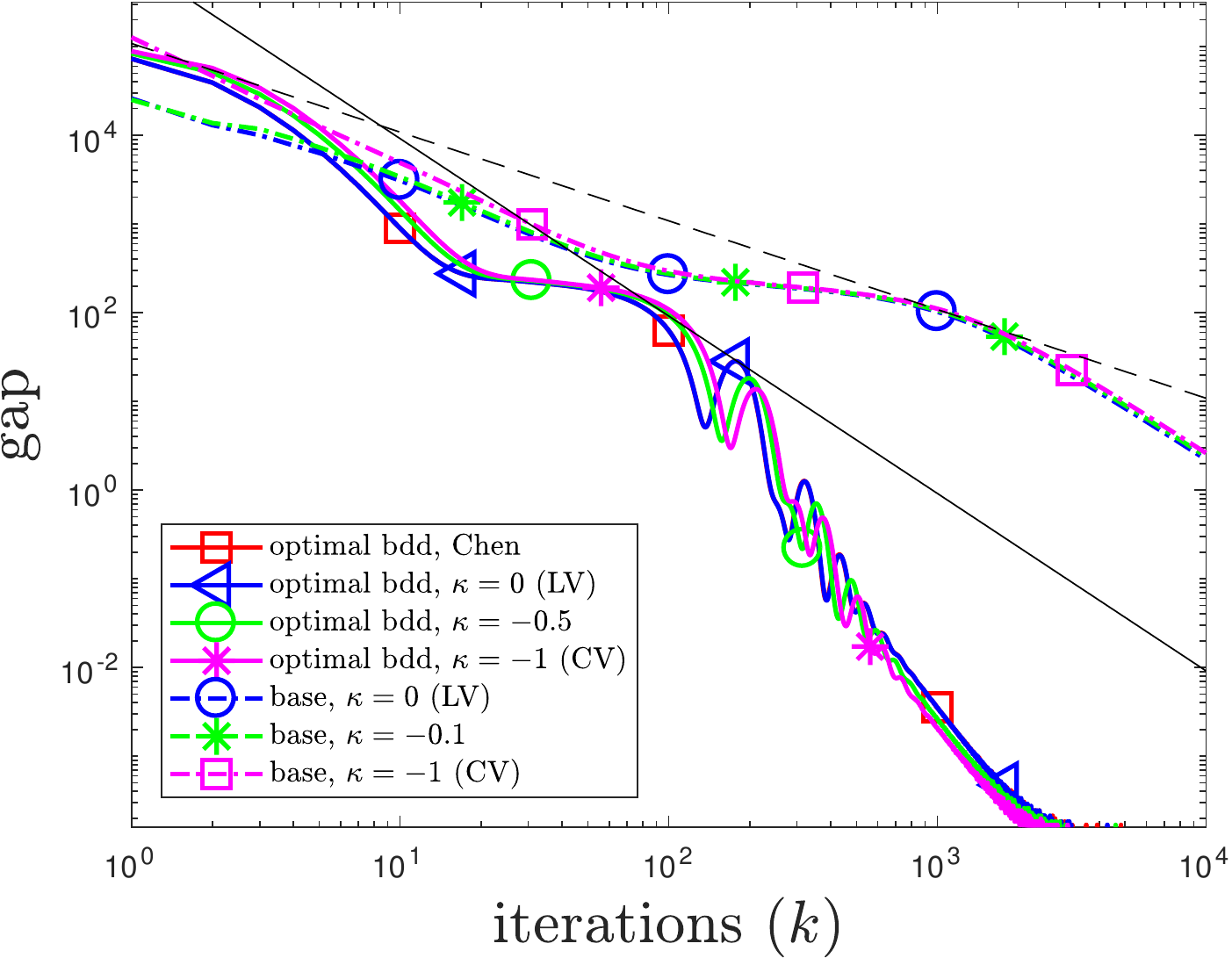}
 \caption{}\label{fig:grpbdd}
 \includegraphics[width=\textwidth]{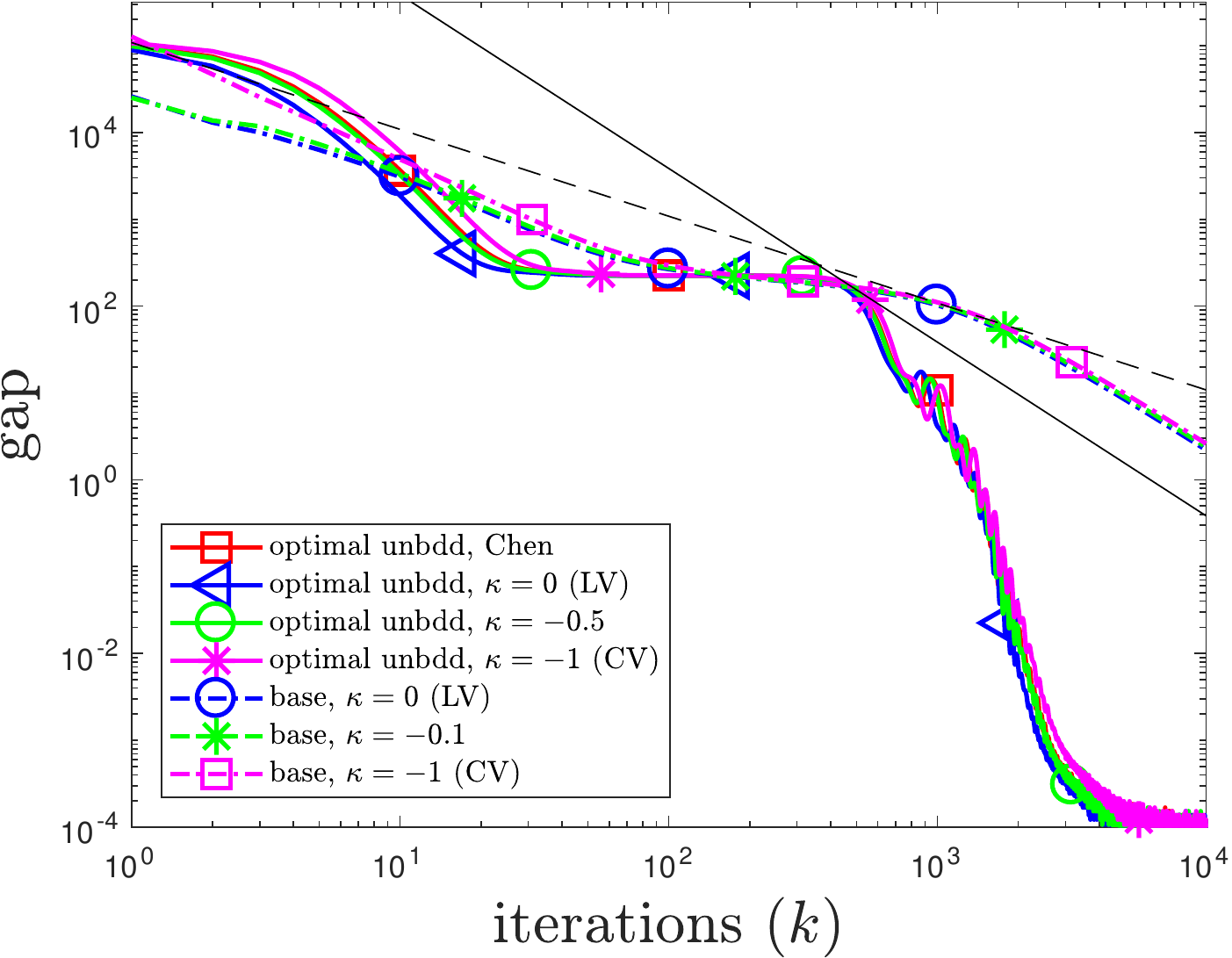}
 \caption{}\label{fig:grpunbdd}
 \includegraphics[width=\textwidth]{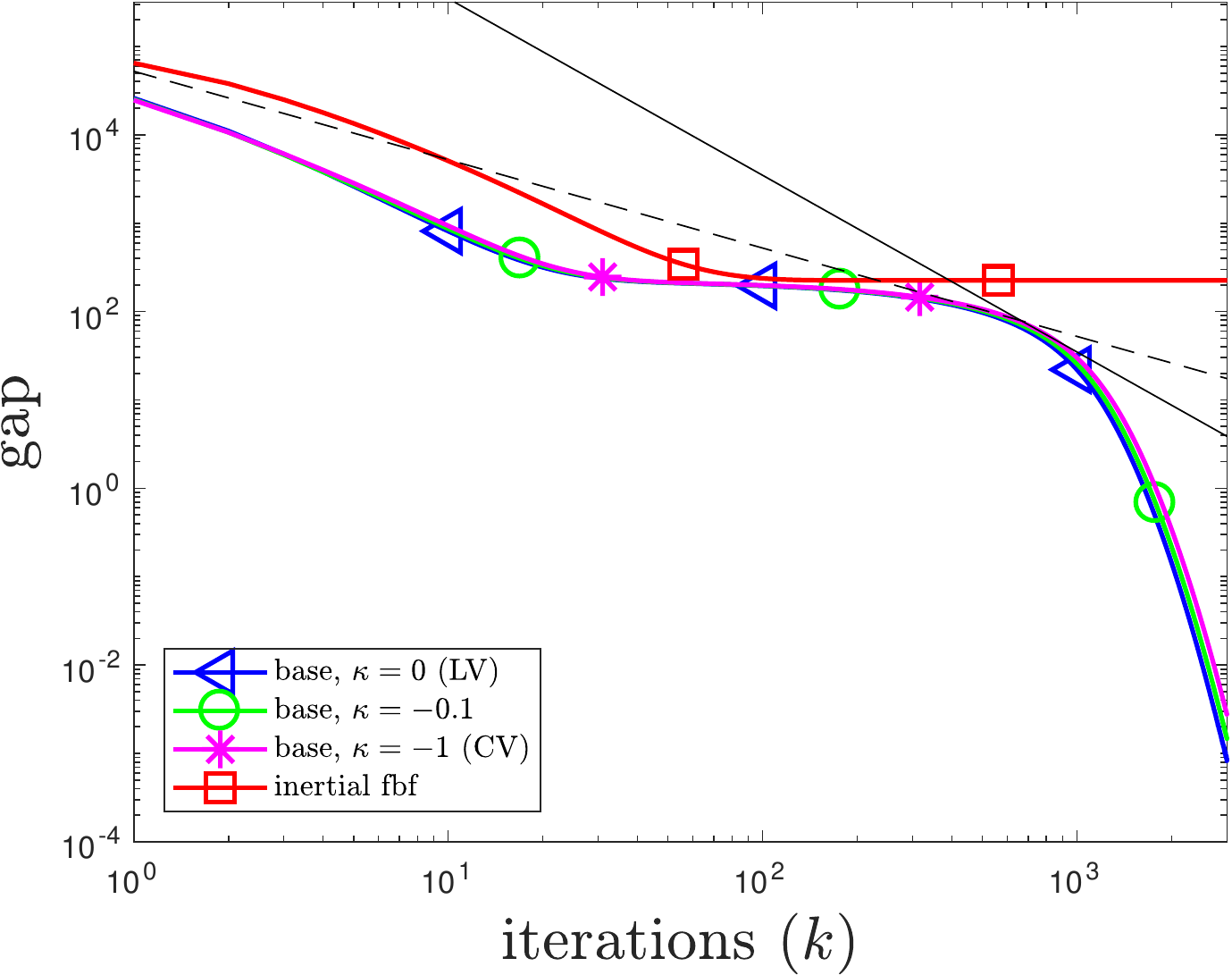}
 \caption{}\label{fig:grpnonerg}
\end{subfigure}
\begin{subfigure}[b]{0.42\textwidth}
 \includegraphics[width=\textwidth]{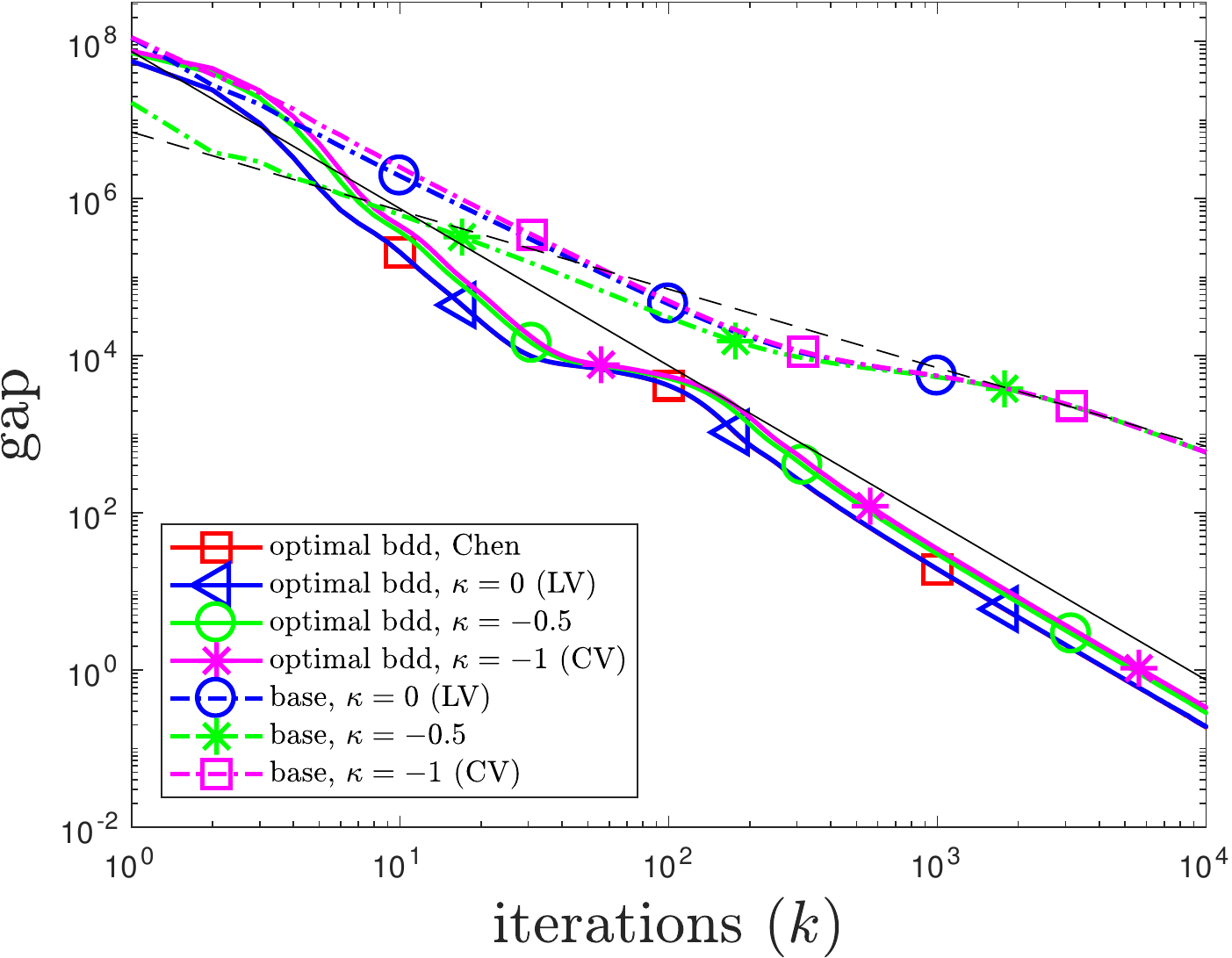}
 \caption{}\label{fig:zhubdd}
 \includegraphics[width=\textwidth]{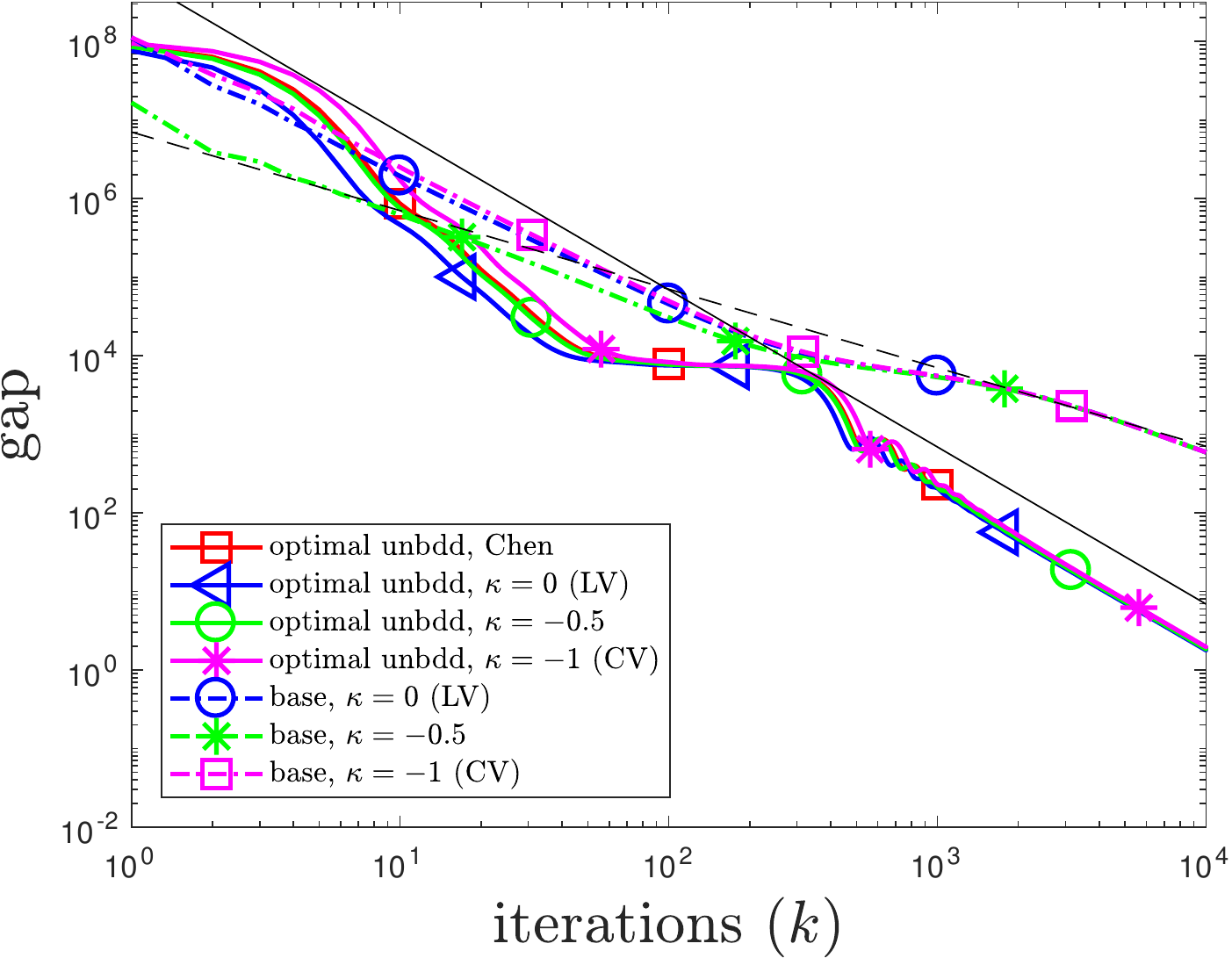}
 \caption{}\label{fig:zhuunbdd}
 \includegraphics[width=\textwidth]{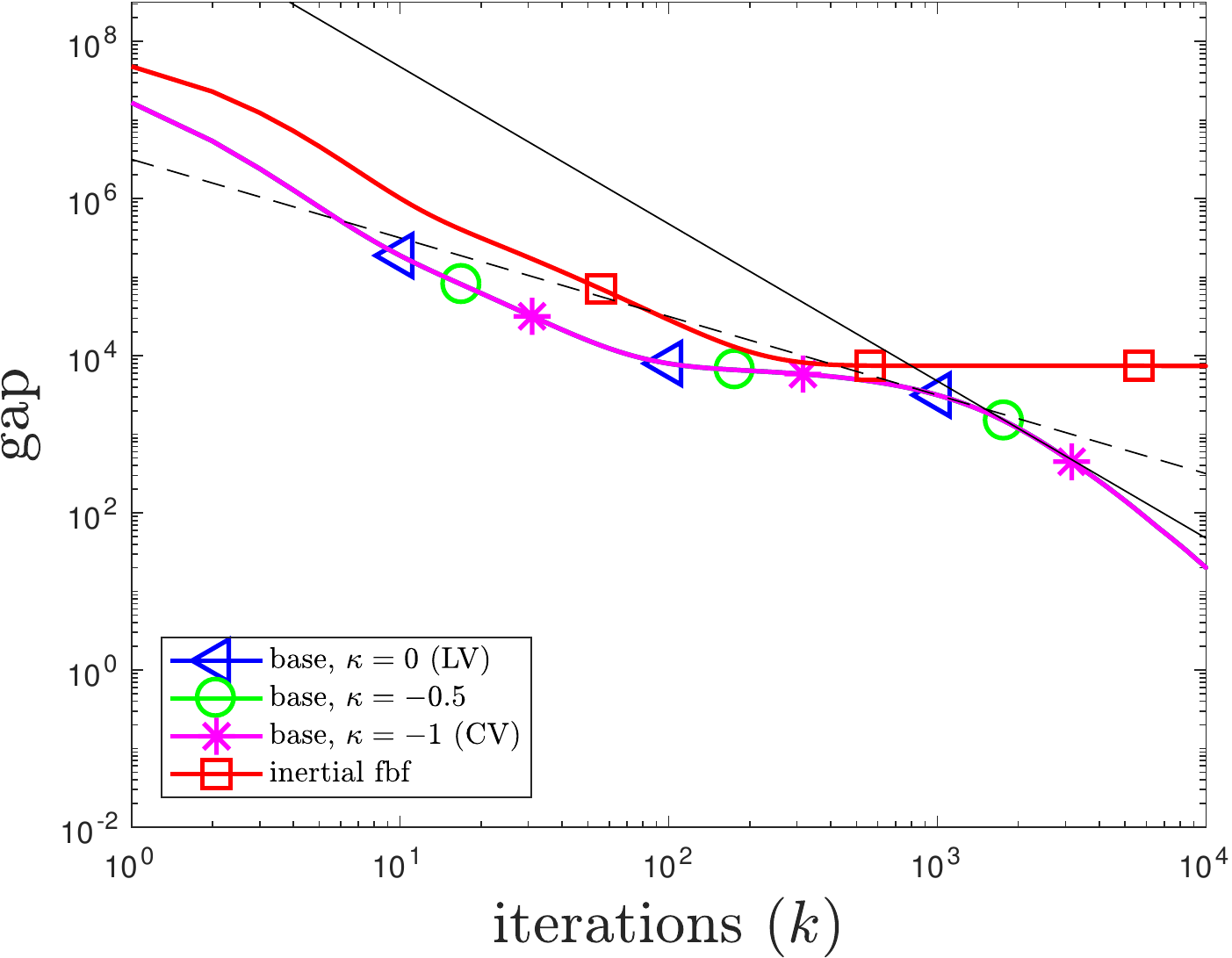}
 \caption{}\label{fig:zhunonerg}
\end{subfigure}
\caption{Convergence of the forward-backward (FB) algorithms generated by \eqref{eqn:generaliteration} and their accelerated variants \eqref{eqn:accgen} for a group lasso model (a-c) and a graph-guided fused lasso model (d-f).  
		 (a), (d), optimal acceleration with bounded parameter setting (``optimal'') with ergodic convergence of the FB algorithm (``base'').
		 (b), (e), optimal acceleration with unbounded parameter setting (``optimal'') with ergodic convergence of the FB algorithm (``base'').
     (c), (f), non-ergodic convergence of the FB (``base'') and inertial FBF (``inertial fbf'') algorithms. 
     Solid black lines represent $O(1/k^2)$ convergence, and dashed black lines represent $O(1/k)$ convergence. 
     }\label{fig:conv}
\end{figure}

%% file: sections/numerical_scale.tex
\subsection{Scalability}

To test the scalability of the studied algorithms, we consider the scenario that the number of features $p$ is so large that, for each sample, the features do not fit into the memory. In other words, the data matrix
$\mathsf{A} = [\mathsf{A}^{[1]}, \dotsc, \mathsf{A}^{[M]} ]$, where $\mathsf{A}^{[i]} \in \mathbb{R}^{n \times p_i}$, $\quad \sum_{i=1}^m p_i = p$,
is stored distributedly in $M$ devices. In this case, it is desirable  to also split the vectors $x \in \mathbb{R}^p$ conformally and store distributedly, i.e., $x = [x_{[1]}^T, \dotsc, x_{[M]}^T]^T$, $x_{[i]} \in \mathbb{R}^{p_i}$. For many instances of \eqref{eqn:primal} including the generalized lasso and group lasso, $l \gtrsim p$, so it is desirable to partition and store the dual variable $y \in \mathbb{R}^l$ likewise. i.e., $y = [y_{[1]}^T, \dotsc, y_{[M]}^T]^T$, $y_{[i]} \in \mathbb{R}^{l_i}$, $\sum_{i=1}^m l_i = l$.  To compute $K^T y$ and $Kx$ efficiently, it is desirable to also distribute rows and columns of $K$ across the devices, i.e., $K^T = [K_{[1]}^T, \dotsc, K_{[M]}^T]$ and $K = [K^{[1]}, \dotsc, K^{[M]}]$, where $K_{[i]} \in \mathbb{R}^{l_i \times p}$, and $K^{[i]} \in \mathbb{R}^{l \times p_i}$.  Duplicating $K$ does not incur too much cost, as $K$ is typically sparse. Then, we can carry out computation in a distributed fashion as follows. 

Suppose that device $i$ stores $\mathsf{A}_{[i]}$, $K_{[i]}$, $K^{[i]}$, $x_{[i]}$, and $y_{[i]}$. To compute $\mathsf{A} x$, we compute $\mathsf{A}^{[k]}x_{[k]}$ within each device, and aggregate the result in a master device. The communication cost required is $O(n)$. Computing $Kx$ is more complicated. 
Denote the submatrix made of the row $1+\sum_{i'=1}^{i-1} l_{i'}$ through $\sum_{i'=1}^i l_{i'}$ and the column $1+\sum_{j'=1}^{j-1} p_{j'}$ through $\sum_{j'=1}^j p_{j'}$ of $K$ by $K_{[i]}^{[j]}$. 
First, we compute $K_{[i]}^{[j]} x_{[j]}=:[Kx]_{ij}$.\deleted{, considering the sparse structure of $K$ within device $j$.} Then we transfer nonzero values in each $[Kx]_{ij}$ to device $i$. Finally, within device $i$, we aggregate $[Kx]_{ij}$ over $j$. When the number of nonzero elements in $K$ is $O(p)$, which is the case for both overlapping group lasso and graph-guided fused lasso, the communication cost is $O(M p)$ in the worst case. 
This type of distribution is especially suitable for multi-GPU platforms. 
We solved the model problems using TensorFlow \cite{tensorflow} v1.2, which deals with inter-GPU communications automatically. 

Each experiment was conducted for 1100 iterations with time recorded every 100 iterations. This is repeated three times. We discarded the result for the first 100 iterations, as this figure includes the time elapsed to build computation graphs. We computed average time per 100 iterations and their standard deviations. 
Table \ref{tab:scalability} shows that our distributed implementation is highly scalable across multiple GPUs. 
The algorithm runs faster with more GPUs in general; for the data that do not fit in the memory, 
it only requires more GPUs.

\begin{table}[ht]
\caption{Scalability of the distributed version of \eqref{eqn:generaliteration} for graph-guided fused lasso and group lasso models. Time was measured in seconds per 100 iterations. Standard deviations are listed in parentheses. Any cell with missing values indicates that the experiment failed to run due to lack of memory.}
\label{tab:scalability}
\vskip -0.2in
\begin{center}
\fontsize{10}{12}\selectfont
\begin{sc}
Graph-guided fused lasso
\vskip 0.1in
\begin{tabular}{rrrrrrrrrr}\hline
         & \#GPUs  & 1     & 2      & 3      & 4      & 5      & 6      & 7      & 8      \\
\#groups & $p$     &       &        &        &        &        &        &        &        \\ \hline
10000    & 120000  & 4.895 & 3.801  & 3.274  & 2.468  & 2.081  & 1.739  & 1.584  & 1.518  \\
         &         & (0.019) & (0.048)  & (0.027)  & (0.021)  & (0.029)  & (0.025)  & (0.023)  & (0.014)  \\
50000    & 600000  &       & 20.631 & 13.779 & 11.962 & 10.124 & 8.568  & 7.699  & 6.520  \\
         &         &       & (0.253)  & (0.309)  & (0.126)  & (0.031)  & (0.058)  & (0.053)  & (0.050)  \\
80000    & 960000  &       &        & 22.695 & 16.957 & 13.712 & 11.559 & 10.343 & 10.828 \\
         &         &       &        & (0.288)  & (0.302)  & (0.140)  & (0.124)  & (0.133)  & (0.056)  \\
100000   & 1200000 &       &        &        & 20.517 & 16.190 & 15.590 & 11.704 & 12.498 \\
         &         &       &        &        & (0.166)  & (0.227)  & (0.170)  & (0.148)  & (0.145) \\ \hline
\end{tabular}
\newline\newline\newline
Overlapping group lasso
\vskip 0.1in
\begin{tabular}{rrrrrrrrrr}\hline
         & \#GPUs  & 1     & 2      & 3      & 4      & 5      & 6      & 7      & 8      \\
\#groups & $p$     &       &        &        &        &        &        &        &        \\ \hline
1000     & 120010  & 4.828 & 4.156  & 2.973  & 2.465  & 2.102  & 1.853  & 1.591  & 1.538  \\
         &         & (0.015) & (0.057)  & (0.034)  & (0.014)  & (0.015)  & (0.012)  & (0.014)  & (0.015)  \\
5000     & 600010  &       & 19.312 & 13.670 & 10.164 & 8.374  & 7.369  & 6.727  & 5.960  \\
         &         &       & (0.075)  & (0.059)  & (0.055)  & (0.029)  & (0.040)  & (0.029)  & (0.038)  \\
8000     & 960010  &       &        & 22.792 & 17.044 & 14.722 & 12.671 & 10.866 & 10.103 \\
         &         &       &        & (0.228)  & (0.101)  & (0.107)  & (0.157)  & (0.110)  & (0.080)  \\
10000    & 1200010 &       &        &        & 22.210 & 16.658 & 15.386 & 14.088 & 11.689 \\
         &         &       &        &        & (0.273)  & (0.049)  & (0.098)  & (0.104) & (0.105)  \\ \hline
\end{tabular}
\end{sc}
\end{center}
\vskip -0.2in
\end{table}

%% file: sections/conclusion.tex
\section{Conclusion}\label{sec:conclusion}
In this paper, we have provided a unified view to Algorithms CV and LV, two classes of primal-dual algorithms for a convex composite minimization problem based on monotone operator theory. 
This unification suggests a continuum of forward-backward operator splitting algorithms for this important optimization problem having many applications in statistics.
It is also this unified understanding that enables us to establish the
$O(L_f /N^2 + \|K\|_2/N)$ optimal accelerations of Algorithms CV and
LV (and those in between), as well as the $O(1/N)$ and $o(1/\sqrt{k})$
convergence rates for the full regions of convergence of their
unaccelerated counterparts. A practical implication of this understanding is that we bring these algorithms to the same arena: as they share the same convergence rate, other factors such as the ability of choosing wider step sizes can be fairly compared in empirical settings.
Thus practitioners now possess more degrees of freedom in choosing from a suite of algorithms with theoretical guarantees.

The simplicity of the algorithms proposed and analyzed here also enables us to implement their distributed multi-GPU version almost painlessly using existing packages. This contrasts to our previous works \citep{yu2015high,lee2017large}, which resort to exploiting the structure of the matrix $K$ in \eqref{eqn:primal}. 

%% file: sections/examples.tex
\section{Flexibility of formulation (\ref{eqn:primal})}\label{sec:examples}

We discuss two cases in which the $f$ and $g$ in \eqref{eqn:primal} does not directly match the two terms in \eqref{eqn:objective}.

\paragraph{More than one penalty.}
When \eqref{eqn:objective} involves more than one penalty, the problem can be formulated as \eqref{eqn:primal} by augmenting the dual variable. 
Suppose we solve the following penalized regression problem
\[
	\min_{x \in \mathbb{R}^p} \quad
\sum_{i=1}^n l_i(a_i^T x, b_i) + H_1(D_1x) + H_2(D_2x).
\]
Then we can set 
\[
f(x) = \sum_{i=1}^n l_i(a_i^T x, b_i),
\quad
h(y_1,y_2) = H_1(y_1)+H_2(y_2), 
\quad
K=\begin{bmatrix}D_1 \\ D_2\end{bmatrix},
\quad
y=\begin{bmatrix}y_1 \\ y_2\end{bmatrix}
.
\]
It is easy to verify that $\prox_h(v_1,v_2)=(\prox_{H_1}(v_1), \prox_{H_2}(v_2))^T$ due to separability of $h$.
For example, consider the latent group lasso problem \citep{jacob2009group}.
The latent group lasso selects groups less conservatively than the original group lasso \citep{yuan2006model}, and allows overlaps.
The penalty is defined as
\[ 
H(x) = \inf_{v_{[g]} \in \mathbb{R}^{|[g]|}, D^T v = x} \sum_{g=1}^{\mathcal{G}} \lambda_g \|v_{[g]}\|_q,
\]
where $[g]$ and $D$ are the group index set and the membership matrix as discusses in Section \ref{sec:intro} for the original group lasso.
Thus the latent group lasso problem can be written as
\[
	\min_{x,v} f(x) + h(v) + \delta_{\{0\}}(x-D^Tv),
\]
where 
$h(v)=\sum_{i=1}^{\mathcal{G}}\lambda_g\|v_{[g]}\|_q$
and
$\delta_S$ is the indicator function for set $S$ so that $\delta_S(u)=0$ if $u\in S$ and $\delta_S(u)=+\infty$ otherwise.
Let $z = (x^T, v^T)^T$, $\tilde{f}(z) = f(\begin{bmatrix} I & 0 \end{bmatrix} z)$, $\tilde{h}(y_1, y_2) = h(y_1) + \delta_{\{0\}}(y_2)$, and $K = \begin{bmatrix} 0 & I \\ I & -D^T \end{bmatrix}$. 
We have an equivalent formulation 
$$
\min_z \tilde{f}(z) + \tilde{h}(Kz),
$$
It has the form of \eqref{eqn:primal}. 
Note that both $h$ and $\delta_{\{0\}}$ are proximable.

\paragraph{Nonsmooth losses.}
When the loss function $l_i$ in \eqref{eqn:objective} does not have Lipschitz gradients yet is closed, proper, and convex, 
a split-dual formulation \citep{Nesterov:MathematicalProgramming:2004} can be utilized.
This includes the case where the loss is not differentiable (e.g. hinge loss).
To cope with this, we 
exploit the saddle-point representation \eqref{eqn:saddlepoint} of \eqref{eqn:primal},
and \emph{dualize} the loss function in addition to the penalty. 
That is, express
$\sum_{i=1}^n l_i (a_i^T x; b_i) = \sup_{w \in \mathbb{R}^n} \langle \mathsf{A} x, w \rangle - \sum_{i=1}^n l_i^{*} (w_i; b_i)$,
yielding
\begin{equation}\label{eqn:splitdual}
\min_x \max_{y,w}~ \langle Dx, y \rangle + \langle \mathsf{A}x, w \rangle - \left(\sum_{i=1}^n l_i^{*} (w_i; b_i)+ H^{*}(y)\right).
\end{equation}
In terms of \eqref{eqn:saddlepoint}, $f(x) \equiv 0$, $K = [ D^T,  \mathsf{A}^T]^T$, $h^{*} (y,w) = H^{*}(y) + \sum_{i=1}^n l_i^{*} (w_i; b_i)$. 
Because $h^{*}$ is separable in $y$ and $w$, we have $\prox_{\sigma h^{*}}(u, v_1,\dotsc,v_n)=(\prox_{\sigma H^{*}}(u), \prox_{\sigma l_1^{*} (\cdot; b_1)} (v_1),\dotsc,\prox_{\sigma l_n^{*} (\cdot; b_n)} (v_n))$.
The cost is that the number of dual variables increases by $n$.
For example, in the linear support vector machine, the proximity operator for the hinge loss $l_i(\cdot;b_i)=\max(0,1-b_i \cdot)$ is given by
$\prox_{\sigma l_i^{*}}(v_i) = \max(\min(v_i-\sigma b_i,0),-b_i)$.
Thus computation of $\prox_{\sigma h^{*}}$ can be conducted in parallel for each element of $v=(v_1,\dotsc,v_n)$.
Note that this formulation 
is not limited to the separable losses in \eqref{eqn:objective}. For example, in the square-root lasso \cite{belloni2011square}, we solve 
\begin{align}
\min_x \|\mathsf{A}x-b\|_2 + H(Dx) = \min_x \max_{y, w: \|w\|_2 \le 1} \langle Dx, y \rangle + \langle \mathsf{A}x, w \rangle - \left(\langle b, w \rangle + H^* (y) \right),
\end{align}
yielding $f(x) \equiv 0$, $K = [D^T, \mathsf{A}^T]^T$, $\prox_{\sigma h^*}(u, v) = \big(\prox_{\sigma H^*}(u), P_{\mathcal{B}_2} (v - 2 \sigma b) \big)$, 
where $P_{\mathcal{B}_2}(\cdot)$ denotes the projection to the unit $\ell_2$-ball.
Note this split-dual technique can be also applied to the PDHG \citep{zhu2008efficient,Esser:SiamJournalOnImagingSciences:2010,chambolle2011first,He:SiamJournalOnImagingSciences:2012,Chambolle:MathematicalProgramming:2015}, whose iteration is given by
\begin{align*}
    x^{k+1}         &= \prox_{\tau f}(x^k - \tau K^T y^k ), \\
    \tilde{x}^{k+1} &= 2 x^{k+1} - x^k, \\
    y^{k+1}         &= \prox_{\sigma h^{*}} (y^k + \sigma K \tilde{x}^{k+1}).
\end{align*}
For the same choices of $f$, $K$, and $h^*$, PDHG coincides with Algorithm CV. 
For losses with Lipschitz-continuous gradients (e.g., $f(x)=\frac{1}{2}\|\mathsf{A}x-b\|_2^2$), however, Algorithms \eqref{eqn:generaliteration} and \eqref{eqn:accgen} can proceed more efficiently (using $\nabla f$)  without dualization.
To efficiently apply PDHG, on the contrary, one may have to dualize the loss terms as \eqref{eqn:splitdual} unless the proximity operator of $f(x) = \sum_{i=1}^n l_i(a_i^T x;b_i)$ is simple to evaluate. 
%


%
%

%% file: sections/numerical_stoc.tex
\section{Additional numerical experiments}\label{sec:numerical:appendix}

\subsection{Stochastic optimal acceleration}
We illustrate an actual convergence behavior of the 
optimal stochastic algorithm \eqref{eqn:stocgen} for the group lasso and graph-guided fused lasso model problems in the main text. 
The estimate  $\hat{\mathcal{F}}(x^k)$ is computed by $\nabla f (\mathcal{M} x^k)$, where $\mathcal{M}$ is a diagonal matrix where each diagonal entry is independently chosen as $1/p$ with probability $\pi$, and 0 with probability $1-\pi$. 
This strategy meets the assumption \eqref{eqn:expect}. 

The convergence behavior of the stochastic algorithm is illustrated in Figure \ref{fig:stoc}. Figures \ref{fig:grpstocbdd} and \ref{fig:zhustocbdd} show the result of \eqref{eqn:stocgen} with parameters \eqref{eqn:stocbddparams} for the group lasso and graph-guided fused lasso problems, respectively. Figures \ref{fig:grpstocunbdd} and \ref{fig:zhustocunbdd} show those with parameters given by \eqref{eqn:stocunbddparams}. Note that for the assumption \eqref{eqn:A1} to hold, both cases need estimates of $\Omega_X$ and $\Omega_Y$. We chose $\pi=0.2$.  
For the simplicity of illustration, we used $\chi = 3 \times 10^5$ for the overlapping group lasso and $\chi = 10^7$ for the graph-guided fused lasso. 
In \eqref{eqn:stocunbddparams}, $\tilde{R}$ was set to 10 for overlapping group lasso and 100 for graph-guided fused lasso. 
The horizon $N$ was set to 10000 for all cases. In \eqref{eqn:stocbddparams} and  \eqref{eqn:stocunbddparams}, $q$, $r$, $s$, and $t$ were chosen to minimize the error bounds $\mathcal{C}_0(N)$ in Corollary \ref{cor:stocbdd} and $\frac{4PL_f}{N(N-1)} + \frac{2 Q \|K\|_2}{N}+ \frac{2 \chi /\tilde{R}}{\sqrt{N-1}}$ in Corollary \ref{cor:stocunbdd}, respectively, in a similar fashion to the deterministic counterparts. 
For a comparison, we included cases with parameters chosen for the deterministic setting \eqref{eqn:bddparams} and \eqref{eqn:unbddparams} but with stochastic estimation of gradients. 
In Figure \ref{fig:stoc}, 
the convergence of the stochastic algorithms is slow initially
because the step sizes $\tau_k$ and $\sigma_k$ are very small for small $k$
due to the presence of an $N^{3/2}$ term in their denominators,
but they eventually converge faster than the $O(1/k)$ rate for both bounded and unbounded parameter selections. 
(Also note the log-log scale of the plots.) 
While Corollaries \ref{cor:stocbdd} and \ref{cor:stocunbdd} guarantee the optimal rate for $A=-K$ (corresponding to CV if $B=K$ and \citet{chen2014optimal} if $B=0$), the choice $A = - \kappa K$, $B = \kappa K$ with $0\le\kappa<1$ (corresponding to LV and ``in-between'') also exhibited a similar  convergence behavior. 
On the contrary, for the ``deterministic'' choice of the parameters the algorithm diverged.

\begin{figure}[ht!] 
 \begin{center}
\begin{subfigure}[b]{0.42\textwidth}
  \includegraphics[width=\textwidth]{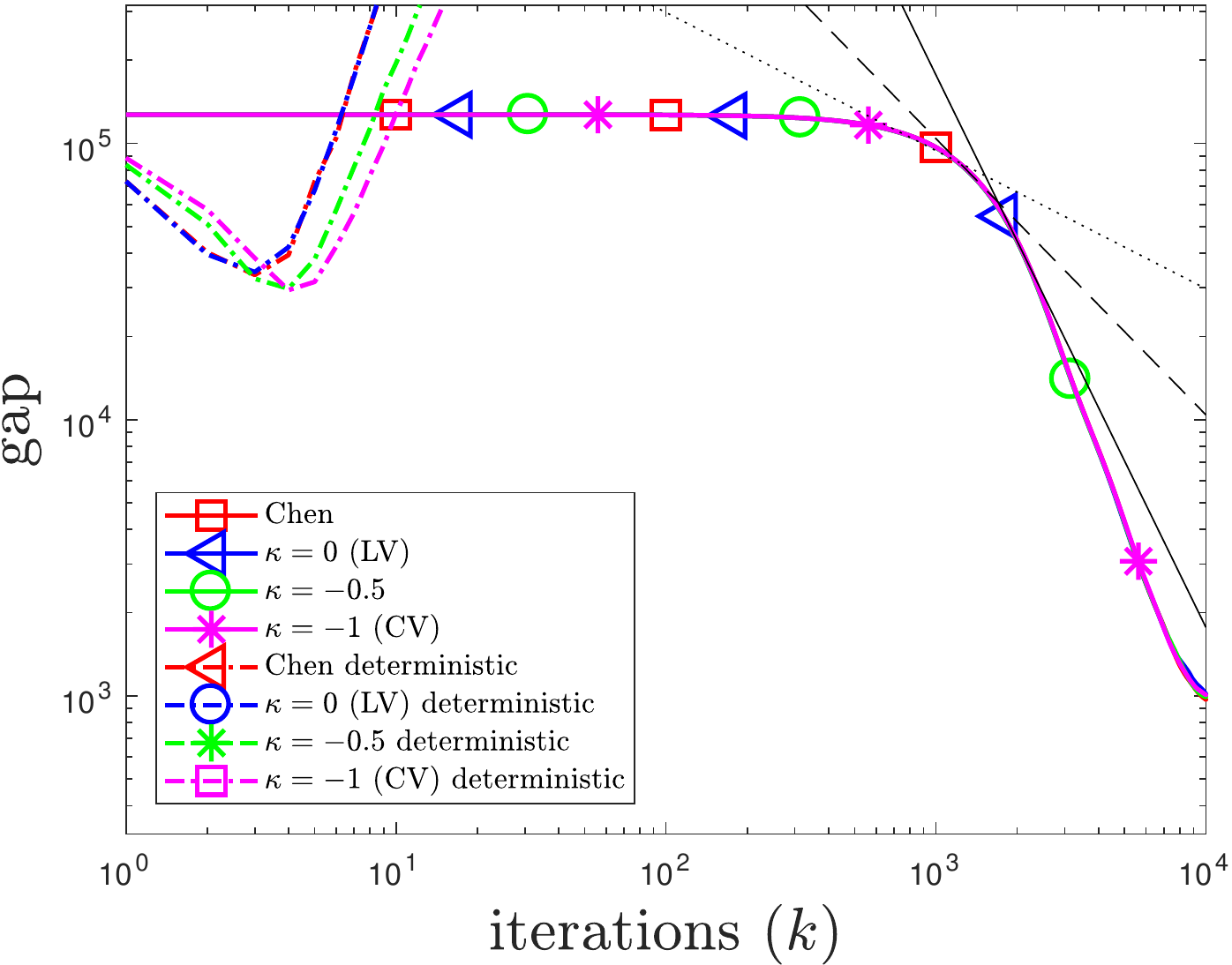}
  \caption{}\label{fig:grpstocbdd}
  \includegraphics[width=\textwidth]{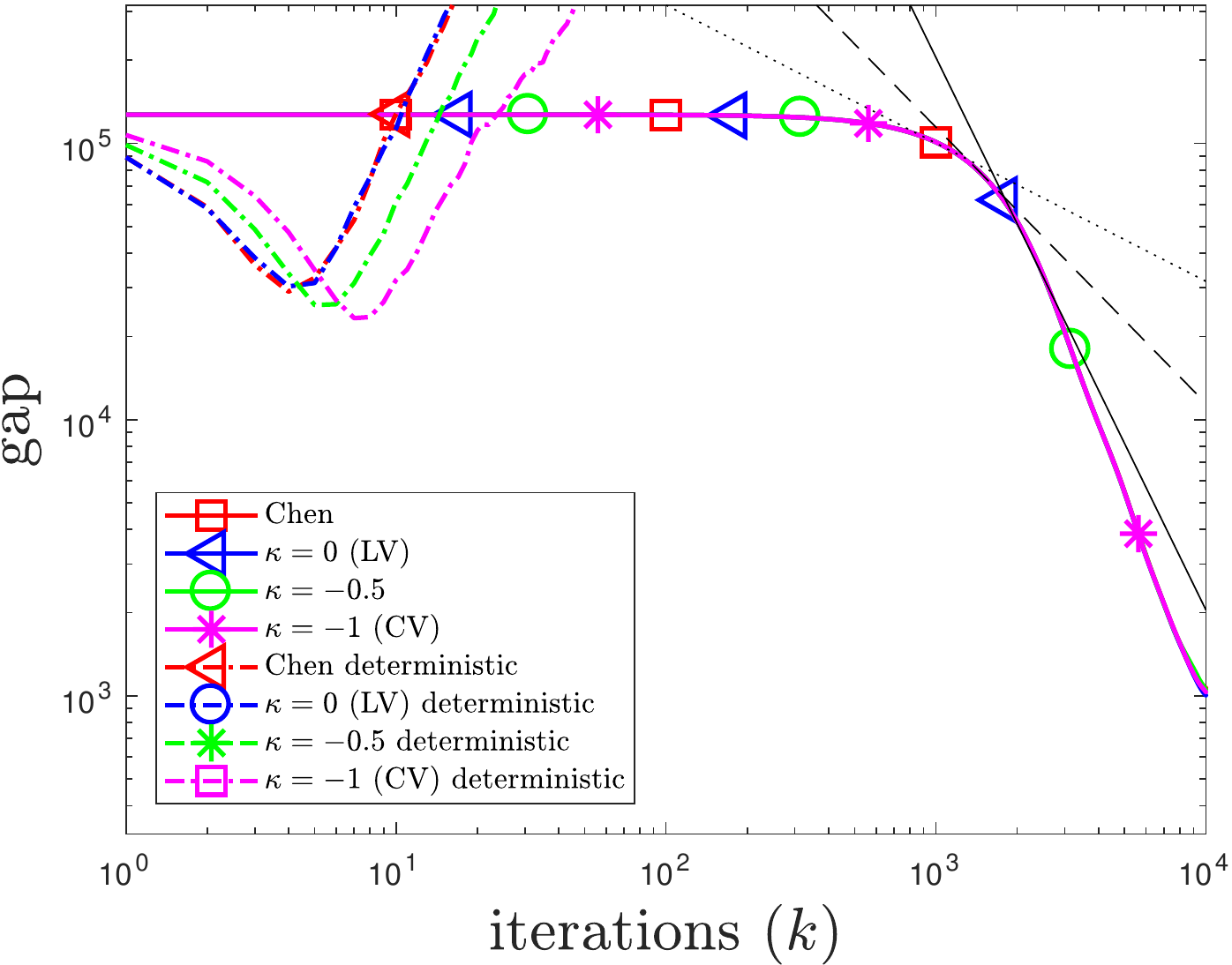}
  \caption{}\label{fig:grpstocunbdd}
\end{subfigure}
\begin{subfigure}[b]{0.42\textwidth}
  \includegraphics[width=\textwidth]{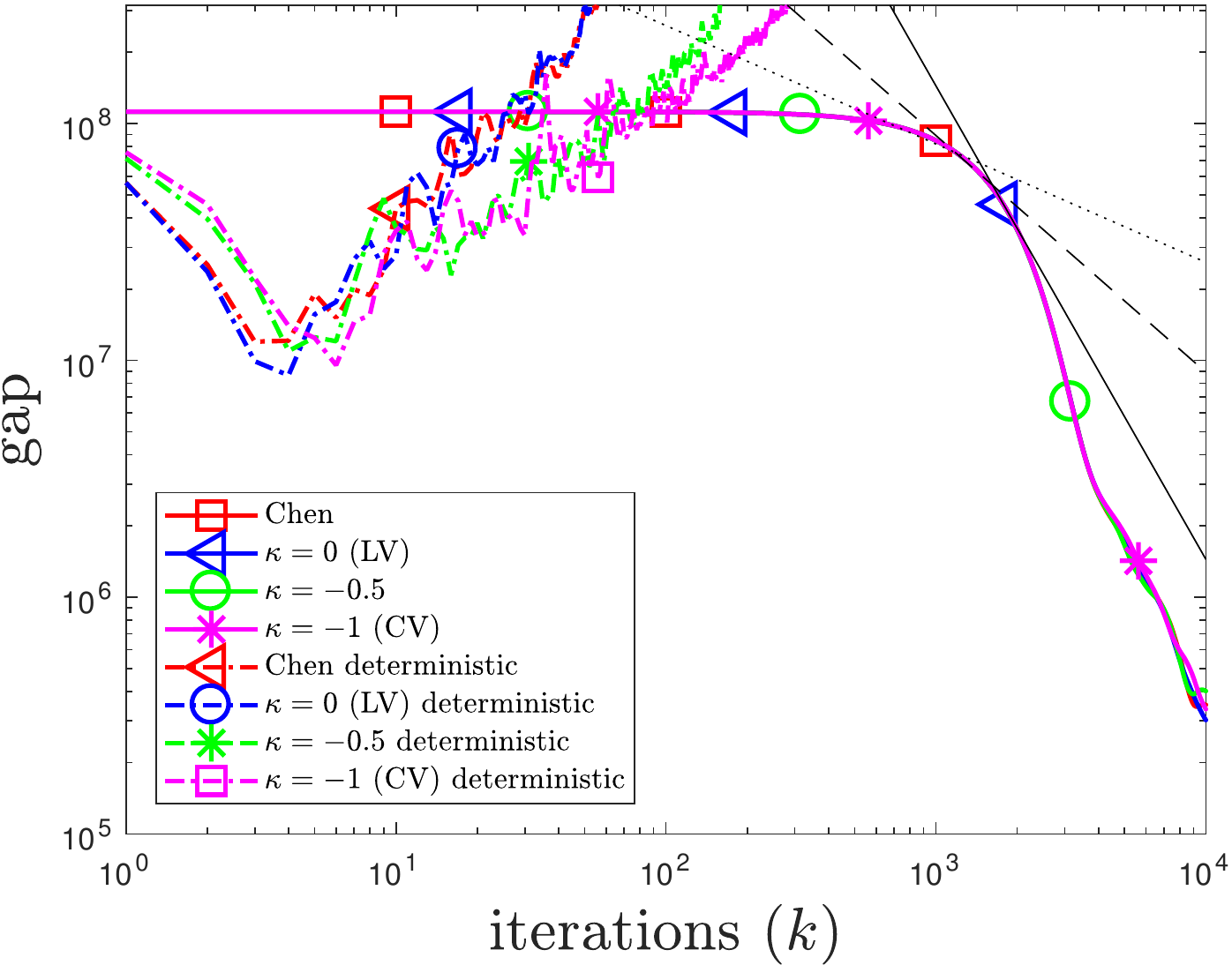}
  \caption{}\label{fig:zhustocbdd}
  \includegraphics[width=\textwidth]{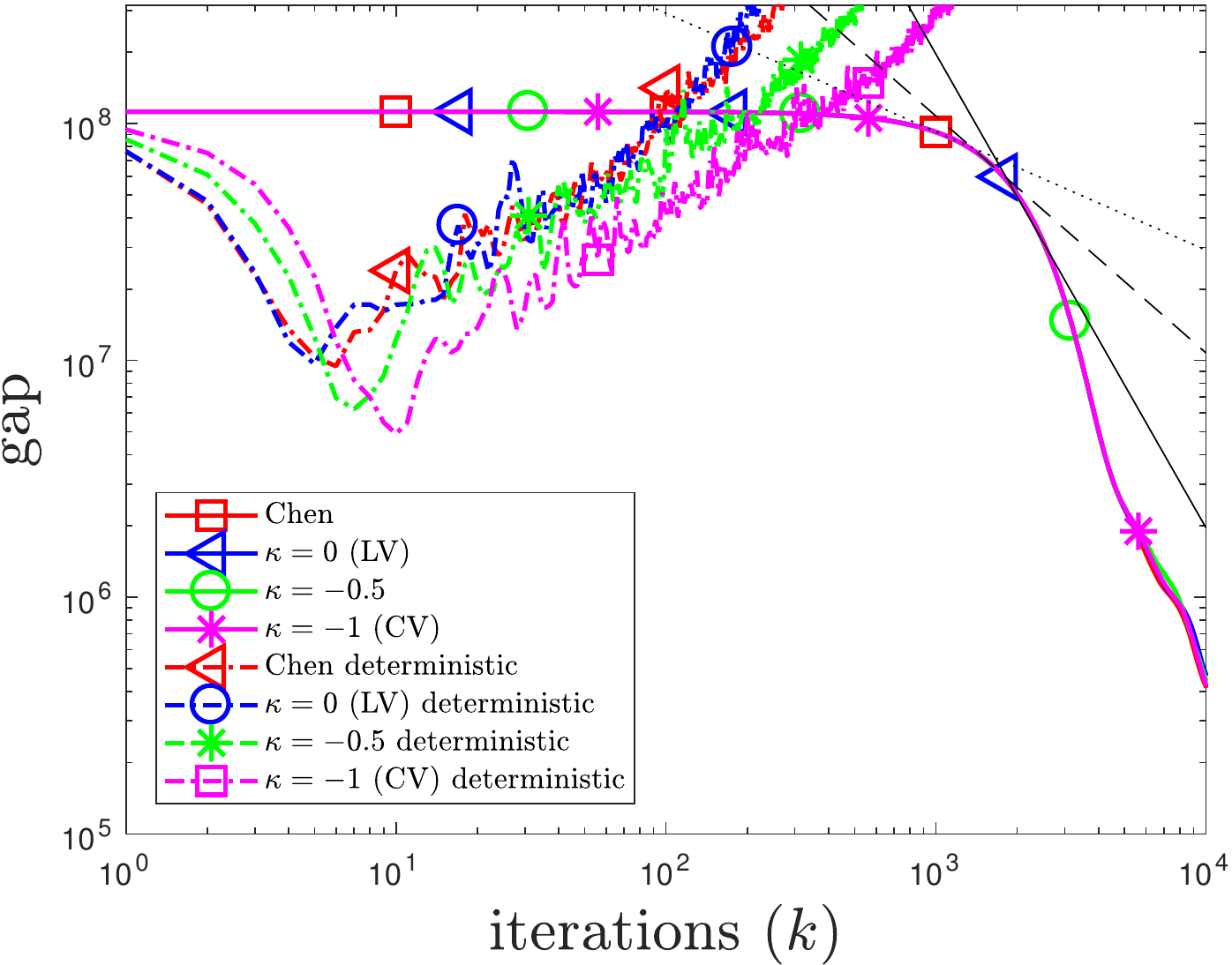}
  \caption{}\label{fig:zhustocunbdd}
\end{subfigure}
 \caption{\fontsize{11}{12}\selectfont Convergence of optimal rate stochastic algorithm for a group lasso model (a-b) and a graph-guided fused lasso model (c-d). (a), (c), optimal rate stochastic algorithm assuming bounded domain \eqref{eqn:stocbddparams} (``optimal'') compared to ergodic convergence of the FB algorithm. (b), (d), optimal rate stochastic algorithm with parameters in \eqref{eqn:stocunbddparams}. The cases labeled ``deterministic'' in the legend denote the deterministic-case parameters given by \eqref{eqn:bddparams} for bounded case and \eqref{eqn:unbddparams} for unbounded case.
 Solid black lines, dashed black lines, and dotted black lines represent $O(1/k^2)$, $O(1/k)$, and $O(1/\sqrt{k})$ convergence, respectively.
 }
 \label{fig:stoc}
 \end{center}
 \vskip -0.25in
 \end{figure}

%% file: sections/numerical_latent.tex
\subsection{Latent group lasso}\label{sec:lgl}

Here we present the numerical experiment results for latent group lasso described in Appendix \ref{sec:examples} in deterministic settings. 
We used the same dataset as in the overlapping group lasso model in Section \ref{sec:experiment}. The convergence behavior is depicted in Figure \ref{fig:lgl_conv}. Scalability is demonstrated in Table \ref{tab:lgl_scalability}. 
Both results exhibit behaviors similar to those in Section \ref{sec:numerical}: 
convergence rates of the accelerated algorithm were close to $O(1/N^2)$, beating their unaccelerated, base counterparts (forward-backward);
the forward-backward-forward (FBF) algorithm stalls after a few hundred iterations;
there is no essential difference among the continuum of the optimal algorithms, leaving a variety of possibilities for choosing a particular algorithm, etc.

\begin{figure}[H]
\centering
        \begin{subfigure}[b]{0.475\textwidth}
            \centering
            \includegraphics[width=\textwidth]{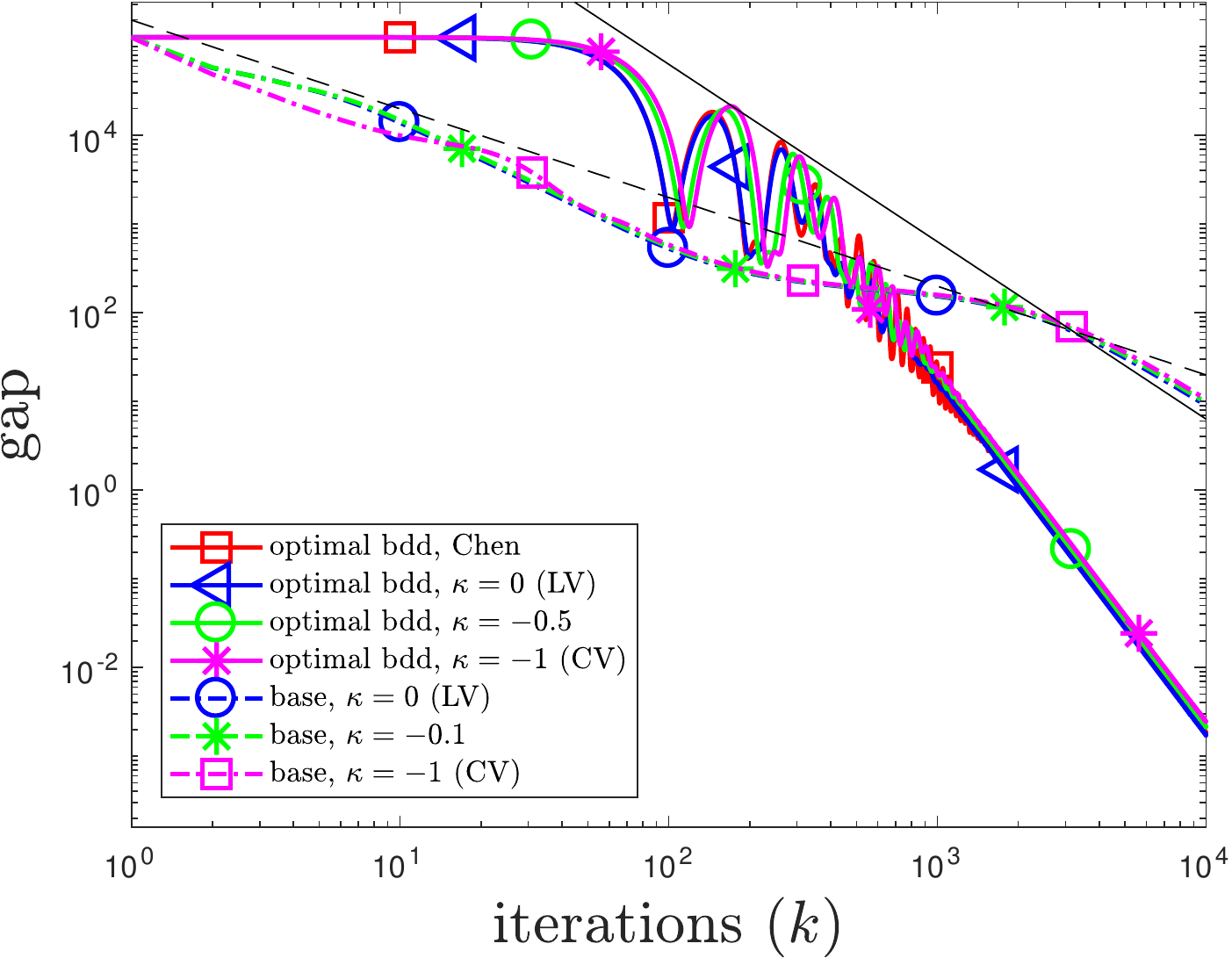}
            \caption[]%
            {}
            \label{fig:lgl_bdd}
        \end{subfigure}
        \hfill
        \begin{subfigure}[b]{0.475\textwidth}
            \centering
            \includegraphics[width=\textwidth]{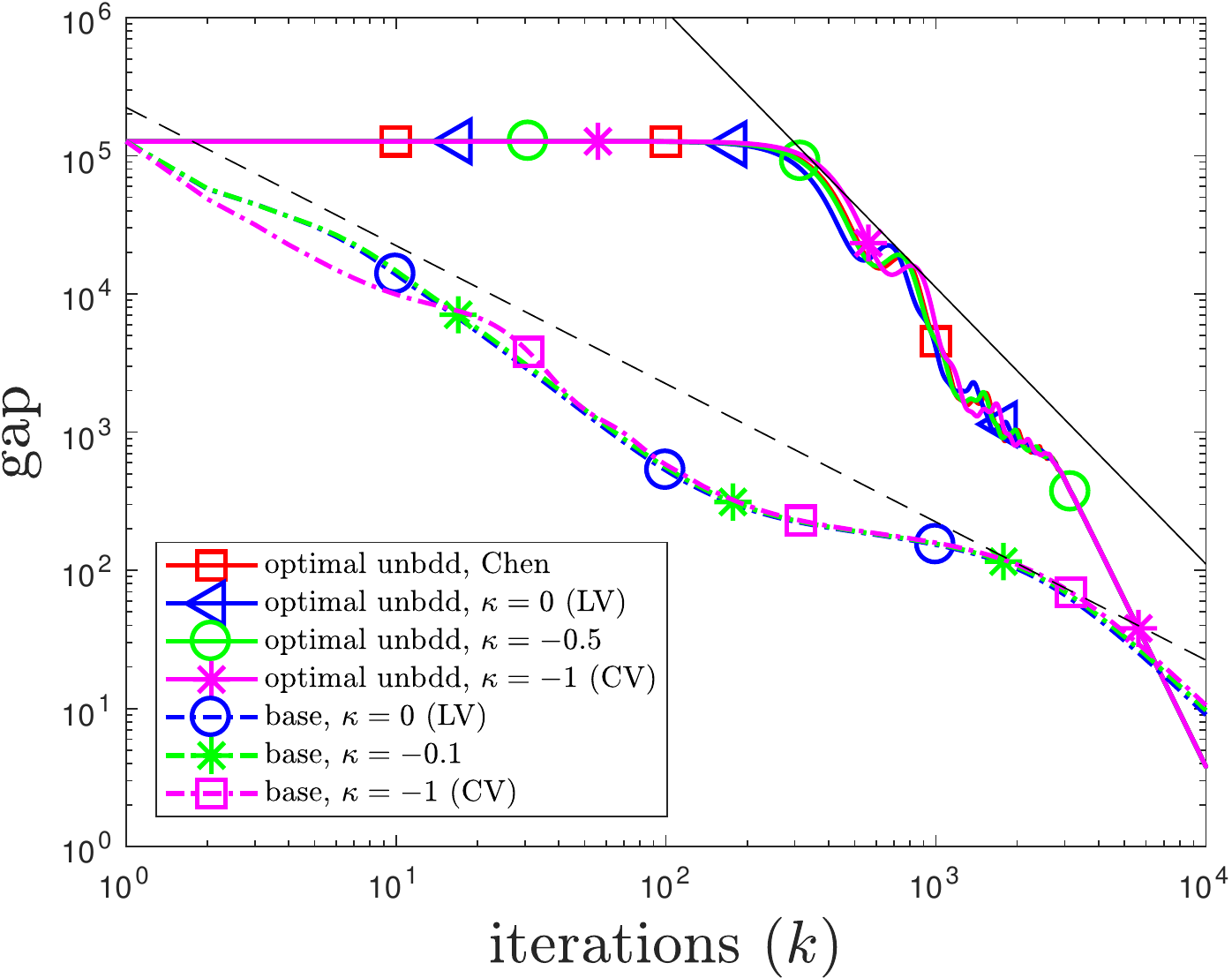}
            \caption[]%
            {}
            \label{fig:lgl_unbdd}
        \end{subfigure}
        \vskip\baselineskip
        \begin{center}
        \begin{subfigure}[b]{0.475\textwidth}
            \centering
            \includegraphics[width=\textwidth]{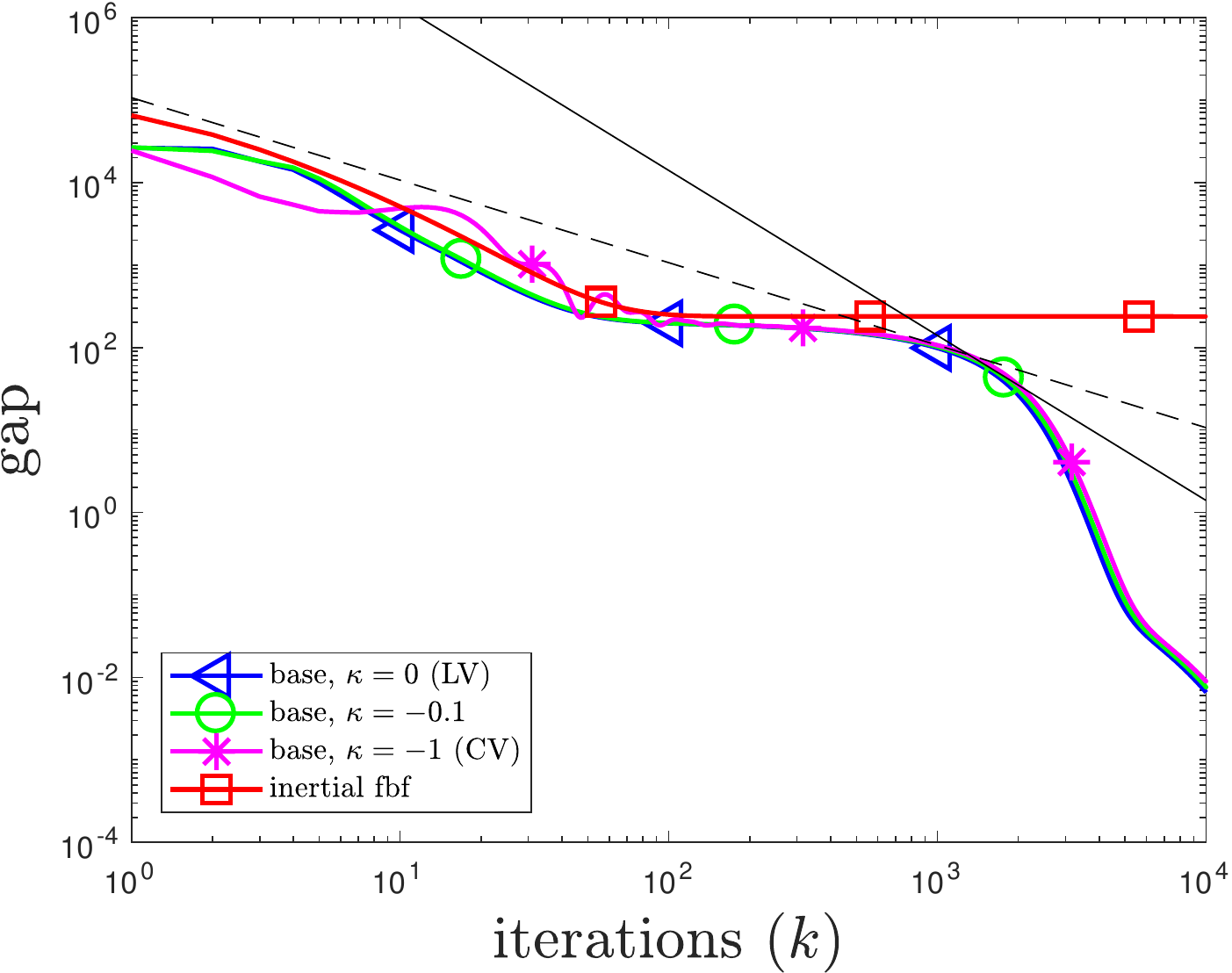}
            \caption[]%
            {}
            \label{fig:lgl_nonergodic}
        \end{subfigure}
        \end{center}
\caption{Convergence of the forward-backward (FB) algorithms generated by \eqref{eqn:generaliteration} and their accelerated variants \eqref{eqn:accgen} for a latent group lasso model.  
		 (a) optimal acceleration with bounded parameter setting (``optimal'') with ergodic convergence of the FB algorithm (``base'').
		 (b) optimal acceleration with unbounded parameter setting (``optimal'') with ergodic convergence of the FB algorithm (``base'').
     (c) non-ergodic convergence of the FB (``base'') and inertial FBF (``inertial fbf'') algorithms.
     Solid black lines represent $O(1/k^2)$ convergence, and dashed black lines represent $O(1/k)$ convergence.}\label{fig:lgl_conv}
\end{figure}

\begin{table}[H]
\caption{Scalability of the distributed version of \eqref{eqn:generaliteration} for latent group lasso. Time was measured in seconds per 100 iterations. Standard deviations are listed in parentheses. Any cell with missing values indicates that the experiment failed to run due to lack of memory.}
\label{tab:lgl_scalability}
\vskip -0.2in
\begin{center}
\fontsize{10}{12}\selectfont
\begin{sc}
\vskip 0.1in
\begin{tabular}{rrrrrrrrrr}\hline
         & \#GPUs  & 1     & 2      & 3      & 4      & 5      & 6      & 7      & 8      \\
\#groups & $p$     &       &        &        &        &        &        &        &        \\ \hline
1000     & 120010  & 4.754 & 3.359  & 2.524  & 2.166  & 1.894  & 1.649  & 1.598  & 1.602  \\
         &         & (0.003) & (0.024)  & (0.090)  & (0.068)  & (0.017)  & (0.020)  & (0.053)  & (0.050)  \\
5000     & 600010  &       & 19.133 & 14.378 & 10.888 & 9.299  & 7.883  & 7.386  & 7.251  \\
         &         &       & (0.142)  & (0.083)  & (0.344)  & (0.451)  & (0.042)  & (0.025)  & (0.074)  \\
8000     & 960010  &       &        & 22.023 & 17.825 & 14.236 & 12.141 & 10.964 & 10.133 \\
         &         &       &        & (0.132)  & (0.180)  & (0.150)  & (0.145)  & (0.077)  & (0.057)  \\
10000    & 1200010 &       &        &        & 22.271 & 17.647 & 15.045 & 13.320 & 12.194 \\
         &         &       &        &        & (0.439)  & (0.476)  & (0.165)  & (0.067) & (0.070)  \\ \hline
\end{tabular}
\end{sc}
\end{center}
\vskip -0.2in
\end{table}

%% file: sections/prelim.tex
\section{Monotone operator theory}\label{sec:theory}
Here we briefly state necessary results from monotone operator theory for the proofs in the subsequent section. 
For more details, see \citet{Bauschke:ConvexAnalysisAndMonotoneOperatorTheoryIn:2011}.

\paragraph{\textit{Set-valued operators.}}
A set-valued operator $T:\mathbb{R}^n \to 2^{\mathbb{R}^n}$
maps a vector $z \in \mathbb{R}^n$ to a set $T(z) \subset \mathbb{R}^n$.
The graph of $T$ is denoted by $\gra T =\{ (z,w) \in \mathbb{R}^n\times\mathbb{R}^n : w \in T(z)\}$.
When $T(z)$ is single-valued, i.e., $T(z)=\{w\}$, $T$ is a function, and we write simply as $T(z)=w$.
We use $I$ to denote the identity operator, i.e, $I(z)=z$.
When no confusion incurs, we also use $Tz$ to mean $T(z)$. In particular, when $T$ is a single-valued linear operator, $Tz$ is identified with a multiplication of the corresponding matrix $T \in \mathbb{R}^{n \times n}$ 
by a vector $z$.
The set of zeros of $T$ is defined as $\zer T = \{ z \in \mathbb{R}^n : 0 \in Tz \}$.
The inverse of $T$ is $T^{-1}:\mathbb{R}^n \to 2^{\mathbb{R}^n}$ such that $T^{-1}(w) = \{ z \in \mathbb{R}^n : w \in Tz \}$, hence $\gra T^{-1} = \{ (w,z) \in \mathbb{R}^n\times\mathbb{R}^n : w \in Tz\}$.
The resolvent of $T$ is $R_T = (I+T)^{-1}$.
Scaling of an operator $T$ by $t\in \mathbb{R}$ is defined by $(tT)(z) = tT(z)$.
Composition of two set-valued operators $T_1: \mathbb{R}^n \to 2^{\mathbb{R}^n}$ and  $T_2: \mathbb{R}^n \to 2^{\mathbb{R}^n}$ is defined by $T_2T_1z = \bigcup_{w \in T_1z} T_2w$.

\paragraph{\textit{Fixed points.}}
An operator $T:\mathbb{R}^n \to 2^{\mathbb{R}^n}$ is called nonexpansive if 
$\|u-u'\|_2 \le \|z-z'\|_2$ for all $u\in T(z), u'\in T(z') \in \mathbb{R}^n$;
it is called contractive if the inequality is strict.
Any nonexpansive operator is single-valued.
The set of fixed points of a single-valued operator $T$ is denoted by $\Fix{T}$, i.e., $\Fix{T}=\{z: z=Tz\}$.
For a contractive operator $T$, the fixed point iteration $z^{k+1}=Tz^{k}$ converges to a point in $\Fix{T}$, if $\Fix{T}\neq\emptyset$.

\paragraph{\textit{Averaged operators.}}
An operator $T$ is called $\alpha$-averaged, $0<\alpha<1$, if $T=(1-\alpha)I + \alpha R$ for some nonexpansive operator $R$.
Usually $R$ is defined implicitly. Note that $T$ itself is nonexpansive, and $\Fix{T}=\Fix{R}$.
If $T_1$ is $\alpha_1$-averaged and $T_2$ is $\alpha_2$-averaged, then $T_1T_2$ is $\alpha$-averaged where
$\alpha = (\alpha_1+\alpha_2-2\alpha_1\alpha_2)/(1-\alpha_1\alpha_2)$. 
An $\alpha$-averaged operator $T$ is nonexpansive but not necessarily contractive, hence the fixed point iteration $z^{k+1}=Tz^{k}$ above may not converge to a fixed point even if $\Fix{T}\neq\emptyset$.
In this case, the Krasnosel'ski\u{i}-Mann (KM) iteration
$z^{k+1} = z^{k} + \rho_k(Tz^{k}-z^{k})$
with a sequence $\{\rho_k\} \subset (0,1/\alpha]$ such that  $\sum_{k=0}^{\infty}\rho_k(1-\alpha\rho_k) = \infty$
ensures convergence. 

\paragraph{\textit{Monotone operators.}}
An operator $T$ is called monotone if $\langle z - z' , w - w' \rangle \ge 0$
for all $z, z' \in \mathbb{R}^n$ and for all $w \in Tz$, $w' \in Tz'$,
and maximally monotone if it is monotone and there is no monotone operator $T'$ such that $T \neq T'$ and $\gra T \subset \gra T'$.
The resolvent of a maximally monotone operator is single-valued; it is $1/2$-averaged.

\paragraph{\textit{Cocoercive operators.}}
A single-valued operator $T$ is called $\gamma$-cocoercive if for some $\gamma > 0$, $\langle z - z', Tz - Tz' \rangle \ge \gamma \| Tz - Tz' \|_2^2$.
A cocoercive operator is maximally monotone.
If an operator $T$ is $\gamma$-cocoercive with $\gamma > 1/2$, then $I-tT$ ($t>0$) is $t/(2\gamma)$-averaged.
A convex, closed, and proper function $\phi$ has $L$-Lipschitz continous gradient $\nabla \phi$ if and only if $\nabla \phi$ is $1/L$-cocoercive. 

\paragraph{\textit{Subdifferential.}}
An important example of a maximally monotone operator is
the subdifferential of a convex closed proper function.
A vector $g\in \mathbb{R}^n$ is a subgradient of a convex function $\phi$ at $z$ if
$
	\phi(z') \ge \phi(z) + \langle g, z'-z \rangle, ~ \forall z' \in \mathbb{R}^n.
$
The subdifferential of $\phi$ at $z$ is the set of subgradients at $z$:
$\partial \phi(z) = \{ g \in \mathbb{R}^n: \phi(z') \ge \phi(z) + \langle g, z'-z \rangle, ~ \forall z' \in \mathbb{R}^n \}$.
When $\phi$ is differentiable, $\partial \phi(z) = \{\nabla \phi(z) \}$.
If $\phi$ is in addition closed and proper, $(\partial \phi)^{-1} = \partial \phi^{*}$ holds, where $\phi^*$ is convex conjugate defined by $\phi^*(w) = \sup_{z \in \mathbb{R}^n} \{ \langle z, w \rangle - \phi(z) \}$.
The resolvent of a maximally monotone subdifferential operator is the proximity operator:
	$R_{\partial \phi} = (I+\partial \phi)^{-1}(z) = \prox_{\phi}(z) = \argmin_{z' \in \mathbb{R}^n} \phi(z') + \frac{1}{2}\| z' - z \|_2^2$.

\paragraph{\textit{Skew-symmetric operators.}}
Another example of a maximally monotone operator is a skew-symmetric matrix. 
The sum of a maximally monotone operator and a skew-symmetric matrix is also maximally monotone. 

\paragraph{\textit{Change of metric.}}
Note that the notion of nonexpansiveness, averagedness, cocoercivity, and monotonicity of an operator requires the inner product
$\langle \cdot, \cdot \rangle$ and its associated norm $\|\cdot\|_2$. 
We can appropriately define these concepts with respect to another inner product and its associated norm as well, say
$\langle \cdot, \cdot \rangle_M$ 
and $\|\cdot\|_M$, 
for $M$ a symmetric, positive definite matrix.
In particular,
averagedness of composition, convergence of the KM iteration, and averagedness of $I-tT$ for cocoercive $T$
hold by substituting the inner products and norms by $\langle \cdot, \cdot \rangle_M$ and $\|\cdot\|_M$, respectively.

\paragraph{\textit{Forward-backward splitting.}}
Some optimization problems can be translated to finding an element of $\zer T$
for an appropriate choice of maximally monotone
operator $T$. Often $T$ can be split into a sum of two maximally monotone operators $F$ and $G$. If $G$ is $\gamma$-cocoercive (hence single-valued), then we see
\begin{align}
	0 \in T(z) & \iff (I + tF)(z) \ni (I - tG)(z)  \nonumber  \\
				& \iff z = R_{tF}(I-tG)(z), \label{eqn:fbsplit_}
\end{align}
for $t > 0$. Equivalence \eqref{eqn:fbsplit_} shows that $\zer{(F+G)} = \Fix(R_{tF}(I-tG))$, 
thus we may solve the problem of finding a zero of $T$ by the following fixed-point iteration
\begin{align}\label{eqn:relaxedforwardbackward_}
		 z^{k+1} = (1-\rho_k)z^k +\rho_k R_{tF}(I-tG)(z^k).
\end{align}
This iteration is a KM iteration
because $R_{tF}(I-tG)$ is a $1/\delta$-averaged operator, 
where $\delta=2-t/(2\gamma)$. 
Thus \eqref{eqn:relaxedforwardbackward_} converges for $t\in(0,2\gamma)$ 
if $\zer(F+G)\neq\emptyset$ and
under the aforementioned condition for $\{\rho_k\}$.
Furthermore, the following hold \citep[proof of Theorems 25.8]{Bauschke:ConvexAnalysisAndMonotoneOperatorTheoryIn:2011}:
\begin{subequations}\label{eqn:FBS}
\begin{align}
	&
	\|z^{k+1}-z\|_2^2 \le \|z^{k}-z\|_2^2, \quad \forall z \in \zer(F+G);
	\label{eqn:FBS:monotone} \\
	&
	\textstyle \sum_{k=0}^{\infty} {\frac{\delta-\rho_k}{\rho_k}}\| z^{k+1}-z^{k} \|_2^2 \le \| z^0 - z\|_2^2, \quad \forall z \in \zer(F+G);
	\label{eqn:FBS:squaresummable}\\
	& \| z^{k+1}-z^{k} \|_2 \to 0.  \label{eqn:FBS:subsequence} \noeqref{eqn:FBS:subsequence}
\end{align}
\end{subequations}

\paragraph{\textit{Preconditioning.}}
In the forward-backward splitting above,
observe that the identity matrices in the first line can be replaced by an invertible matrix $M$, yielding a \emph{preconditioned} forward-backward splitting algorithm
\begin{align}\label{eqn:preconditionedFBS}
	z^{k+1} &= (1-\rho_k)z^k + \rho_k R_{tM^{-1}F}(I-tM^{-1}G)(z^k).
\end{align}
Preconditioning is useful when evaluating the resolvent $R_{tM^{-1}F}$ is easier than $R_{tF}$.
It can be shown that if $M$ is symmetric positive definite,
$M^{-1}F$ is maximally monotone with respect to $\langle \cdot, \cdot \rangle_M$  \cite{Combettes:Optimization:2012},
and $M^{-1}G$ is $\gamma\lambda_{\min}(M)$-cocoercive with respect to $\|\cdot\|_M$ 
\citep{Davis:SiamJOptim:2015}.
Therefore 
we can replace
$\|\cdot\|_2$ by $\|\cdot\|_M$, and $\gamma$ by $\gamma\lambda_{\min}(M)$
in \eqref{eqn:FBS}.

%% file: sections/prelimproof.tex
\section{Proofs}\label{sec:proofs}

\subsection{Preconditioned forward-backward splitting}
\begin{proof}[Proof of Lemma \ref{lemma:pregap}]
	Observe that
	\begin{align}\label{eqn:pregap1}
		\|z^{-}-z\|_M^2 &= \|z^{-}-z_{\rho}+z_{\rho}-z\|_M^2 \nonumber\\
						&= \|z^{-}-z_{\rho}\|_M^2 - 2\langle z^{-}-z_{\rho},z-z_{\rho}\rangle_M + \|z_{\rho}-z\|_M^2,
	\end{align}
	and, from \eqref{eqn:fbsplit_},
	\begin{align}\label{eqn:fbinclusion}
	&z^{+} + M^{-1}Fz^{+} \ni z^{-} - M^{-1}Gz^{-} \nonumber \\
	&\iff z^{+} + M^{-1}\begin{bmatrix} 0 & K^T \\ -K & \partial h^{*} \end{bmatrix} \begin{bmatrix} x^{+} \\ y^{+} \end{bmatrix}
		\ni z^{-} - M^{-1} \begin{bmatrix} \nabla f &  \\  & 0 \end{bmatrix}\begin{bmatrix} x^{-} \\ y^{-} \end{bmatrix} \nonumber \\
	&\iff  (1/\rho)(z^{-}-z_{\rho}) = z^{-}-z^{+} \in M^{-1} \begin{bmatrix} \nabla f(x^{-}) + K^Ty^{+} \\ -Kx^{+} + \partial h^{*}(y^{+}) \end{bmatrix}
	\end{align}
	Then,
	\begin{align}\label{eqn:pregap2}
	\langle z^{-}-z_{\rho},z-z_{\rho}\rangle_M &= \langle \rho(z^{-}-z^{+}), z-z_{\rho} \rangle_M
				= \rho \left\langle \begin{bmatrix} \nabla f(x^{-})+K^Ty^{+} \nonumber\\ -Kx^{+} + \partial h^{*}(y^{+}) \end{bmatrix},
									\begin{bmatrix} x-x_{\rho} \\ y-y_{\rho} \end{bmatrix} \right \rangle \nonumber\\
				&= \rho \langle \nabla f(x^{-}), x-x_{\rho} \rangle + \rho \langle K^Ty^{+}, x-x_{\rho} \rangle \nonumber\\
				& \quad + \rho \langle -Kx^{+}, y-y_{\rho} \rangle + \rho \langle \partial h^{*}(y^{+}), y-y_{\rho} \rangle \nonumber\\
				&= \rho \langle \nabla f(x^{-}), x^{-}-x_{\rho} \rangle + \rho \langle \nabla f(x^{-}), x-x^{-} \rangle + \rho \langle K^Ty^{+}, x-x_{\rho} \rangle \nonumber\\
				& \quad + \rho \langle -Kx^{+}, y-y_{\rho} \rangle + \rho \langle \partial h^{*}(y^{+}), y^{+}-y_{\rho} \rangle + \rho \langle \partial h^{*}(y^{+}), y-y^{+} \rangle \nonumber\\
				&\le \rho \langle \nabla f(x^{-}), x^{-}-x_{\rho} \rangle + \rho (f(x)-f(x^{-})) + \rho \langle K^Ty^{+}, x-x_{\rho} \rangle \nonumber\\
				& \quad + \rho \langle -Kx^{+}, y-y_{\rho} \rangle + \rho \langle \partial h^{*}(y^{+}), y^{+}-y_{\rho} \rangle + \rho (h^{*}(y)-h^{*}(y^{+}) ),
	\end{align}
	understanding that ``$\partial h^{*}(\cdot)$'' represents a subgradient in the corresponding subdifferential.
	The first and second equalities follow from \eqref{eqn:fbinclusion};
	the last inequality is due to the definition of subgradient.
	By plugging the inequality \eqref{eqn:pregap2} in \eqref{eqn:pregap1} and rearranging terms, we obtain
	\begin{align}\label{eqn:pregap3}
	2\rho(\mathcal{L}(x^{+},y) &-\mathcal{L}(x,y^{+})) - \|z^{-}-z\|_M^2 + \|z_{\rho}-z\|_M^2  \nonumber\\
		& \le - \|z^{-}-z_{\rho}\|_M^2 + 2\rho \langle \nabla f(x^{-}),  x^{-}-x_{\rho} \rangle + 2\rho \langle \partial h^{*}(y^{+}),y^{+}-y_{\rho} \rangle \nonumber\\
		& \quad - 2\rho \langle K^Ty^{+}, x_{\rho} \rangle + 2\rho \langle Kx^{+}, y_{\rho} \rangle + 2\rho (f(x^{+})-f(x^{-}) )
	\end{align}
	Now it suffices to show that the right-hand side of \eqref{eqn:pregap3} is less than or equal to $(1-2/\rho)\|z^{-}-z_{\rho}\|_M^2 + (L_f/\rho) \|x^{-}-x_{\rho}\|_2^2$. 
	To see this,
	\begin{align*}
	\text{(RHS)} &=
		- \|z^{-}-z_{\rho}\|_M^2
		+ 2\rho \langle \nabla f(x^{-}),  x^{-}-x_{\rho} \rangle + 2\rho \langle \partial h^{*}(y^{+}),y^{+}-y_{\rho} \rangle \nonumber\\
		& \quad\quad - 2\rho \langle K^Ty^{+}, x_{\rho}-x^{+} \rangle
		- 2\rho \langle K^Ty^{+}, x^{+} \rangle
		+ 2\rho \langle Kx^{+}, y_{\rho} - y^{+} \rangle
		+ 2\rho \langle Kx^{+}, y^{+} \rangle \\
		& \quad\quad + 2\rho (f(x^{+})-f(x^{-}) ) \\
		&= - \|z^{-}-z_{\rho}\|_M^2  + 2\rho (f(x^{+}) - f(x^{-}) - \langle \nabla f(x^{-}), x^{+}-x^{-}\rangle ) \nonumber\\
		& \quad\quad + 2\rho \langle \nabla f(x^{-}) + K^Ty^{+}, x^{+}-x_{\rho} \rangle + 2\rho \langle -Kx^{+} + \partial h^{*}(y^{+}),y^{+}-y_{\rho} \rangle \nonumber\\
		&= - \|z^{-}-z_{\rho}\|_M^2  + 2\rho (f(x^{+}) - f(x^{-}) - \langle \nabla f(x^{-}), x^{+}-x^{-}\rangle ) \nonumber\\
		& \quad\quad + 2\rho \langle M(z^{-}-z^{+}),z^{+}-z_{\rho} \rangle \nonumber\\
		&= - \|z^{-}-z_{\rho}\|_M^2  + 2\rho (f(x^{+}) - f(x^{-}) - \langle \nabla f(x^{-}), x^{+}-x^{-}\rangle ) \nonumber\\
		& \quad\quad + 2 \langle z^{-}-z_{\rho},z^{+}-z_{\rho} \rangle_M \nonumber\\
		&= - \|z^{-}-z_{\rho}\|_M^2  + 2\rho (f(x^{+}) - f(x^{-}) - \langle \nabla f(x^{-}), x^{+}-x^{-}\rangle ) \nonumber\\
		& \quad\quad + 2(1-1/\rho) \langle z^{-}-z_{\rho},z^{-}-z_{\rho} \rangle_M \nonumber\\
		&\le (1-2/\rho)\|z^{-}-z_{\rho}\|_M^2 + \rho L_f \|x^{-}-x^{+}\|_2^2 \\
		&= (1-2/\rho)\|z^{-}-z_{\rho}\|_M^2 + (L_f/\rho) \|x^{-}-x_{\rho}\|_2^2  
	\end{align*}
	where the third equality follows from \eqref{eqn:fbinclusion};
	the fourth and fifth equalities are from \eqref{eqn:fbsplit_};
	the first inequality is due to the Lipschitz continuity of $\nabla f$; the final equality is again from \eqref{eqn:fbsplit_}.
\end{proof}

%% file: sections/originalproof.tex
We need the following fact to prove Theorem \ref{thm:ergodic}.

\begin{proposition}\label{prop:preconditionedcocoercivity}
	Let $M$ be a symmetric, positive definite matrix in $\mathbb{R}^{(p+l)\times(p+l)}$ and $G$ as given in \eqref{eqn:operatorsplit}. Then, for $\mu >0$ such that
	\begin{align}\label{eqn:normupperbound}
		\| (x,0) \|_{M^{-1}}^2 \le (1/\mu)\|x\|_2^2, \quad \forall x \in \mathbb{R}^p,
	\end{align}
	operator $M^{-1}G$ is $\mu/L_f$-cocoercive in $\langle\cdot,\cdot\rangle_M$.
\end{proposition}

\begin{proof}
	\begin{align*}
		\| M^{-1}Gz - M^{-1}Gz' \|_M^2 = \| Gz - Gz' \|_{M^{-1}}^2 &= \| (\nabla f(x) - \nabla f(x'), 0) \|_{M^{-1}}^2 \\
			& \le (1/\mu)\| \nabla f(x) - \nabla f(x') \|_2^2 \\
			& \le (L_f/\mu)\langle \nabla f(x) - \nabla f(x'), x - x' \rangle \\
			&= (L_f/\mu) \langle Gz - Gz', z-z' \rangle \\
			&= (L_f/\mu) \langle M^{-1}Gz - M^{-1}Gz', z-z' \rangle_M.
	\end{align*}
	Note that we used $1/L_f$-cocoercivity of $\nabla f$ in the third line.
\end{proof}
\begin{proof}[Proof of Proposition \ref{prop:Gtildecocoercive}]
Note that $\|L^T\cdot\|_{M_{\mathsf{LV}}}^2 = \langle  M_{\mathsf{LV}} L^T\cdot, L^T\cdot \rangle = \|\cdot\|_{M_{\mathsf{CV}}}^2$ and likewise $\langle L^T\cdot, L^T\cdot \rangle_{M_{\mathsf{LV}}} = \langle \cdot, \cdot \rangle_{M_{\mathsf{CV}}}$.
Then,
\begin{align*}
	\langle M_{\mathsf{LV}}^{-1}&L^{-1}GL^{-T}w - M_{\mathsf{LV}}^{-1}L^{-1}GL^{-T}w', w-w' \rangle_{M_{\mathsf{LV}}}  \\
	&= \langle (L^TM_{\mathsf{CV}}^{-1}L)(L^{-1}GL^{-T})(L^Tz) - (L^TM_{\mathsf{CV}}^{-1}L)(L^{-1}GL^{-T})(L^Tz'), L^Tz-L^Tz' \rangle_{M_{\mathsf{LV}}} \\
	&= \langle L^T(M_{\mathsf{CV}}^{-1}Gz - M_{\mathsf{CV}}^{-1}Gz'), L^T(z-z') \rangle_{M_{\mathsf{LV}}} \\
	&= \langle M_{\mathsf{CV}}^{-1}Gz - M_{\mathsf{CV}}^{-1}Gz', z-z' \rangle_{M_{\mathsf{CV}}} \\
	&\ge (\mu/L_f) \| M_{\mathsf{CV}}^{-1}Gz - M_{\mathsf{CV}}^{-1}Gz' \|_{M_{\mathsf{CV}}} \\
	&=   (\mu/L_f) \| L^T(L^{-T}M_{\mathsf{LV}}^{-1}L^{-1}GL^{-T}(L^Tz)-L^{-T}M_{\mathsf{LV}}^{-1}L^{-1}GL^{-T}(L^Tz') \|_{M_{\mathsf{LV}}}^2 \\
	&=   (\mu/L_f) \| M_{\mathsf{LV}}^{-1}L^{-1}GL^{-T}w - M_{\mathsf{LV}}^{-1}L^{-1}GL^{-T}w' \|_{M_{\mathsf{LV}}}^2,
\end{align*}
where the inequality and $\mu$ come from Proposition \ref{prop:preconditionedcocoercivity}, as follows.
From \eqref{eqn:M1M2}, we see that
\begin{align*}
	M_{\mathsf{CV}}^{-1}
	&= L^{-T} M_{\mathsf{LV}}^{-1} L^{-1}
	= \begin{bmatrix} I &   \tau K \\  & I \end{bmatrix}
	  \begin{bmatrix} \tau I &  \\  & (\frac{1}{\sigma}I-\tau KK^T)^{-1} \end{bmatrix}
	  \begin{bmatrix} I & 0 \\ \tau K  & I \end{bmatrix}  \\
	&= \begin{bmatrix} \tau I + \tau^2 K^T(\frac{1}{\sigma}I-\tau KK^T)^{-1}K & \tau K^T(\frac{1}{\sigma}I-\tau KK^T)^{-1} \\
			\tau(\frac{1}{\sigma}I-\tau KK^T)^{-1}K  & (\frac{1}{\sigma}I-\tau KK^T)^{-1} \end{bmatrix}
\end{align*}
and can choose $\mu=1/\tau-\sigma\|K\|_2^2$, because
$\lambda_{\max}(\tau I + \tau^2 K^T(\frac{1}{\sigma}I-\tau KK^T)^{-1}K) = \tau + \tau^2\|K\|_2^2/(1/\sigma-\tau\|K\|_2^2) = \frac{1}{1/\tau-\sigma\|K\|_2^2}$.
Therefore $M_{\mathsf{LV}}^{-1}L^{-1}GL^{-T}$ is $(1/\tau-\sigma\|K\|_2^2)/L_f$-cocoercive with respect to $\|\cdot\|_{M_{\mathsf{LV}}}$.
\end{proof}
\begin{proof}[Proof of Proposition \ref{prop:unifiedconvergence}]
For $M$ given by \eqref{eqn:generalM}, its inverse is given by 
\begin{align*}
M^{-1} &= \tilde{L}^{-T} M_{\mathsf{LV}} \tilde{L}^{-1} = 
  \begin{bmatrix} I & - \tau C \\ & I\end{bmatrix}
  \begin{bmatrix} \tau I & \\ & (\frac{1}{\sigma} I - \tau K K^T)^{-1} \end{bmatrix}
  \begin{bmatrix} I & \\ - \tau C & I \end{bmatrix} \\
  &=\begin{bmatrix}\tau I + \tau^2 C^T (\frac{1}{\sigma} I - \tau K K^T)^{-1} C & - \tau C^T (\frac{1}{\sigma} I - \tau K K^T)^{-1} \\
    -\tau (\frac{1}{\sigma} I - \tau K K^T)^{-1} C & (\frac{1}{\sigma} I - \tau K K^T)^{-1} \end{bmatrix}.
\end{align*}
Since $\lambda_{\max}(\tau I + \tau^2 C^T (\frac{1}{\sigma} I - \tau K K^T)^{-1} C) = \tau + \tau^2 \|C\|_2^2/(1/\sigma - \tau \|K\|_2^2)$,
we see that \eqref{eqn:normupperbound} holds with $\mu = \left(\tau + \tau^2 \|C\|_2^2/(1/\sigma - \tau \|K\|_2^2)\right)^{-1} 
= \frac{1/\tau - \sigma \|K\|_2^2}{1 - \sigma \tau (\|K\|_2^2 - \|C\|_2^2)}$. 
This shows that the operator $M^{-1}G$ is $\frac{1/\tau - \sigma \|K\|_2^2}{L_f(1 - \sigma \tau (\|K\|_2^2 - \|C\|_2^2))}$-cocoercive with respect to $\langle \cdot, \cdot \rangle_M$. 
Hence, Algorithm \eqref{eqn:generaliteration} meets the condition for \eqref{eqn:preconditionedFBS} with $t=1$ if
with $\gamma = \frac{1/\tau - \sigma \|K\|_2^2}{L_f(1 - \sigma \tau (\|K\|_2^2 - \|C\|_2^2))}>1/2$.
Required positive definiteness of $M$ implies 
$\frac{1}{\tau\sigma} > \|K\|_2^2$.
Thus the result \eqref{eqn:generalparamrange} follows.  
\end{proof}

\begin{proof}[Proof of Theorem \ref{thm:ergodic}]
	From the convexity-concavity of $\mathcal{L}(x,y)$, we have
	\begin{align*}
		\mathcal{L}(\bar{x}^N,y)-\mathcal{L}(x,\bar{y}^N) &\le \frac{1}{\sum_{k=0}^N\rho_k}\sum_{k=0}^N\rho_k(\mathcal{L}(\tilde{x}^{k},y)-\mathcal{L}(x,\tilde{y}^{k})) \\
			&\le \frac{1}{2\sum_{k=0}^N \rho_k}\left( \|z^{0}-z\|_M^2 + \sum_{k=0}^N\frac{L_f}{\rho_k}\|x^{k+1}-x^{k}\|_2^2\right),
	\end{align*}
	where the second inequality comes from Lemma \ref{lemma:pregap} by putting $z^{-}=z^{k}$, $z^{+}=\tilde{z}^{k}$, $\rho=\rho_k$, $z_{\rho}=z^{k+1}$, and noting that $1<1/\alpha<2$ by the assumption $\mu>L_f/2$.
	Now by Proposition \ref{prop:preconditionedcocoercivity} we see that $R_{M^{-1}F}(I-M^{-1}G)$ is $\alpha$-averaged with respect to $\|\cdot\|_M$, 
	thus by \eqref{eqn:FBS:squaresummable} we have
	\begin{align*}
	\frac{1-\alpha\bar{\rho}}{\alpha}\sum_{k=0}^{\infty}\frac{1}{\rho_k}\|x^{k+1}-x^{k}\|_2^2 &\le \sum_{k=0}^{\infty}\frac{1-\alpha\rho_k}{\alpha\rho_k}\|x^{k+1}-x^{k}\|_2^2  \\
	&\le \sum_{k=0}^{\infty}\frac{1-\alpha\rho_k}{\alpha\rho_k}\|z^{k+1}-z^{k}\|_2^2 \le \frac{1}{\lambda_{\min}(M)}\|z^{0}-z^{\star}\|_M^2.
	\end{align*}
	Therefore
	\begin{align*}
		\mathcal{L}(\bar{x}^N,y)-\mathcal{L}(x,\bar{y}^N)
			&\le \frac{1}{2\sum_{k=0}^N \rho_k}\left( \|z^{0}-z\|_M^2 + \sum_{k=0}^N\frac{L_f}{\rho_k}\|x^{k+1}-x^{k}\|_2^2\right) \\
			& \le \frac{1}{2\sum_{k=0}^N\rho_k}\left( \|z^{0}-z\|_M^2 + \frac{\alpha L_f}{(1-\alpha\bar{\rho})\lambda_{\min}(M)}\|z^{0}-z^{\star}\|_M^2\right).
	\end{align*}
\end{proof}

\begin{proof}[Proof of Corollary \ref{cor:ergodic}]
The proof closely follows that of \citet[Theorem 1]{loris2011generalization}, given for $f$ being quadratic.
Because $z^{k}=(x^{k},y^{k}) \to (x^{\star},y^{\star})=z^{\star} \in \Fix{T}$ where $T=R_{M^{-1}F}(I-M^{-1}G)$,
we see $\tilde{z}^k=Tz^{k} \to z^{\star}$ and thus $\bar{z}^N=(\bar{x}^N,\bar{y}^N)={(\sum_{k=1}^N \rho_k \tilde{z}^{k})}/{(\sum_{k=1}^N \rho_k)} \to z^{\star}$.
Also because $(x^{\star},y^{\star})$ is a saddle-point of $\mathcal{L}(x,y)$, we have $\mathcal{F}^{\star}=\mathcal{F}(x^{\star})=\mathcal{L}(x^{\star},y^{\star}) \ge \mathcal{L}(x^{\star},y)$ for all $y\in\mathbb{R}^l$.
Then
\[
	0 \le \mathcal{F}(\bar{x}^N)-\mathcal{F}^{\star} = \mathcal{F}(\bar{x}^N)-\mathcal{L}(x^{\star},y^{\star})
		\le \mathcal{F}(\bar{x}^N) - \mathcal{F}(x^{\star},\bar{y}^N)
		= \sup_{y\in\mathbb{R}^l} \mathcal{L}(\bar{x}^N,y)-\mathcal{L}(x^{\star},\bar{y}^N).
\]
The $\sup_{y\in\mathbb{R}^l} \mathcal{L}(\bar{x}^N,y) = f(\bar{x}^N) + \sup_{y\in\mathbb{R}^l}\langle K\bar{x}^N,y \rangle - h^{*}(y)$ is attained at a $\hat{y}^N \in \partial h(K\bar{x}^N)$ because under the assumption $\dom h = \mathbb{R}^l$, $h^*$ is 1-coercive, thus $- \langle K\bar{x}^N, \cdot \rangle + h^\star (\cdot)$ is coercive (\citealp[Prop.X.1.3.9]{hiriart1993convex}; \citealp[Proposition 11.14]{,Bauschke:ConvexAnalysisAndMonotoneOperatorTheoryIn:2011}).
As $\bar{x}^N$ converges, $K\bar{x}^N$ is bounded independent of $N$.
Now because $h$ is real-valued, it follows that $h$ is locally Lipschitz in the neighborhood of $K\bar{x}^N$ \citep[see, e.g.,][Proposition 5.4.2]{Bertsekas:ConvexOptimizationTheory:2009}.
Let the local Lipschitz constant be $Q$. It also follows that $\partial h(K\bar{x}^N)$ is bounded by $Q$, i.e. $\|\hat{y}^N\|_2\le Q$. Therefore
\begin{align*}
	0 \le \mathcal{F}(\bar{x}^N)-\mathcal{F}^{\star} &= \mathcal{F}(\bar{x}^N)-\mathcal{L}(x^{\star},y^{\star})
		= \sup_{y\in\mathbb{R}^l} \mathcal{L}(\bar{x}^N,y)-\mathcal{L}(x^{\star},\bar{y}^N) \\
		&= \max_{\|y\|_2 \le Q} \mathcal{L}(\bar{x}^N,y)-\mathcal{L}(x^{\star},\bar{y}^N) \\
		&\le \max_{\|y\|_2 \le Q} \frac{1}{2\sum_{k=0}^N\rho_k}\left( \|(x^0,y^0)-(x^{\star},y)\|_M^2 + \frac{\alpha L_f}{(1-\alpha\bar{\rho})\lambda_{\min}(M)}\|z^0-z^{\star}\|_M^2\right) \\
		&= C_1/(\sum_{k=0}^N\rho_k).
\end{align*}
\end{proof}

We need the following lemma to prove Theorem \ref{thm:nonergodic}.
\begin{lemma}[\citet{Davis:SiamJOptim:2015}, Theorem 4.1]\label{lemma:sequencerate}
	Suppose $\mathcal{T}:\mathbb{R}^n\to\mathbb{R}^n$ is an $\alpha$-averaged operator with respect to $\|\cdot\|_M$, where $0<\alpha<1$ and $M \succ 0$.
	Let $z^{\star} \in \Fix{\mathcal{T}}$ and $z^0 \in \mathbb{R}^n$.
	For $\{\rho_k\} \subset (1,1/\alpha)$, consider a sequence $\{z^k\}$ generated by the KM iteration:
	\[
		z^{k+1} = z^{k} + \rho_k(\mathcal{T}z^{k}-z^{k}).
	\]
	If $\tau = \sup_{k\ge 0}(1-\alpha\rho_k)\rho_k/\alpha > 0$, then we have
	\begin{align}\label{eqn:sequencerate}
		\|\mathcal{T}z^{k}-z^{k}\|_M^2 \le \frac{\|z^{0}-z^{\star}\|_M^2}{\tau(k+1)}
		\quad\text{and}\quad
		\|\mathcal{T}z^{k}-z^{k}\|_M^2 = o\left(\frac{1}{k+1}\right).
	\end{align}
\end{lemma}

\begin{proof}[Proof of Theorem \ref{thm:nonergodic}]
By condition \eqref{eqn:Mbound},
$\|z'\|_M^2 \ge \nu \|x'\|_2^2 + \epsilon \|y'\|_2^2$ for all $z'=(x',y')$.
Then, in the same manner as the proof of Theorem \ref{thm:ergodic}, we put $z^{-}=z^{k}$, $z^{+}=\tilde{z}^{k}$, $\rho=\rho_k$, $z_{\rho}=z^{k+1}$ in Lemma \ref{lemma:pregap} and note that $1<1/\alpha<2$ by the assumption $\nu>L_f/2$ to have
\begin{align}\label{eqn:pregapimp}
		2\rho_k(\mathcal{L}(\tilde{x}^k,y) - \mathcal{L}(x,\tilde{y}^k) )
		&\le \|z^{k}-z\|_M^2 - \|z^{k+1}-z\|_M^2
		+ \epsilon(1-2/\rho_k)\|y^{k}-y^{k+1}\|_2^2 \nonumber \\
		& \quad + \left(\nu-\textstyle\frac{2\nu-L_f}{\rho_k}\right) \|x^{k}-x^{k+1}\|_2^2 .
\end{align}

The rest of the proof closely follows that of \citet[Theorem 4.2]{Davis:SiamJOptim:2015}.
Note $\nu$ satisfies \eqref{eqn:normupperbound} and hence by Proposition \ref{prop:preconditionedcocoercivity}, $R_{M^{-1}F}(I-M^{-1}G): z^{k}\mapsto\tilde{z}^{k}$ is $\alpha$-averaged with respect to $\|\cdot\|_M$.
Let $z_{\rho} = (1-\rho)z^{k} + \rho\tilde{z}^{k}=:T_{\rho}z^{k}$ for any $\rho \in (0,1/\alpha)$; for $\rho=\rho_k$, we have $z_{\rho}=z^{k+1}$.
Then the map $T_{\rho}:z^{k} \mapsto z_{\rho}$ is $\alpha\rho$-averaged with respect to $\|\cdot\|_M$ and hence $\|z_{\rho}-z^{\star}\|_M \le \|z^{k}-z^{\star}\|_M$.
From \eqref{eqn:FBS:monotone},
we have $\|z_{\rho}-z^{\star}\|_M \le \|z^{0}-z^{\star}\|_M$, thus by the triangle inequality $\|z_{\rho}-z\|_M \le \|z^{0}-z^{\star}\|_M + \|z^{\star}-z\|_M$ for any $z\in\mathbb{R}^{p+l}$.
Then we have
\begin{align}\label{eqn:innerproductbound}
\begin{split}
	(1/\rho)\langle z^{k}-z_{\rho},z_{\rho}-z \rangle_M &= \langle \tilde{z}^{k}-z^{k}, z_{\rho}-z \rangle_M \\
   		&\le \|\tilde{z}^{k}-z^{k}\|_M \|z_{\rho}-z\|_M
		\le \frac{\|z^{0}-z^{\star}\|_M}{\sqrt{\tau(k+1)}}(\|z^{0}-z^{\star}\|_M+\|z^{\star}-z\|_M)\quad
\end{split}
\end{align}
for all $\rho \in (0,1/\alpha)$,
where the last inequality is from Lemma \ref{lemma:sequencerate}.

Note that Lemma \ref{lemma:pregap} (with the improvement \eqref{eqn:pregapimp} above) still holds if $\rho_k$ is replaced by any $\rho \in (0,1/\alpha)$ and $z^{k+1}$ is replaced by $z_{\rho}$.
Therefore we have
\begin{align*}
	\mathcal{L}(&\tilde{x}^k, y) - \mathcal{L}(x,\tilde{y}^k)  \\
	&\le \inf_{0<\rho<1/\alpha} \frac{1}{2\rho}\left(\|z^{k}-z\|_M^2 - \|z_{\rho}-z\|_M^2 - \epsilon(\textstyle\frac{2}{\rho}-1)\|y_{\rho}-y^{k}\|_2^2 + (\nu-\textstyle\frac{2\nu-L_f}{\rho})\|x_{\rho}-x^{k}\|_2^2 \right) \\
	&=   \inf_{0<\rho<1/\alpha} \frac{1}{2\rho}\left(2\langle z^{k}-z_{\rho},z_{\rho}-z \rangle_M + \|z_{\rho}-z^{k}\|_M^2 - \epsilon(\textstyle\frac{2}{\rho}-1)\|y_{\rho}-y^{k}\|_2^2 + (\nu-\textstyle\frac{2\nu-L_f}{\rho})\|x_{\rho}-x^{k}\|_2^2 \right) \\
	&\le \inf_{0<\rho<1/\alpha} \frac{1}{2\rho}\left(2\langle z^{k}-z_{\rho},z_{\rho}-z \rangle_M + (\bar{\lambda}+\epsilon-\textstyle\frac{2\epsilon}{\rho})\|y_{\rho}-y^{k}\|_2^2
					+ (\bar{\lambda}+\nu-\textstyle\frac{2\nu-L_f}{\rho})\|x_{\rho}-x^{k}\|_2^2 \right) \\
	&=   \inf_{0<\rho<1/\alpha} \frac{1}{\rho}\langle z^{k}-z_{\rho},z_{\rho}-z \rangle_M + \frac{1}{2 \rho}\left( (\bar{\lambda}+\epsilon-\textstyle\frac{2\epsilon}{\rho})\|\tilde{y}^{k}-y^{k}\|_2^2 + (\bar{\lambda}+\nu-\textstyle\frac{2\nu-L_f}{\rho})\|\tilde{x}^{k}-x^{k}\|_2^2 \right) \\
	&\le \frac{1}{\tilde\rho}\langle z^{k}-z_{\tilde\rho},z_{\tilde\rho}-z \rangle_M
\end{align*}
by choosing a small $\tilde\rho \in (0,1/\alpha)$ such that $\bar{\lambda}+\epsilon \le 2\epsilon/\tilde{\rho}$ and $\bar{\lambda}+\nu \le (2\nu-L_f)/\tilde{\rho}$,
where $\bar{\lambda}=\lambda_{\max}(M)$.
The first equality uses the cosine rule
\[
	2\langle a-b, c-b \rangle_M = -\|a-c\|_M^2 + \|a-b\|_M^2 + \|c-b\|_M^2
\]
for any $a,b,c \in \mathbb{R}^{p+l}$.
The desired result follows from \eqref{eqn:innerproductbound}.

The $o(1/\sqrt{k+1})$ rate is also from \eqref{eqn:innerproductbound} and Lemma \ref{lemma:sequencerate}.
\end{proof}

\begin{proof}[Proof of Proposition \ref{prop:Mcondition}]
	We first show that Condition \ref{prop:Mcondition:2} is equivalent to
	\begin{align}\label{eqn:Mcondition2equiv}
		0 \prec M^{-1} \prec \begin{bmatrix} \frac{2}{L_f}I &  \\  & \infty \end{bmatrix},
	\end{align}
	or $z^TM^{-1}z < \frac{2}{L_f}\|x\|_2^2 + \delta_{\{0\}}(y)$ for all $z=(x,y)\neq 0$. 
	To see this, let $g_1(z) = (1/2)z^T Mz$ and $g_2(z) = \textstyle \frac{1}{2}z^T\begin{bmatrix} \frac{L_f}{2}I &  \\ & 0 \end{bmatrix}z = \frac{L_f}{4}\|x\|_2^2$.
	Then Condition \ref{prop:Mcondition:2} ensures that $g_1(z) > g_2(z)$ for all $z\neq 0$.
	Take the convex conjugates of $g_1$ and $g_2$. Observe that for $w=(w_1,w_2)$, $g_1^{*}(w)=\sup_z\langle w,z \rangle - (1/2)z^TMz = (1/2)w^TM^{-1}w$ and
	$$
	g_2^{*}(w) = \sup_{z}\langle w,z \rangle - g_2(z) = \sup_{x} \langle w_1,x \rangle - \textstyle\frac{L_f}{4}\|x\|_2^2 + \sup_{y} \langle w_2,y \rangle
		= \begin{cases} \frac{1}{L_f}\|w_1\|_2^2, & \text{if~} w_2=0, \\ \infty, & \text{otherwise}. \end{cases}
	$$
	Conjugacy asserts that $g_1^{*}(w) \le g_2^{*}(w)$, or equivalently
	$$
	0 \prec M^{-1} \preceq \begin{bmatrix} \frac{2}{L_f}I &  \\  & \infty \end{bmatrix}.
	$$
	Now for $w=(w_1,0)$ ($w_1\neq 0$), $f_1^{*}(w) = \langle w,\hat{z} \rangle - (1/2)\hat{z}^TM\hat{z} = (1/2)w_1 \bar{M}_{11}w_1$, where
   	$$
	M^{-1} = \begin{bmatrix} \bar{M}_{11} & \bar{M}_{12} \\ \bar{M}_{12}^T & \bar{M}_{22} \end{bmatrix}, \quad
	\hat{z} = M^{-1}w = \begin{bmatrix} \bar{M}_{11}w_1 \\ \bar{M}_{12}^Tw_1 \end{bmatrix} \neq 0,
	$$
	because $\bar{M}_{11} \succ 0$.
	Then
	$$
	\textstyle\frac{1}{2}w_1^T\bar{M}_{11}w_1 = g_1^{*}(w) = \langle w,\hat{z} \rangle - g_1(\hat{z}) < \langle w,\hat{z} \rangle - g_2(\hat{z}) \le \sup_{z} \langle w,z \rangle - g_2(z) = g_2^{*}(w) = \frac{1}{L_f}\|w_1\|_2^2,
    $$
	or $\bar{M}_{11} \prec \frac{2}{L_f}I$. It follows \eqref{eqn:Mcondition2equiv}.
	Because both $g_1$ and $g_2$ are convex, closed, and proper, the same logic applies to $g_1^{*}$ and $g_2^{*}$, meaning that the above matrix inequality implies Condition \ref{prop:Mcondition:2}, establishing the equivalence.

	Now Condition \ref{prop:Mcondition:1} implies $(x,0)^TM^{-1}(x,0) < \frac{2}{L_f}\|x\|_2^2$ for all $x\neq 0$ and $z^TM^{-1}z < \infty$, implying \eqref{eqn:Mcondition2equiv}, thus Condition \ref{prop:Mcondition:2}.
	That Condition \ref{prop:Mcondition:2} implies Condition \ref{prop:Mcondition:1} is straightforward, by choosing $1/\mu \in [\lambda_{\max}(\bar{M}_{11}), 2/L_f)$.

	Condition \ref{prop:Mcondition:3} is equivalent to
	\begin{align}\label{eqn:invMupper}
	0 \prec M^{-1} \preceq \begin{bmatrix} \frac{1}{\nu}I & \\  & \frac{1}{\epsilon}I \end{bmatrix},
	\end{align}
	thus $(x,0)^TM^{-1}(x,0) \le \frac{1}{\nu}\|x\|_2^2$ where $\nu > L_f/2$. This implies Condition \ref{prop:Mcondition:1}.
	Finally, note that
	$$
	z^TM^{-1}z = x^T\bar{M}_{11}x + 2x^T\bar{M}_{12}^Ty + y^T\bar{M}_{22}y
				\le \lambda_{\max}(\bar{M}_{11})\|x\|_2^2 + 2x^T\bar{M}_{12}^Ty + \lambda_{\max}(\bar{M}_{12})\|y\|_2^2,
	$$
	or
	$$
	M^{-1} \preceq \begin{bmatrix} \lambda_{\max}(\bar{M}_{11})I & \bar{M}_{12} \\ \bar{M}_{12}^T & \lambda_{\max}(\bar{M}_{22})I \end{bmatrix}.
	$$
	Both $\lambda_{\max}(\bar{M}_{11})$ and $\lambda_{\max}(\bar{M}_{22})$ are positive because $\bar{M}_{11}, \bar{M}_{22} \succ 0$.
	Then the second inequality in \eqref{eqn:invMupper} holds if and only if either $\frac{1}{\nu} = \lambda_{\max}(\bar{M}_{11})$, $\bar{M}_{12}=0$, $\frac{1}{\epsilon}-\lambda_{\max}(\bar{M}_{22}) \ge 0$ or
	$\frac{1}{\nu} > \lambda_{\max}(\bar{M}_{11})$, $\bar{M}_{12}=0$,
	$\frac{1}{\epsilon}-\lambda_{\max} \ge (\frac{1}{\nu}-\lambda_{\max}(\bar{M}_{11}))^{-1}\|\bar{M}_{12}\|_2^2$
	\citep[][Appendix A]{Boyd:ConvexOptimization:2004}.
	Now because Condition \ref{prop:Mcondition:1} implies $\lambda_{\max}(\bar{M}_{11}) \le \frac{1}{\mu} < \frac{2}{L_f}$, we can choose $\nu$ and $\epsilon$ so that $\frac{1}{\mu} \le \frac{1}{\nu} < \frac{2}{L_f}$ and $\frac{1}{\epsilon} \ge \lambda_{\max}(\bar{M}_{22}) + (\frac{1}{\nu}-\lambda_{\max}(\bar{M}_{11}))^{-1}\|\bar{M}_{12}\|_2^2$. This implies \eqref{eqn:invMupper} and thus Condition \ref{prop:Mcondition:3}.
\end{proof}

%% file: sections/optimalproof.tex
\subsection{Optimal acceleration}

The following proposition plays a central role in proving Theorems \ref{thm:bddgen} and \ref{thm:unbddgen}.
\begin{proposition}\label{prop:chen41}
Assume that $\rho_k \le 1$ for any $k$. If $\tilde{z}^{k+1}$ is generated by \eqref{eqn:accgen}, then for any $z = (x,y) \in Z$, 
\begin{align}\label{eqn:prop:chen41}
\begin{split}
\rho_k^{-1} \mathcal{G}(\tilde{z}^{k+1}, z) - (\rho_k^{-1} -1)\mathcal{G}(\tilde{z}^k , z) 
\le & \langle \nabla f(x_{md}^k), \tilde{x}^{k+1} - x \rangle+ \frac{L_f\rho_k}{2} \|\tilde{x}^{k+1} - \tilde{x}^k\|_2^2 \\
&+ h^*(\tilde{y}^{k+1}) - h^*(y) + \langle K \tilde{x}^{k+1}, y \rangle - \langle Kx, \tilde{y}^{k+1} \rangle.
\end{split}
\end{align}
\end{proposition}
\begin{proof}
The result follows from \citet[Proposition 4.1]{chen2014optimal} as it involves only strong smoothness of $f$, convexity of $f$ and $h^*$, \eqref{eqn:mid}, \eqref{eqn:xag}, and \eqref{eqn:yag}. 
\end{proof}

The following lemmas find an upper bound for $\mathcal{G}(\tilde{z}^{k+1}, z)$.

\begin{lemma}[\citet{loris2011generalization}, Lemma 1]\label{lem:prox}
If $y^+ = \prox_{\sigma h^*}(y^- + \sigma \Delta)$, then
\begin{align}\label{eqn:prox}
\langle y - y^+, \Delta \rangle - h^*(y) + h^*(y^+) \le \frac{1}{2 \sigma} \left( \|y - y^- \|_2^2 - \|y - y^+\|_2^2 - \|y^- - y^+\|_2^2 \right)
\end{align}
for any $y$.
\end{lemma}

\begin{lemma}\label{lem:cos}
If $x^+ = x^- + \tau \Delta$, then
\begin{align}\label{eqn:cos}
\langle x - x^+ , \Delta \rangle = \frac{1}{2 \tau} \left( \| x - x^- \|_2^2 - \| x - x^+ \|_2^2 - \| x^+ - x^- \|_2^2 \right)
\end{align}
for any $x$.
\end{lemma}
\begin{lemma}\label{lem:Qboundgen}
If $\tilde{z}^{k+1}=(\tilde{x}^{k+1}, \tilde{y}^{k+1})$ is obtained by \eqref{eqn:accgen}, we have the following under the condition \eqref{eqn:bddcondgen}:
\begin{align}\label{eqn:Qboundgen}
\rho_k^{-1} & \gamma_k \mathcal{G}(\tilde{z}^{k+1}, z)  \le  \mathcal{D}_k (z, \tilde{z}^{[k]}) - \gamma_k \langle \tilde{x}^{k+1} - x, B^T(\tilde{y}^{k+1} - \tilde{y}^k) \rangle 
  + \gamma_k \langle A (\tilde{x}^{k+1} - \tilde{x}^k), \tilde{y}^{k+1}-y\rangle  \nonumber \\
&  +  \tau_k\gamma_k \langle (K+B)^T (\tilde{y}^{k+1} - \tilde{y}^k), (K+A)^T(\tilde{y}^{k+1}-y) \rangle 
 - \gamma_k \left(\frac{1 - q}{2 \tau_k}-\frac{L_f\rho_k}{2}\right) \|\tilde{x}^{k+1}-\tilde{x}^k\|_2^2  \nonumber \\
& - \gamma_k \left(\frac{1 - r}{2 \sigma_k}-\frac{\|K+A\|_2\|K+B\|_2 \tau_{k-1}}{2} \right)\|\tilde{y}^{k+1}-\tilde{y}^k\|_2^2, 
\end{align}
where $\gamma_k$ is defined by
\begin{align}\label{eqn:gamma}
\gamma_k = \begin{cases}1 &\text{if $k=1$} \\ \theta_k^{-1} \gamma_{k-1} & \text{if $k \ge 2$} \end{cases},
\end{align}
and $\mathcal{D}_k(z, \tilde{z}^{[k]})$ is defined by 
\begin{align}\label{eqn:Bdef}
\mathcal{D}_k(z, \tilde{z}^{[k]}) := \sum_{i=1}^k \left[ \frac{\gamma_i}{2 \tau_i} (\|x - \tilde{x}^i\|_2^2 - \|x - \tilde{x}^{i+1}\|_2^2) +  \frac{\gamma_i}{2 \sigma_i} (\|y - \tilde{y}^i\|_2^2 - \|y - \tilde{y}^{i+1}\|_2^2) \right].
\end{align}
\end{lemma}
\begin{proof}
For iteration \eqref{eqn:accgen}, the following relation holds by Lemma \ref{lem:prox} and Lemma \ref{lem:cos}:
\begin{align*}
\langle y - \tilde{y}^{k+1}, \tilde{u}^{k+1} \rangle + h^*(\tilde{y}^{k+1}) - h^*(y) & \le \frac{1}{2 \sigma_k} \left(\|y - \tilde{y}^k\|_2^2 - \|\tilde{y}^{k+1} - \tilde{y}^k\|_2^2 - \|y - \tilde{y}^{k+1}\|_2^2 \right), \\ 
\langle \tilde{x}^{k+1} - x, \nabla f(x_{md}^k)+\tilde{v}^{k+1} \rangle & = \frac{1}{2 \tau_k} \left(\|x - \tilde{x}^k\|_2^2 - \|\tilde{x}^{k+1} - \tilde{x}^k\|_2^2 - \|x - \tilde{x}^{k+1}\|_2^2 \right). 
\end{align*}
%
Using the above relationship along with Proposition \ref{prop:chen41}, we obtain the following. 
\begin{align}
\begin{split}\label{eqn:Qbound3}
\rho_k^{-1} \mathcal{G}(\tilde{z}^{k+1}, z) & - (\rho_k^{-1} -1) \mathcal{G}(\tilde{z}^k, z)  \\ 
&\le \frac{1}{2 \tau_k}\left(\|x - \tilde{x}^k \|_2^2 - \|x-\tilde{x}^{k+1}\|_2^2 \right) - \left(\frac{1}{2 \tau_k} - \frac{L_f\rho_k}{2}\right) \|\tilde{x}^{k+1} - \tilde{x}^k\|_2^2 \\
& + \frac{1}{2 \sigma_k}\left(\|y-\tilde{y}^k\|_2^2 - \|y-\tilde{y}^{k+1}\|_2^2 \right) - \frac{1}{2 \sigma_k} \|\tilde{y}^{k+1} - \tilde{y}^k \|_2^2 \\
& - \langle \tilde{x}^{k+1} - x,  \tilde{v}^{k+1} \rangle + \langle  \tilde{u}^{k+1}, \tilde{y}^{k+1} - y \rangle + \langle K\tilde{x}^{k+1}, y  \rangle - \langle Kx, \tilde{y}^{k+1} \rangle.
\end{split}
\end{align}
The sum of the four inner products on the last line, namely, $- \langle \tilde{x}^{k+1} - x,  \tilde{v}^{k+1} \rangle + \langle  \tilde{u}^{k+1}, \tilde{y}^{k+1} - y \rangle + \langle K\tilde{x}^{k+1}, y  \rangle - \langle Kx, \tilde{y}^{k+1} \rangle$, multiplied by $\gamma_k$ can be computed as follows. 
\begin{align*}
\gamma_k [- \langle \tilde{x}^{k+1} - x, & \tilde{v}^{k+1} \rangle + \langle  \tilde{u}^{k+1}, \tilde{y}^{k+1} - y \rangle + \langle K\tilde{x}^{k+1}, y  \rangle - \langle Kx, \tilde{y}^{k+1} \rangle ] \\
= & \; \gamma_k [- \left( \langle \tilde{x}^{k+1} - x, B^T(\tilde{y}^{k+1} - \tilde{y}^k) \rangle - \theta_k \langle \tilde{x}^{k+1}-x, B^T(\tilde{y}^k-\tilde{y}^{k-1})  \rangle \right) \\
& +  \left( \langle A(\tilde{x}^{k+1}-\tilde{x}^k), \tilde{y}^{k+1}-y \rangle - \theta_k \langle A(\tilde{x}^k - \tilde{x}^{k-1}), \tilde{y}^{k+1} - y \rangle \right) \\
& +   \tau_k \langle (K+A)(K+B)^T (\tilde{y}^{k+1}-\tilde{y}^k), \tilde{y}^{k+1} - y \rangle \\
& - \tau_{k-1} \theta_k \langle (K+A)(K+B)^T (\tilde{y}^k - \tilde{y}^{k-1}), \tilde{y}^{k+1} - y \rangle] \\
= & - \left( \gamma_k \langle \tilde{x}^{k+1} - x, B^T(\tilde{y}^{k+1} - \tilde{y}^k) \rangle - \gamma_{k-1}\langle \tilde{x}^k - x, B^T(\tilde{y}^k - \tilde{y}^{k-1}) \right) \\
& +  \left( \gamma_k \langle A (\tilde{x}^{k+1} - \tilde{x}^k), \tilde{y}^{k+1}-y\rangle - \gamma_{k-1} \langle A (\tilde{x}^k - \tilde{x}^{k-1}), \tilde{y}^k - y \rangle  \right) \\
& + \tau_k\gamma_k \langle (K+B)^T (\tilde{y}^{k+1} - \tilde{y}^k), (K+A)^T (\tilde{y}^{k+1}-y) \rangle \\
& - \tau_{k-1}\gamma_{k-1} \langle (K+B)^T (\tilde{y}^k - \tilde{y}^{k-1}), (K+A)^T(\tilde{y}^k - y) \rangle   \\
& +\gamma_{k-1} \langle \tilde{x}^{k+1}-\tilde{x}^k, B^T(\tilde{y}^k - \tilde{y}^{k-1}) \rangle -  \gamma_{k-1 }\langle A (\tilde{x}^k - \tilde{x}^{k-1}), \tilde{y}^{k+1} - \tilde{y}^k \rangle \\
& - \gamma_{k-1}\tau_{k-1}\langle (K+B)^T (\tilde{y}^k - \tilde{y}^{k-1}), (K+A)^T (\tilde{y}^{k+1} - \tilde{y}^k) \rangle.
\end{align*}
We used the relation 
\begin{align*}
\tilde{u}^{k+1} &= K\tilde{x}^{k+1} +A(\tilde{x}^{k+1}-\tilde{x}^k) - \theta_k A(\tilde{x}^k - \tilde{x}^{k-1}) \\ 
&+ \tau_k (K+A) (K+B)^T (\tilde{y}^{k+1}-\tilde{y}^k) - \theta_k \tau_{k-1} (K+A)(K+B)^T (\tilde{y}^k - \tilde{y}^{k-1}).
\end{align*}
in the first equality. 

By upper bounding the inner product terms, and noting that $\theta_k = \gamma_{k-1}/\gamma_k = \tau_{k-1}/\tau_k=\sigma_{k-1}/\sigma_k$, we have:

\begin{align}
\begin{split}\label{eqn:genericbound}
|\gamma_{k-1} \langle \tilde{x}^{k+1} - \tilde{x}^k, B^T(\tilde{y}^k - \tilde{y}^{k-1}) \rangle| & \le  \frac{\gamma_k q}{2 \tau_k} \| \tilde{x}^{k+1} - \tilde{x}^k \|_2^2 + \frac{\|B\|_2 ^2 \gamma_{k-1} \tau_{k-1}}{2q} \|\tilde{y}^k - \tilde{y}^{k-1}\|_2^2 \\
|\gamma_{k-1} \langle \tilde{x}^k - \tilde{x}^{k-1}, A^T(\tilde{y}^{k+1} - \tilde{y}^k) \rangle| & \le \frac{\|A\|_2^2 \gamma_{k-1} \sigma_{k-1}}{2r} \|\tilde{x}^k - \tilde{x}^{k-1}\|_2^2 + \frac{\gamma_k r}{2 \sigma_k} \|\tilde{y}^{k+1}- \tilde{y}^k\|_2^2 \\
|\gamma_{k-1} \tau_{k-1} \langle (K+B)^T (\tilde{y}^k - \tilde{y}^{k-1}),& (K+A)^T (\tilde{y}^{k+1} - \tilde{y}^k) \rangle|  \\
& \le \frac{\|K+A\|_2\|K+B\|_2 \gamma_{k-1} \tau_{k-1} \theta_k}{2} \|\tilde{y}^k- \tilde{y}^{k-1}\|_2^2  \\ 
& + \frac{\|K+A\|_2\|K+B\|_2 \gamma_{k-1} \tau_{k-1}}{2 \theta_k} \| \tilde{y}^{k+1} - \tilde{y}^k \|_2^2 
\end{split}
\end{align}
for some positive $q$ and $r$.
Thus 
\begin{align*}
\rho_k^{-1} \gamma_k & \mathcal{G}(\tilde{z}^{k+1}, z)  - (\rho_k^{-1} -1) \gamma_k \mathcal{G}(\tilde{z}^k, z) \\
\le & \frac{1}{2 \tau_k}\left(\|x - \tilde{x}^k \|_2^2 - \|x-\tilde{x}^{k+1}\|_2^2 \right)  + \frac{1}{2 \sigma_k}\left(\|y-\tilde{y}^k\|_2^2 - \|y-\tilde{y}^{k+1}\|_2^2 \right) \\
& -\left( \gamma_k \langle \tilde{x}^{k+1} - x, B^T(\tilde{y}^{k+1} - \tilde{y}^k) \rangle - \gamma_{k-1}\langle \tilde{x}^k - x, B^T(\tilde{y}^k - \tilde{y}^{k-1}) \right) \\
& + \left( \gamma_k \langle \tilde{x}^{k+1} - \tilde{x}^k, A^T(\tilde{y}^{k+1}-y)\rangle - \gamma_{k-1} \langle \tilde{x}^k - \tilde{x}^{k-1}, A^T(\tilde{y}^k - y) \rangle  \right) \\
& +  \tau_k\gamma_k \langle (K+B)^T (\tilde{y}^{k+1} - \tilde{y}^k), (K+A)^T (\tilde{y}^{k+1}-y) \rangle \\
& - \tau_{k-1}\gamma_{k-1} \langle (K+B)^T (\tilde{y}^k - \tilde{y}^{k-1}), (K+A)^T(\tilde{y}^k - y) \rangle   \\
& - \gamma_k \left(\frac{1 - q}{2\tau_k} - \frac{L_f\rho_k}{2}\right)\|\tilde{x}^{k+1} - \tilde{x}^k\|_2^2 + \frac{\|A\|_2^2\gamma_{k-1}\sigma_{k-1}}{2 r}\|\tilde{x}^k-\tilde{x}^{k-1}\|_2^2 \\
& - \gamma_k \left(\frac{1 - r}{2\sigma_k} - \frac{\|K+A\|_2\|K+B\|_2 \tau_{k-1}}{2}\right)\|\tilde{y}^{k+1}-\tilde{y}^k\|_2^2 \\
& + \frac{\gamma_{k-1} \tau_{k-1}}{2}\left(\frac{\|B\|_2^2}{q} + \|K+A\|_2\|K+B\|_2  \theta_k \right) \|\tilde{y}^k - \tilde{y}^{k-1}\|_2^2. 
\end{align*}

Recursively applying the above relation, we obtain:
\begin{align*}
\rho_k^{-1} \gamma_k & \mathcal{G}(\tilde{z}^{k+1}, z) \\
\le & \; \mathcal{D}_k (z, \tilde{z}^{[k]}) - \gamma_k (\langle \tilde{x}^{k+1} - x, B^T(\tilde{y}^{k+1} - \tilde{y}^k) \rangle - \langle \tilde{x}^{k+1} - \tilde{x}^k, A^T(\tilde{y}^{k+1}-y)\rangle \\
& - \tau_k \langle (K+B)^T (\tilde{y}^{k+1} - \tilde{y}^k), (K+A)^T(\tilde{y}^{k+1}-y) \rangle  ) \\ 
& - \gamma_k \left(\frac{1 - q}{2 \tau_k}-\frac{L_f\rho_k}{2}\right) \|\tilde{x}^{k+1}-\tilde{x}^k\|_2^2 - \gamma_k \left(\frac{1 - r}{2 \sigma_k}-\frac{\|K+A\|_2\|K+B\|_2 \tau_{k-1}}{2} \right)\|\tilde{y}^{k+1}-\tilde{y}^k\|_2^2\\
& - \sum_{i=1}^{k-1} \gamma_i \left(\frac{1 - q}{2 \tau_i}-\frac{L_f\rho_k}{2}-\frac{\|A\|_2^2 \sigma_i}{2r} \right) \|\tilde{x}^{i+1} - \tilde{x}^i\|_2^2 \\
& - \sum_{i=1}^{k-1} \gamma_i \left(\frac{1 - r}{2\sigma_i}-\frac{\|K+A\|_2\|K+B\|_2\tau_{i-1}}{2} - \frac{\tau_i}{2}\left(\frac{\|B\|_2^2}{q}+ \|K+A\|_2\|K+B\|_2\theta_i\right) \right)\|\tilde{y}^{i+1}-\tilde{y}^i\|_2^2.
\end{align*}
Thus by the conditions \eqref{eqn:bddcondgen}, the desired result holds.
\end{proof}

\iftrue
\begin{proof}[Proof of Theorem \ref{thm:bddgen}]
First we find an upper bound of $\mathcal{D}_k(z, \tilde{z}^{[k]})$. 
\begin{align}
\mathcal{D}_k (z, \tilde{z}^{[k]}) & = \frac{\gamma_1}{2\tau_1} \|x- \tilde{x}^1\|_2^2 - \sum_{i=1}^{k-1}\frac{1}{2} \left(\frac{\gamma_i}{\tau_i}-\frac{\gamma_{i+1}}{\tau_{i+1}}\right) \|x - \tilde{x}^{i+1}\|_2^2 - \frac{\gamma_k}{2\tau_k} \|x - \tilde{x}^{k+1}\|_2^2 \nonumber \\
& \quad + \frac{\gamma_1}{2\sigma_1} \|y- \tilde{y}^1\|_2^2 - \sum_{i=1}^{k-1}\frac{1}{2} \left(\frac{\gamma_i}{\sigma_i}-\frac{\gamma_{i+1}}{\sigma_{i+1}}\right) \|y - \tilde{y}^{i+1}\|_2^2 - \frac{\gamma_k}{2\sigma_k} \|y - \tilde{y}^{k+1}\|_2^2 \nonumber \\
& \le \frac{\gamma_1}{\tau_1} \Omega_X^2 - \sum_{i=1}^{k-1} \left(\frac{\gamma_i}{\tau_i}-\frac{\gamma_{i+1}}{\tau_{i+1}}\right)  \Omega_X^2 - \frac{\gamma_k}{2\tau_k} \|x - \tilde{x}^{k+1}\|_2^2 \nonumber\\
& \quad + \frac{\gamma_1}{\sigma_1} \Omega_Y^2 - \sum_{i=1}^{k-1} \left(\frac{\gamma_i}{\sigma_i}-\frac{\gamma_{i+1}}{\sigma_{i+1}}\right)  \Omega_Y^2 - \frac{\gamma_k}{2\sigma_k} \|y - \tilde{y}^{k+1}\|_2^2 \nonumber\\
& = \frac{\gamma_k}{\tau_k}\Omega_X^2 + \frac{\gamma_k}{\sigma_k}\Omega_Y^2 - \gamma_k \left(\frac{1}{2\tau_k}\|x - \tilde{x}^{k+1}\|_2^2 +  \frac{1}{2 \sigma_k}\|y - \tilde{y}^{k+1}\|_2^2 \right),\label{eqn:Bbound}
\end{align}
where we used \eqref{eqn:bdd} for the inequality.

Consider the following upper bounds of the three inner product terms in \eqref{eqn:Qboundgen}:
\begin{align}\label{eqn:bddinner}
\begin{split}
|\gamma_k   \langle \tilde{x}^{k+1} - x, B^T(\tilde{y}^{k+1} - \tilde{y}^k) \rangle| &\le \frac{\gamma_k q}{2 \tau_k} \|\tilde{x}^{k+1}-x\|_2^2 + \frac{\|B\|_2^2 \gamma_k \tau_k}{2q} \|\tilde{y}^{k+1}-\tilde{y}^k\|_2^2 \\
|\gamma_k \langle \tilde{x}^{k+1} - \tilde{x}^k, A^T(\tilde{y}^{k+1}-y) \rangle| & \le  \frac{\|A\|_2^2 \gamma_k \sigma_k}{2r} \|\tilde{x}^{k+1} - \tilde{x}^k \|_2^2 + \frac{\gamma_k r}{2 \sigma_k} \|\tilde{y}^{k+1}-y \|_2^2 \\
|\tau_k \langle (K+B)^T(\tilde{y}^{k+1} - \tilde{y}^k), & (K+A)^T(\tilde{y}^{k+1}-y) \rangle|  \\
\le & \frac{\|K+A\|_2\|K+B\|_2 \gamma_k \tau_k}{2}\|\tilde{y}^{k+1} - \tilde{y}^k \|_2^2 \\
& + \frac{\|K+A\|_2\|K+B\|_2 \gamma_k \tau_k}{2} \|\tilde{y}^{k+1}-y\|_2^2.
\end{split}
\end{align}
Then \eqref{eqn:rhocondgen}, \eqref{eqn:Qboundgen}, \eqref{eqn:Bbound}, and \eqref{eqn:bddinner} imply that 
\begin{align*}
\gamma_k \mathcal{G}(\tilde{z}^{k+1}, z)  \le & \;  \frac{\gamma_k}{\tau_k} \Omega_X^2 + \frac{\gamma_k}{\sigma_k} \Omega_Y^2 - \gamma_k \frac{1-q}{2 \tau_k}\|x-\tilde{x}^{k+1}\|_2^2 \\
& - \gamma_k \left(\frac{1-r}{2 \sigma_k} - \frac{\|K+A\|_2\|K+B\|_2 \tau_k}{2}\right) \|y - \tilde{y}^{k+1}\|_2^2 \\
& - \gamma_k \left(\frac{1-q}{2 \tau_k} - \frac{L_f\rho_k}{2} - \frac{\|A\|_2^2 \sigma_k}{2} \right) \|\tilde{x}^{k+1}-\tilde{x}^k \|_2^2 \\
& - \gamma_k \left( \frac{1-r}{2 \sigma_k}-\frac{\tau_k}{2} \left( 2\|K+A\|_2\|K+B\|_2 + \|B\|_2^2\right)\right) \|\tilde{y}^{k+1}-\tilde{y}^k \|_2^2 \\
 \le & \;  \frac{\gamma_k}{\tau_k} \Omega_X^2 + \frac{\gamma_k}{\sigma_k} \Omega_Y^2.
\end{align*}
That is, \eqref{eqn:bddupper}.
\end{proof}
\begin{proof}[Proof of Corollary \ref{cor:bddgen}]
First check \eqref{eqn:bddparams} and
\eqref{eqn:bddparamcond} satisfy \eqref{eqn:bddcondgen}:
\begin{align*}
\frac{1-q}{\tau_k} - L_f \rho_k - \frac{\|A\|_2^2 \sigma_k}{r} 
& \ge  \left((1-q) Q - \frac{a^2}{r} \right)\frac{\Omega_X \|K\|_2}{\Omega_Y} \ge 0, \\
\frac{1-r}{\sigma_k} - \tau_k\left(2\|K+A\|_2\|K+B\|_2 + \frac{\|B\|_2^2}{q} \right) 
& \ge  \left(1-r - \frac{2cd + b^2/q}{Q}\right) \frac{\Omega_X \|K\|_2}{\Omega_Y} \ge 0,
\end{align*}
Then by \eqref{eqn:bddupper}, we have 
\begin{align*}
\mathcal{G}^\star (\tilde{z}^k) & \ge  \frac{\rho_{k-1}}{\tau_{k-1}} \Omega_X^2 + \frac{\rho_{k-1}}{\sigma_{k-1}} \Omega_Y^2 \\
& =  \frac{4 P L_f + 2 Q (k-1) \|K\|_2 \Omega_Y/\Omega_X}{k(k-1)} \Omega_X^2 + \frac{2 \|K\|_2 \Omega_X /\Omega_Y}{k} \Omega_Y^2 \\
&= \frac{4P \Omega_X^2}{k(k-1)}L_f + \frac{2\Omega_X \Omega_Y(Q+1)}{k}\|K\|_2.
\end{align*}
\end{proof}
\fi

We need the following lemma to prove Theorem \ref{thm:unbddgen}.
\begin{lemma}\label{lem:unbddgen}
Consider a saddle point $\hat{z} = (\hat{x}, \hat{y})$ of the problem \eqref{eqn:saddlepoint}, and the parameters $\rho_k$, $\theta_k$, $\tau_k$, and $\sigma_k$ satisfying the conditions for Theorem \ref{thm:unbddgen}. Then 
\begin{align}\label{eqn:lem:unbddgen:cond1}
\|x - \tilde{x}^1\|_2^2 + \frac{\tau_k}{\sigma_k}\|y-\tilde{y}^1\|_2^2 & \ge (1 - q)\|x-\tilde{x}^{k+1}\|_2^2 + \frac{\tau_k}{\sigma_k} \left(\frac{1}{2}-r\right) \|y - \tilde{y}^{k+1}\|_2^2 
\end{align} 
and 
\begin{align}\label{eqn:lem:unbddgen:cond2}
\tilde{\mathcal{G}}(\tilde{z}^{k+1}, v^{k+1}) & \le \frac{\rho_k}{2 \tau_k}\|x^{k+1}-\tilde{x}^1\|_2^2 + \frac{\rho_k}{2 \sigma_k} \|y^{k+1}-\tilde{y}^1\|_2^2=:\delta_{k+1}
\end{align}
for all $t \ge 1$, where $\tilde{\mathcal{G}}$ is defined in \eqref{eqn:pgap}, and 
\begin{align}\label{eqn:vdefgen}
v^{k+1} &=  \left( \frac{\rho_k}{\tau_k}(\tilde{x}^1 - \tilde{x}^{k+1}) - B^T (\tilde{y}^{k+1}-\tilde{y}^k), \right.  \nonumber\\
& \quad\quad \left. \frac{\rho_k}{\sigma_k}(\tilde{y}^1 - \tilde{y}^{k+1})+ A(\tilde{x}^{k+1}-\tilde{x}^k) + (K+A)(K+B)^T (\tilde{y}^{k+1}-\tilde{y}^k) \right)
\end{align}

\end{lemma}

\begin{proof}
First, let us prove \eqref{eqn:lem:unbddgen:cond1}.
The conditions for Lemma \ref{lem:Qboundgen} clearly holds. Note that 
\begin{align}\label{eqn:Bexpand}
\begin{split}
\mathcal{D}_k (z, \tilde{z}^{[k]}) &= \frac{\gamma_1}{2\tau_1} \|x - \tilde{x}^1\|_2^2 - \sum_{i=1}^{k-1} \left(\frac{\gamma_i}{2 \tau_i} - \frac{\gamma_{i+1}}{2 \tau_{i+1}} \right) \|x - \tilde{x}^{k+1}\|_2^2 - \frac{\gamma_k}{2 \tau_k} \|x - \tilde{x}^{k+1}\|_2^2 \\
& \quad  + \frac{\gamma_1}{2\sigma_1} \|y - \tilde{y}^1\|_2^2 - \sum_{i=1}^{k-1} \left(\frac{\gamma_i}{2 \sigma_i} - \frac{\gamma_{i+1}}{2 \sigma_{i+1}} \right) \|y - \tilde{y}^{k+1}\|_2^2 - \frac{\gamma_k}{2 \sigma_k} \|y - \tilde{y}^{k+1}\|_2^2.
\end{split}
\end{align}
By \eqref{eqn:thetacondunbdd}, one may see that 
\begin{align*}
\mathcal{D}_k (z, z^{[k]}) &= \frac{1}{2 \tau_k}\|x-\tilde{x}^1\|_2^2 - \frac{1}{2 \tau_k}\|x - \tilde{x}^{k+1}\|_2^2 + \frac{1}{2 \sigma_k}\| y - \tilde{y}^1 \|_2^2 - \frac{1}{2 \sigma_k}\|y - \tilde{y}^{k+1}\|_2^2.
\end{align*}
Thus \eqref{eqn:Qboundgen} is equivalent to 
\begin{align*}
\rho_k^{-1}  \mathcal{G}(\tilde{z}^{k+1}, z) \le & \; \frac{1}{2 \tau_k}\|x-\tilde{x}^1\|_2^2 - \frac{1}{2 \tau_k}\|x - \tilde{x}^{k+1}\|_2^2 + \frac{1}{2 \sigma_k}\| y - \tilde{y}^1 \|_2^2 - \frac{1}{2 \sigma_k}\|y - \tilde{y}^{k+1}\|_2^2  \\
& - \gamma_k \langle \tilde{x}^{k+1} - x, B^T(\tilde{y}^{k+1} - \tilde{y}^k) \rangle + \gamma_k \langle A (\tilde{x}^{k+1} - \tilde{x}^k), \tilde{y}^{k+1}-y\rangle\\  
&   + \tau_k\gamma_k \langle (K+B)^T (\tilde{y}^{k+1} - \tilde{y}^k), (K+A)^T(\tilde{y}^{k+1}-y) \rangle \\
& - \gamma_k \left(\frac{1 - q}{2 \tau_k}-\frac{L_f\rho_k}{2}\right) \|\tilde{x}^{k+1}-\tilde{x}^k\|_2^2 \\
& - \gamma_k \left(\frac{1 - r}{2 \sigma_k}-\frac{\|K+A\|_2\|K+B\|_2 \tau_{k-1}}{2} \right)\|\tilde{y}^{k+1}-\tilde{y}^k\|_2^2.
\end{align*}

Note that 
\begin{align}\label{eqn:unbddinner}
\begin{split}
\left| \langle A(\tilde{x}^{k+1}-\tilde{x}^k), \tilde{y}^{k+1}-y \rangle \right| &\le \frac{\|A\|_2^2 \sigma_k}{2r} \|\tilde{x}^{k+1}-\tilde{x}^k\|_2^2 + \frac{r}{2\sigma_k} \|\tilde{y}^{k+1}-y\|_2^2 \\
| \tau_k \langle (K+B)^T (\tilde{y}^{k+1}- \tilde{y}^k), & \; (K +A)^T(\tilde{y}^{k+1} -y) \rangle |  \\
&\le \tau_k^2 \sigma_k \|K+A\|_2^2 \|K+B\|_2^2 \|\tilde{y}^{k+1}-\tilde{y}^k\|_2^2  + \frac{1}{4\sigma_k}\|\tilde{y}^{k+1}-y\|_2^2 \quad\quad \\
\left| \langle \tilde{x}^{k+1}-x, B^T(\tilde{y}^{k+1}-\tilde{y}^k) \rangle \right| &\le \frac{q}{2 \tau_k} \|\tilde{x}^{k+1}-x\|_2^2 + \frac{\|B\|_2^2 \tau_k}{2q} \| \tilde{y}^{k+1}-\tilde{y}^k \|_2^2.
\end{split}
\end{align}
Thus 
\begin{align*}
\rho_k^{-1} & \mathcal{G}(\tilde{z}^{k+1}, z) 
\le  \frac{1}{2 \tau_k} \|x-\tilde{x}^1\|_2^2 - \frac{1-q}{2 \tau_k} \|x - \tilde{x}^{k+1}\|_2^2 \\
& + \frac{1}{2 \sigma_k} \|y - \tilde{y}^1\|_2^2 - \frac{1}{2 \sigma_k}\left( 1-r - \frac{1}{2} \right) \|y - \tilde{y}^{k+1}\|_2^2 \\
& - \left( \frac{1-q}{2 \tau_k}-\frac{L_f\rho_k}{2} - \frac{\|A\|_2^2\sigma_k}{2r}  \right)\|\tilde{x}^{k+1}-\tilde{x}^k\|_2^2 \\
& - \left(\frac{1-r}{2 \sigma_k} - \frac{\|K+A\|_2\|K+B\|_2\tau_{k-1}}{2} - \frac{\|B\|_2^2 \tau_k}{2q} - \|K+A\|_2^2\|K+B\|_2^2 \tau_k^2 \sigma_k  \right)\|\tilde{y}^{k+1}-\tilde{y}^k\|_2^2.
\end{align*}
It can be easily seen that  
\begin{align*}
\frac{1-r}{2 \sigma_k} - \frac{\|K+A\|_2\|K+B\|_2\tau_{k-1}}{2} &- \frac{\|B\|_2^2 \tau_k}{2q} - \|K+A\|_2^2 \|K+B\|_2^2 \tau_k^2 \sigma_k  \\
& \ge \frac{1-r}{2 \sigma_k} - \frac{\tau_k \|B\|_2^2}{2q} - \tau_k \|K+A\|_2\|K+B\|_2\ge 0.
\end{align*}
Hence 
\[
\rho_k^{-1} \mathcal{G}(\tilde{z}^{k+1}, z) \le \frac{1}{2 \tau_k} \|x-\tilde{x}^1\|_2^2 - \frac{1-q}{2 \tau_k}\|x - \tilde{x}^{k+1}\|_2^2 + \frac{1}{2 \sigma_k} \|y-\tilde{y}^1\|_2^2 - \frac{1/2-r}{2 \sigma_k}\|y - \tilde{y}^{k+1}\|_2^2.
\]
Since $\mathcal{G}(\tilde{z}^{k+1}, \hat{z}) \ge 0$, we obtain 
\[
\|x - \tilde{x}^1\|_2^2 + \frac{\tau_k}{\sigma_k}\|y-\tilde{y}^1\|_2^2 \ge (1 - q)\|x-\tilde{x}^{k+1}\|_2^2 + \frac{\tau_k}{\sigma_k} (1/2-r) \|y - \tilde{y}^{k+1}\|_2^2. 
\]
Next, we prove \eqref{eqn:lem:unbddgen:cond2}. Note that 
\begin{align}
\begin{split}\label{eqn:normid2}
\|x - \tilde{x}^1 \|_2^2 - \|x - \tilde{x}^{k+1}\|_2^2 & = 2 \langle \tilde{x}^{k+1} - \tilde{x}^1, x - x^{k+1} \rangle + \|x^{k+1} - \tilde{x}^1 \|_2^2 - \| x^{k+1} - \tilde{x}^{k+1} \|_2^2 \\
\|y - \tilde{y}^1 \|_2^2 - \|y - \tilde{y}^{k+1}\|_2^2 & = 2 \langle \tilde{y}^{k+1} - \tilde{y}^1, y - y^{k+1} \rangle + \|y^{k+1} - \tilde{y}^1 \|_2^2 - \| y^{k+1} - \tilde{y}^{k+1} \|_2^2.
\end{split}
\end{align}
From this, we have: 
\begin{align*}
& \rho_k^{-1} \mathcal{G}(\tilde{z}^{k+1}, z) - \frac{1}{\tau_k} \langle \tilde{x}^1 - \tilde{x}^{k+1}, x^{k+1} -x \rangle - \frac{1}{\sigma_k} \langle \tilde{y}^1-\tilde{y}^{k+1}, y^{k+1}-y  \rangle \\
& \quad - \langle x - x^{k+1}, B^T(\tilde{y}^{k+1}-\tilde{y}^k) \rangle +  \langle A(\tilde{x}^{k+1}-\tilde{x}^k), y-y^{k+1}\rangle \\
& \quad +  \tau_k \langle (K+B)^T(\tilde{y}^{k+1}-\tilde{y}^k), (K+A)^T(y-y^{k+1})\rangle \\
& \le  \frac{1}{2 \tau_k} \left( \|x^{k+1} - \tilde{x}^1 \|_2^2 - \| x^{k+1}-\tilde{x}^{k+1}\|_2^2 \right) + \frac{1}{2 \sigma_k} \left(\|y^{k+1} - \tilde{y}^1\|_2^2 - \|y^{k+1}-\tilde{y}^{k+1}\|_2^2 \right) \\
& \quad - \left(\frac{1-q}{2 \tau_k}-\frac{L_f\rho_k}{2} \right) \|\tilde{x}^{k+1}-\tilde{x}^k\|_2^2 \\
& \quad - \left(\frac{1-r}{2 \sigma_k}-\frac{\|K+A\|_2\|K+B\|_2 \tau_{k-1}}{2} \right) \|\tilde{y}^{k+1}-\tilde{y}^k\|_2^2 \\
& \quad -  \langle \tilde{x}^{k+1}-x^{k+1}, B^T(\tilde{y}^{k+1}-\tilde{y}^k) \rangle +  \langle A(\tilde{x}^{k+1}-\tilde{x}^k), \tilde{y}^{k+1} - y^{k+1}\rangle \\
& \quad +  \tau_k \langle  (K+B)^T (\tilde{y}^{k+1}-\tilde{y}^k), (K+A)^T(\tilde{y}^{k+1}-y^{k+1}) \rangle \\
& \le  \frac{1}{2 \tau_k} \|x^{k+1}-\tilde{x}^k\|_2^2 + \frac{1}{2 \sigma_k} \|y^{k+1}-\tilde{y}^1\|_2^2 \\
& \quad - \frac{1-q}{2 \tau_k} \|x^{k+1}-\tilde{x}^{k+1}\|_2^2 - \frac{1/2-r}{2 \sigma_k} \|y^{k+1}-\tilde{y}^{k+1}\|_2^2 \\
& \quad - \left( \frac{1-q}{2 \tau_k} - \frac{L_f\rho_k}{2} - \frac{\|A\|_2^2 \sigma_k}{2r}\right)\|\tilde{x}^{k+1} - \tilde{x}^k\|_2^2 \\
& \quad - \left(\frac{1-r}{2 \sigma_k}-\frac{\|K+A\|_2\|K+B\|_2\tau_{k-1}}{2}-\frac{\|B\|_2^2\tau_k}{2q}-\|K+A\|_2^2 \|K+B\|_2^2 \tau_k^2 \sigma_k \right) \|\tilde{y}^{k+1}-\tilde{y}^k\|_2^2\\
& \le \frac{1}{2 \tau_k} \|x^{k+1} -\tilde{x}^1\|_2^2 + \frac{1}{2 \sigma_k} \|y^{k+1}-\tilde{y}^1\|_2^2.
\end{align*}
In the penultimate inequality, the upper bound for inner product terms similar to  \eqref{eqn:unbddinner} was used.
\end{proof}

\begin{proof}[Proof of Theorem \ref{thm:unbddgen}]
It is sufficient to find upper bounds of $\|v^{k+1}\|_2$ and $\delta_{k+1}$. From the definition of $R$ and \eqref{eqn:lem:unbddgen:cond1}, we have
$\|\hat{x}-\tilde{x}^{k+1}\|_2 \le \mu R$ and
$\|\hat{y}-\tilde{y}^{k+1}\|_2 \le \sqrt{\frac{\sigma_k}{\tau_k}}\nu R$. 
For $v^{k+1}$ defined in \eqref{eqn:vdefgen}, 
\begin{align*}
\|v^{k+1}\|_2 & \le  \rho_k(\frac{1}{\tau_k} \|\tilde{x}^1 - \tilde{x}^{k+1}\|_2 + \|B\|_2 \|\tilde{y}^{k+1}-\tilde{y}^k\|_2 \\
& \quad + \frac{1}{\sigma_k}\|\tilde{y}^1 - \tilde{y}^{k+1}\|_2 +  \|A\|_2\|\tilde{x}^{k+1}-\tilde{x}^k\|_2 + \|K+A\|_2\|K+B\|_2 \tau_k \|\tilde{y}^{k+1}-\tilde{y}^k\|_2) \\
& \le  \rho_k( \frac{1}{\tau_k}(\|\hat{x}-\tilde{x}^1\|_2+\|\hat{x}-\tilde{x}^{k+1}\|_2)+\frac{1}{\sigma_k}(\|\hat{y} - \tilde{y}^1\|_2 + \|\hat{y}-\tilde{y}^{k+1}\|_2) \\
& \quad + \|A\|_2(\|\hat{x}-\tilde{x}^{k+1}\|_2 + \|\hat{x}-\tilde{x}^k\|_2) + (\|B\|_2 + \|K+A\|_2\|K+B\|_2 \tau_k)(\|\hat{y} - \tilde{y}^{k+1}\|_2 + \|\hat{y}-\tilde{y}^k\|_2) \\
& \le  \frac{\rho_k}{ \tau_k}\|\hat{x}-\tilde{x}^1\|_2 + \frac{\rho_k}{ \sigma_k} \|\hat{y}-\tilde{y}^1\|_2 \\
& \quad + \rho_k\left(\frac{1}{\tau_k}+2\|A\|_2\right) \mu R + \rho_k\left( \frac{1}{\sigma_k} + 2 \|B\|_2 + 2\|K+A\|_2\|K+B\|_2 \tau_k\right) \nu R \\
&= \frac{\rho_k}{ \tau_k}\|\hat{x}-\tilde{x}^1\|_2 + \frac{\rho_k}{ \sigma_k} \|\hat{y}-\tilde{y}^1\|_2 \\
& \quad + R \left[\frac{\rho_k}{ \tau_k}\left(\mu  + \frac{\tau_1}{\sigma_1} \nu \right) + 2\rho_k \left( \|A\|_2\mu  + \|B\|_2 \nu \right) + 2 \tau_k\rho_k\|K+A\|_2\|K+B\|_2 \nu \right],  
\end{align*}
i.e., \eqref{eqn:vbound}. In the last equality, we used 
\[
\frac{1}{\sigma_k} = \frac{\tau_k}{\sigma_k}\frac{1}{\tau_k} = \frac{\tau_1}{\sigma_1}\frac{1}{\tau_k}.
\]

Now, we find an upper bound for $\delta_{k+1}$ defined in \eqref{eqn:lem:unbddgen:cond2}. 
\begin{align*}
\delta_{k+1} & = \frac{\rho_k}{2 \tau_k} \|x^{k+1} - \tilde{x}^1\|_2^2 + \frac{\rho_k}{2 \sigma_k} \|y^{k+1}- \tilde{y}^1\|_2^2 \\
& \le  \frac{\rho_k}{ \tau_k} \left( \|\hat{x}-x^{k+1}\|_2^2 + \|\hat{x}-\tilde{x}^1\|_2^2 \right) + \frac{\rho_k}{ \sigma_k}\left( \|\hat{y}-y^{k+1}\|_2^2 + \|\hat{y}-\tilde{y}^1\|_2^2 \right) \\
&= \frac{1}{\tau_k} (R^2 + (1-q) \|\hat{x}-x^{k+1}\|_2^2 +\frac{\tau_k}{\sigma_k} (1/2-r) \|\hat{y}-y^{k+1}\|_2^2 \\
& \quad + q \|\hat{x}-x^{k+1}\|_2^2 + \frac{\tau_k}{\sigma_k}(r + 1/2) \|\hat{y}-y^{k+1}\|_2^2 ) \\
&\le \frac{\rho_k}{ \tau_k}[ R^2 + \frac{\rho_k}{ \gamma_k} \sum_{i=1}^k \gamma_i [(1-q) \|\hat{x}-\tilde{x}^{i+1}\|_2^2 + \frac{\tau_k}{\sigma_k} (1/2-r)\|\hat{y}-\tilde{y}^{i+1}\|_2^2 \\
& \quad + q\|\hat{x}-\tilde{x}^{i+1}\|_2^2 +\frac{\tau_k}{\sigma_k}(r+1/2)\|\hat{y}-\tilde{y}^{i+1}\|_2^2]] \\
& \le  \frac{\rho_k}{ \tau_k}[R^2 + \frac{\rho_k}{ \gamma_k} \sum_{i=1}^k \gamma_i [R^2 + q \|\hat{x}-\tilde{x}^{i+1}\|_2^2 + \frac{\tau_k}{\sigma_k} (r+1/2)\|\hat{y}-\tilde{y}^{i+1}\|_2^2]] \\
&\le \frac{\rho_k}{ \tau_k} \left[2 + q \mu ^2 + (r+1/2) \nu ^2 \right]R^2 \\
&= \frac{\rho_k}{ \tau_k} \left[2 + \frac{q}{1-q} + \frac{r+1/2}{1/2-r}\right]R^2,
\end{align*}
i.e., \eqref{eqn:unbddupper}. In the second inequality, we used 
\[
x^{k+1} = \frac{\rho_k}{ \gamma_k} \sum_{i=1}^k \gamma_i \tilde{x}^{i+1},\;\;\;y^{k+1} = \frac{\rho_k}{ \gamma_k} \sum_{i=1}^k \gamma_i \tilde{y}^{i+1},\;\;\;\text{and} \quad \frac{\rho_k}{ \gamma_k} \sum_{i=1}^k \gamma_i = 1.
\]
\end{proof}
\begin{proof}[Proof of Corollary \ref{cor:unbddgen}]
First check if \eqref{eqn:unbddparams} and \eqref{eqn:unbddparamcond} satisfy \eqref{eqn:bddcondgen} and \eqref{eqn:thetacondunbdd}. Conditions \eqref{eqn:thetacondunbdd} and \eqref{eqn:rhocondgen} are trivial to see. To prove \eqref{eqn:condgen1} and \eqref{eqn:condgen2}: 
\begin{align*}
\frac{1-q}{\tau_k} - L_f\rho_k - \frac{\|A\|_2^2 \sigma_k}{r} 
& \ge  \|K\|_2 \left( (1-q) Q \frac{N}{k} - \frac{a^2 k}{rN}\right) \\
& \ge  \|K\|_2 \left((1-q)Q - \frac{a^2}{2r}\right) \ge 0,
\end{align*}
and
\begin{align*}
\frac{1-r}{\sigma_k} - & \tau_k\left(2\|K+A\|_2\|K+B\|_2 + \frac{\|B\|_2^2}{q} \right) \\
& \ge  \left(\frac{(1-r)QN}{k} - \frac{(2cd+b^2/q)k}{N}\right) \|K\|_2 \ge \left((1-r)Q - (2cd+b^2/q)\right) \|K\|_2.
\end{align*}
Condition \eqref{eqn:unbddparamcond} also implies that $\tau_k \le \sigma_k$.

Note that 
\begin{align}\label{eqn:rhobounds}
\begin{split}
\frac{\rho_N}{\tau_N} & \le  \frac{4PL_f}{N^2} + \frac{2 Q \|K\|_2}{N} \\
\|K\|_2^2 \rho_N \tau_N & \le \frac{2N\|K\|_2^2}{(2PL_f + QN\|K\|_2)(N+1)} \le \frac{2\|K\|_2}{QN} \\ 
\rho_N \|K\|_2 & \le \frac{2\|K\|_2}{N}.
\end{split}
\end{align}
When we put $\|A\|_2 \le a\|K\|_2$, $\|B\|_2 \le b\|K\|_2$, $\|K+A\|_2 \le c\|K\|_2$, and $\|K+B\|_2 \le d\|K\|_2$, $\|v^{k+1}\|_2$ is bounded above by 
\begin{align*}
\|v^{k+1}\|_2 & \le \frac{\rho_k}{\tau_k} \left( \|\hat{x} - \tilde{x}^1\|_2 + \|\hat{y} - \tilde{y}^1\|_2\right) \\
& \quad + R \left[ \frac{\rho_k}{\tau_k} \left(\mu + \frac{\tau_1}{\sigma_1} \right) + 2 \rho_k \|K\|_2(a\mu + b\nu) + 2 \tau_k \rho_k \|K\|_2^2 cd\nu \right].
\end{align*}
Thus by \eqref{eqn:rhobounds}, we have
\begin{align*}
\epsilon_{N+1} \le \delta_{N+1} \le \left(\frac{4PL_f}{N^2} + \frac{2Q\|K\|_2}{N}\right) \left[2+\frac{q}{1-q} + \frac{r+1/2}{1/2-r}\right] R^2,
\end{align*}
which is \eqref{eqn:epsfbound}, and
\begin{align*}
\|v^{N+1}\|_2 & \le \frac{4PL_f}{N^2} \left[\left( \|\hat{x} - \tilde{x}^1\|_2 + \|\hat{y} - \tilde{y}^1\|_2\right) + R\left(\mu + \frac{\tau_1}{\sigma_1} \nu\right) \right] \\
&\quad + \frac{\|K\|_2}{N} \left[2Q\left(\left( \|\hat{x} - \tilde{x}^1\|_2 + \|\hat{y} - \tilde{y}^1\|_2\right) + R\left(\mu +\frac{\tau_1}{\sigma_1}\nu\right)\right) + 4R(a\mu + b\nu) + \frac{4Rcd\nu}{Q} \right],
\end{align*}
which is \eqref{eqn:vfbound}.
\end{proof}

\begin{proof}[Proof of Proposition \ref{prop:approxopt}]
The result follows directly from Proposition 3.13, Definition 3.4, Proposition 3.5, and Proposition 3.6 of \citet{monteiro2011complexity}. 
\end{proof}

%% file: sections/stocproof.tex
\subsection{Stochastic optimal acceleration}
We obtain a bound similar to Lemma \ref{lem:Qboundgen} first.  
The following lemma provides an upper bound on $\rho_k^{-1} \gamma_k \mathcal{G}(z^k, z)$. 

\begin{lemma}\label{lem:stocQbound}
Assume that $z^k = (x^k, y^k)$ is the iterates generated by the iteration \eqref{eqn:stocgen}. Also assume that the parameters satisfy
\eqref{eqn:rhocondgen}
\eqref{eqn:thetacondunbdd},  and \eqref{eqn:stocparam}. Then for any $z \in Z$, we have 

\begin{align}\label{eqn:stocQbound}
\begin{split}
\rho_k^{-1} \gamma_k \mathcal{G}(z^{k+1}, z) & \le  \mathcal{D}_k (z, \tilde{z}^{[k]}) - \gamma_k \langle \tilde{x}^{k+1} - x, B^T(\tilde{y}^{k+1} - \tilde{y}^k) \rangle  \\  
& \quad +  \gamma_k \langle A (\tilde{x}^{k+1} - \tilde{x}^k), \tilde{y}^{k+1}-y\rangle  \\
& \quad +  \tau_k\gamma_k \langle (K+B)^T (\tilde{y}^{k+1} - \tilde{y}^k), (K+A)^T(\tilde{y}^{k+1}-y) \rangle  \\
& \quad - \gamma_k \left(\frac{s - q}{2 \tau_k}-\frac{\rho_k L_f}{2}\right) \|\tilde{x}^{k+1}-\tilde{x}^k\|^2  \\ 
& \quad - \gamma_k \left(\frac{t - r}{2 \sigma_k}-\frac{\|K+A\|_2\|K+B\|_2 \tau_{k-1}}{2} \right)\|\tilde{y}^{k+1}-\tilde{y}^k\|^2 \\
& \quad + \sum_{i=1}^k \Lambda_i(z),
\end{split}
\end{align}
where $\gamma_k$ and $\mathcal{D}(z, \tilde{z}^{[k]})$ are defined in \eqref{eqn:gamma} and \eqref{eqn:Bdef}, respectively, and 
\begin{align}\label{eqn:lambdadef}
\Lambda_i(z) & := -\frac{(1-s) \gamma_i}{2 \tau_i}\|\tilde{x}^{i+1}-z^i\|^2 - \frac{(1-t) \gamma_i}{2 \sigma_i} \|\tilde{y}^{i+1}-y^i\|^2 - \gamma_i \langle \Delta^i, z^{i+1}-z \rangle. 
\end{align}
\end{lemma}
\begin{proof}
Analogous to the proof of Lemma \ref{lem:Qboundgen}, except for that we start with 
\begin{align*}
 \langle - \tilde{u}_{k+1}, \tilde{y}^{k+1}-y \rangle + h^*(\tilde{y}^{k+1}) - h^*(y) 
\le \frac{1}{2\sigma_k} \|y - \tilde{y}^k\|^2 - \frac{1}{2 \sigma_k} \|\tilde{y}^{k+1}-\tilde{y}^k\|^2 - \frac{1}{2 \sigma_k} \|y-\tilde{y}^{k+1}\|^2 
\end{align*}
\begin{align*}
 \langle \hat{\mathcal{F}}(x^k_{md}), \tilde{x}^{k+1}-x \rangle + \langle \tilde{x}^{k+1} - x, \tilde{v}_{k+1} \rangle 
\le \frac{1}{2 \tau_k} \|x - \tilde{x}^k \|^2 - \frac{1}{2 \tau_k} \|\tilde{x}^{k+1}-\tilde{x}^k \|^2 - \frac{1}{2 \tau_k} \|x-\tilde{x}^{k+1}\|^2.
\end{align*}
\end{proof}

Now we define $\Delta_{x, f}^k := \hat{\mathcal{F}}(x^k_{md})-\nabla f(x^k_{md})$, $\Delta_{x, K}^k:= \tilde{v}_{k+1}-\tilde{v}_{k+1, o}$, $\Delta_y^k := - \tilde{u}_{k+1} + \tilde{u}_{k+1, o}$, and $\Delta^k := (\Delta_x^k, \Delta_y^k)$, where $\tilde{u}_{k+1, o}$ and $\tilde{v}_{k+1, o}$ is the result from \eqref{eqn:accgen} calculated with the recent iterates $(\tilde{x}^{k+1}, \tilde{y}^{k+1})$, $(\tilde{x}^k, \tilde{y}^k)$ and $(\tilde{x}^{k-1}, \tilde{y}^{k-1})$ from \eqref{eqn:stocgen}.  

We need the following lemmas.
\begin{lemma}[Lemma 4.5, Chen et al., 2011]\label{lem:delta}
Let $\tau_i$, $\sigma_i$, and $\gamma_i >0$. For any $\tilde{x}^1 \in Z$, define $\tilde{x}^1_v = \tilde{x}^1$ and 
\begin{align}\label{eqn:vdef}
z^{i+1}_v = \argmin_{z =(x,y) \in Z} - \tau_i \langle \Delta_x^i, x \rangle - \sigma_i \langle \Delta_y^i, y \rangle + \frac{1}{2}\|z - z^i_v\|^2,
\end{align}
then 
\begin{align}
\sum_{i=1}^k \gamma_i \langle -\Delta_i, z^i_v - z \rangle \le \mathcal{D}_k (z, \tilde{z}^{[k]}_v) + \sum_{i=1}^k \frac{\tau_i \gamma_i}{2} \|\Delta_x^i\|^2 + \sum_{i=1}^k \frac{\sigma_i \gamma_i}{2}\|\Delta_y^i\|^2,
\end{align}
where $\tilde{z}^{[k]}_v := \{z^i_v\}_{i=1}^k$. 
\end{lemma}

\begin{lemma}\label{lem:deltabound}
The following holds for $\expect[\|\Delta_{x, f}^i\|^2]$, $\expect[\|\Delta_{x, K}^i\|^2]$, and $\expect[\|\Delta_{y}^i\|^2]$. 
\begin{subequations}\label{eqn:deltabound}
\begin{align}
\expect[\|\Delta_{x, f}^i\|^2] &\le \chi_{x, f}^2 \label{eqn:deltabound:xf}\\
\expect[\|\Delta_{x, K}^i\|^2] &\le \chi_{x, K}^2 + \chi_B^2 \label{eqn:deltabound:xK}\\
\expect[\|\Delta_{y}^i\|^2] &\le \chi_{y}^2 +\chi_A^2 +  \tau_i^2 \|K+A\|_2^2 (\chi_x^2 + \chi_B^2). \label{eqn:deltabound:y}\noeqref{eqn:deltabound:y}
\end{align}
\end{subequations}
If $A=-K$ and $B=bK$, after rearranging terms in \eqref{eqn:stocgen}, we have 
\begin{align}\label{eqn:deltabound:simple}
\begin{split}
\expect[\|\Delta_{x, f}^i\|^2] &\le \chi_{x, f}^2 \\
\expect[\|\Delta_{x, K}^i\|^2] &\le \chi_{x, K}^2 \\
\expect[\|\Delta_{y}^i\|^2] &\le \chi_{y}^2.
\end{split}
\end{align}
\end{lemma}
\begin{proof}
\eqref{eqn:deltabound:xf} is trivial, by \eqref{eqn:A1}. Note that
\begin{align*}
\Delta_{x,K}^i = \hat{\mathcal{K}}_y (\tilde{y}^{k+1})-K^T \tilde{y}^{k+1} + \hat{\mathcal{B}}(\tilde{y}^{k+1}-\tilde{y}^k - \theta_k(\tilde{y}^k - \tilde{y}^{k-1})) - B^T(\tilde{y}^{k+1}-\tilde{y}^k - \theta_k(\tilde{y}^k - \tilde{y}^{k-1})), 
\end{align*}
and as separate calls for the stochastic oracle are independent, we obtain \eqref{eqn:deltabound:xK}. 
If we define 
\begin{align*}
\Delta_v^i := \hat{\mathcal{F}} (\tilde{x}^k) - \nabla f(\tilde{x}^k) + \bar{v}_k - \bar{v}_{k,o}, 
\end{align*}
then one may easily check that
\begin{align*}
\expect[\|\Delta_v^i\|^2] \le \chi_{x, f}^2 + \chi_{x, K}^2 + \chi_B^2.
\end{align*}
Then we have:
\begin{align*}
\Delta_y^i & = \hat{\mathcal{K}}_x(\tilde{x}^k - \tau_k (\nabla f(\tilde{x}^k) + \bar{v}_{k,0} + \Delta_v^i)) - \hat{\mathcal{A}}(\theta_k(\tilde{x}^k - \tilde{x}^{k-1}) +\tau_k (\nabla f(\tilde{x}^k) + \bar{v}_{k,0} + \Delta_v^i)) \\
& - K(\tilde{x}^k - \tau_k (\nabla f(\tilde{x}^k) + \bar{v}_{k,0}+ \Delta_v^i)) + A(\theta_k(\tilde{x}^k - \tilde{x}^{k-1}) + \tau_k (\nabla f(\tilde{x}^k) + \bar{v}_{k,0}+ \Delta_v^i)) \\
& -\tau_k (K+A)\Delta_v^i,
\end{align*}
thus 
\begin{align*}
\expect[\|\Delta_{y}^i\|^2] \le \chi_y^2 +\chi_A^2 + \tau_k^2 \|K+A\|_2^2 (\chi_x^2 + \chi_B^2).
\end{align*}

When $A=-K$ and $B=bK$, we may rearrange \eqref{eqn:stocgen} to include only one call to either $\hat{\mathcal{K}}_x$ or $\hat{\mathcal{K}}_y$, as
\begin{align*}
\tilde{u}_{k+1} &= \hat{\mathcal{K}}_x (\tilde{x}^k + \theta_k (\tilde{x}^k - \tilde{x}^{k-1})) \\
\tilde{v}_{k+1} &= \hat{\mathcal{K}}_y (\tilde{y}^{k+1} + b((\tilde{y}^{k+1}-\tilde{y}^k) - \theta_k (\tilde{y}^k - \tilde{y}^{k-1}))).
\end{align*}
Then using the approach similar to above, we may obtain \eqref{eqn:deltabound:simple}.
\end{proof}

\iftrue
\begin{proof}[Proof of Theorem \ref{thm:stocbdd}]
First we use the bound in \eqref{eqn:bddinner} to obtain
\[
\rho_k^{-1} \gamma_k \mathcal{G}(z^{k+1}, z)\le \frac{\gamma_k}{\tau_k} \Omega_X^2 + \frac{\gamma_k}{\sigma_k} \Omega_Y^2 + \sum_{i=1}^k \Lambda_i(z).
\]
Then by the definition of $\Lambda_i(z)$, we have 
\begin{align*}
\Lambda_i(z) &= - \frac{(1-s) \gamma_i}{2 \tau_i} \|\tilde{x}^{i+1}-z^i\|^2 - \frac{(1-t) \gamma_i}{2 \sigma_i} \|\tilde{y}^{i+1}-y^i\|^2 + \gamma_i \langle \Delta^i, z-z^{i+1} \rangle \\
&= - \frac{(1-s) \gamma_i}{2 \tau_i} \|\tilde{x}^{i+1}-z^i\|^2 - \frac{(1-t) \gamma_i}{2 \sigma_i} \|\tilde{y}^{i+1}-y^i\|^2 + \gamma_i \langle \Delta^i, z^i - z^{i+1} \rangle + \gamma_i \langle \Delta^i, z-z^i \rangle \\
& \le  \frac{\tau_i \gamma_i}{2(1-s)} \|\Delta_x^i\|^2 + \frac{\sigma_i \gamma_i}{2(1-t)} \|\Delta_y^i\|^2 + \gamma_i \langle \Delta^i, z-z^i \rangle,
\end{align*}
where the last line is due to Young's inequality. By this result and Lemma \ref{lem:delta}, we have 
\begin{align}\label{eqn:lambdabdd}
\begin{split}
\sum_{i=1}^k \Lambda_i(z) & \le  \sum_{i=1}^k\left[\frac{\tau_i \gamma_i}{2(1-s)} \|\Delta_x^i\|^2 + \frac{\sigma_i \gamma_i}{2(1-t)}\|\Delta_y^i\|^2 + \gamma_i \langle \Delta^i, z^i_v - z^i \rangle + \gamma_i \langle -\Delta^i, z^i_v-z \rangle \right]  \\
& \le  \mathcal{D}_k(z, \tilde{z}^{[k]}_v) + \frac{1}{2}\sum_{i=1}^k\left[\frac{(2-s) \tau_i \gamma_i}{1-s} \|\Delta_x^i\|^2 + \frac{(2-t) \sigma_i \gamma_i}{1-t} \|\Delta_y^i\|^2 + \gamma_i \langle \Delta^i, z^i_v-z^i \rangle \right].
\end{split}
\end{align}
Let us define $U_k$ as
\begin{align}\label{eqn:udef}
U_k := \frac{1}{2}\sum_{i=1}^k\left[\frac{(2-s) \tau_i \gamma_i}{1-s} \|\Delta_x^i\|^2 + \frac{(2-t) \sigma_i \gamma_i}{1-t} \|\Delta_y^i\|^2 + \gamma_i \langle \Delta^i, z^i_v-z^i \rangle \right]
\end{align}
for later use.

Note that $\Delta^i$ and $z^i$ are independent by the assumptions of stochastic oracle. 
By this fact and Lemma \ref{lem:deltabound}, 
\begin{align}\label{eqn:Ubdd}
\expect[U_k] \le \frac{1}{2} \sum_{i=1}^k \left[\frac{(2-s)\tau_i \gamma_i (\chi_x^2 + \chi_B^2)}{1-s} + \frac{(2-t)\sigma_i \gamma_i (\chi_y^2+\chi_A^2 + \tau_k^2 \|K+A\|_2^2(\chi_x^2+ \chi_B^2))}{1-t}\right].\quad \quad
\end{align}

Similar to \eqref{eqn:Bbound}, $\mathcal{D}_k(z, \tilde{z}^{[k]}_v) \le \frac{\Omega_X^2 \gamma_k}{\tau_k}+\frac{\Omega_Y^2 \gamma_k}{\sigma_k}$. Thus we have:
\begin{align*}
\expect[\rho_k^{-1} \gamma_k \mathcal{G}^\star                                   (z^{k+1})] & \le  \frac{2 \gamma_k}{\tau_k} \Omega_X^2 + \frac{2 \gamma_k}{\sigma_k} \Omega_Y^2 + \expect [U_k].
\end{align*}

The above relation along with \eqref{eqn:Ubdd} implies the condition (a).

Proof of part (b) is analogous to the proof of Theorem 3.1 in \citet{chen2014optimal}. This uses a large-deviation theorem for martingale-difference sequence.   
\end{proof}

\begin{proof}[Proof of Corollary \ref{cor:stocbdd}]
First we check \eqref{eqn:stoccond1} and \eqref{eqn:stoccond2}.
\begin{align*}
\frac{s-q}{\tau_k} - \rho_k L_f - \frac{\|A\|_2^2 \sigma_k}{r} \ge \frac{\|K\|_2 \Omega_Y}{\Omega_X} \left((s-q)Q - \frac{1}{r}\right) \ge 0
\end{align*}
\begin{align*}
\frac{t-r}{\sigma_k} -\tau_k\left(2\|K+A\|_2\|K+B\|_2 + \frac{\|B\|_2^2}{p} \right) \ge \left( (t-r) R - \frac{b^2/q}{Q}\right) \frac{\Omega_X}{\Omega_Y} \ge 0,
\end{align*}
by \eqref{eqn:stocbddparamcond}. 
Note that $\gamma_k = t$, $\sum_{i=1}^{N-1} i^2 \le \frac{N^2(N-1)}{3}$, so 
\begin{align*}
\frac{1}{\gamma_{N-1}} \sum_{i=1}^{N-1} \tau_i \gamma_i \le \frac{\Omega_X}{(N-1)^{3/2} N \chi_x} \sum_{i=1}^{N-1} i^2 \le \frac{\Omega_X N}{3  \chi_x \sqrt{N-1}} \\
\frac{1}{\gamma_{N-1}} \sum_{i=1}^{N-1} \sigma_i \gamma_i \le \frac{\Omega_Y}{(N-1)^{3/2} N \chi_y} \sum_{i=1}^{N-1} i^2 \le \frac{\Omega_Y N}{3  \chi_y \sqrt{N-1}}.
\end{align*}
The above implies 
\begin{align*}
\mathcal{C}_0(N-1) &\le \frac{2}{N}[ \frac{2 (2PL_f \Omega_X + Q\|K\|_2 \Omega_Y (N-1) +  \chi_x N \sqrt{N-1}}{\Omega_X (N-1)} \Omega_X^2 \\
&\quad + \frac{2 ( \|K\|_2 \Omega_X (N-1) +  \chi_y N \sqrt{N-1}}{\Omega_Y (N-1)}\Omega_Y^2 \\
&\quad + \frac{(2-r) \Omega_X N \chi_x^2}{(1-r) 6  \chi_x \sqrt{N-1}} + \frac{(2-s) \Omega_Y N \chi_y^2}{(1-s) 6  \chi_y \sqrt{N-1}} \\
& \le  \frac{8 P L_f \Omega_X^2}{N(N-1)} + \frac{4 \|K\|_2 \Omega_X \Omega_Y (Q+1)}{N} + \frac{4 \chi_x \Omega_X + 4 \chi_y \Omega_Y}{\sqrt{N-1}} \\
&\quad + \frac{(2-r) \Omega_X \chi_x}{3 (1-r)  \sqrt{N-1}} + \frac{(2-s) \Omega_Y \chi_y}{3(1-s) \sqrt{N-1}}
\end{align*}
and 
\begin{align*}
\mathcal{C}_1(N-1) & \le  \frac{2}{N(N-1)} \left( \sqrt{2} \chi_x \Omega_X + \chi_y \Omega_Y \right) \sqrt{\frac{2(N-1)N^2}{3}} \\ 
&\quad + \frac{1}{N} \frac{(2-r) \Omega_X N \chi_x}{(1-r) 3 \sqrt{N-1}} + \frac{1}{N} \frac{(2-s) \Omega_Y \chi_y}{(1-s) 3   \sqrt{N-1}} \\
&= \left( \frac{4}{\sqrt{3}} + \frac{2-r}{3(1-r)}\right) \frac{\Omega_X \chi_x}{\sqrt{N-1}} + \left(\frac{2\sqrt{2}}{\sqrt{3}} + \frac{2-s}{3 (1-s)}\right) \frac{\Omega_Y \chi_y}{\sqrt{N-1}}.
\end{align*}
\end{proof}

\fi

\begin{lemma}\label{lem:stocunbdd}
For a saddle point  $\hat{z} = (\hat{x}, \hat{y})$ of \eqref{eqn:saddlepoint}, and the parameters $\rho_k$, $\theta_k$, $\tau_k$, and $\sigma_k$ satisfy \eqref{eqn:rhocondgen}, \eqref{eqn:thetacondunbdd},  and \eqref{eqn:stocparam}, then 
\begin{align}
& \quad (1-q)\|\hat{x} - \tilde{x}^{k+1}\|^2 + \|\hat{x} - \tilde{x}^{k+1}_v\|^2 + \frac{\tau_k(1/2-r)}{\sigma_k} \|\hat{y} - \tilde{y}^{k+1}\|^2+ \frac{\tau_k}{\sigma_k} \|\hat{y} - \tilde{y}^{k+1}_v \|^2 \nonumber \\ 
&\le  2 \|\hat{x} - \tilde{x}^1\|^2 + \frac{2 \tau_k}{\sigma_k} \|\hat{y} - \tilde{y}^1\|^2 + \frac{2 \tau_k}{\sigma_k} U_k, \label{eqn:diffbdd2}
\end{align}
where $(\tilde{x}^{k+1}_v, \tilde{y}^{k+1}_v)$ is defined in \eqref{eqn:vdef}, and $U_k$ is defined by \eqref{eqn:udef}.

Furthermore, 
\begin{align}\label{eqn:lem7b}
\tilde{\mathcal{G}}(z^{k+1}, v^{k+1}) \le \frac{\rho_k}{ \tau_k} \|x^{k+1} - \tilde{x}^1\|^2 +\frac{\rho_k}{ \sigma_k} \|y^{k+1}-\tilde{y}^1\|^2 + \frac{\rho_k}{ \gamma_k} U_k := \delta_{k+1}, 
\end{align}
for $k \ge 1$, where 
\begin{align*}
v_{k+1} &= \rho_k ( \frac{1}{\tau_k}(2 \tilde{x}^1 - \tilde{x}^{k+1} - \tilde{x}^{k+1}_v) - B^T (\tilde{y}^{k+1}-\tilde{y}^k), \\
&\quad \frac{1}{\sigma_k} (2\tilde{y}^1 - \tilde{y}^{k+1}-\tilde{y}^{k+1}_v) + A(\tilde{x}^{k+1} - \tilde{x}^k)+  \tau_k (K+A)(K+B)^T (\tilde{y}^{k+1}-\tilde{y}^k))
\end{align*}
\end{lemma}
\begin{proof}
By applying the bound \eqref{eqn:unbddinner} and \eqref{eqn:lambdabdd} to \eqref{eqn:stocQbound}, we obtain:
\begin{align*}
\rho_k^{-1} \gamma_k \mathcal{G}(z^{k+1}, z) & \le \bar{\mathcal{D}}_k(z, \tilde{z}^{[k]}) + \frac{q \gamma_k}{2\tau_k}\|x-\tilde{x}^{k+1}\|^2 + \frac{(r+1/2) \gamma_k}{2 \sigma_k}\|y - \tilde{y}^{k+1}\|^2 + \bar{\mathcal{D}}_k(z, \tilde{z}^{[k]}_v) + U_k,
\end{align*}
where 
\begin{align*}
\bar{\mathcal{D}}_k(z, \tilde{z}^{[k]}) & = \frac{\gamma_k}{2 \tau_k}(\|x-\tilde{x}_1\|^2 - \|x - \tilde{x}_{k+1}\|^2 + \frac{\gamma_k}{2 \sigma_k}(\|y - \tilde{y}_1\|^2 - \|y - \tilde{y}_{k+1}\|^2). 
\end{align*}
Letting $z= \hat{z}$ and using $\mathcal{G}(z^{k+1}, \hat{z}) \ge 0$ leads to \eqref{eqn:diffbdd2}.
If we only use \eqref{eqn:lambdabdd} on \eqref{eqn:stocQbound}, we get:
\begin{align*}
\rho_k^{-1} \gamma_k \mathcal{G}(z^{k+1}, z) & \le \bar{\mathcal{D}}_k(z, \tilde{z}^{[k]}) - \gamma_k \langle \tilde{x}^{k+1} - x, B^T (\tilde{y}^{k+1}-\tilde{y}^k) \rangle \\
& + \gamma_k \langle A(\tilde{x}^{k+1}-\tilde{x}^k), \tilde{y}^{k+1}-y \rangle \\
& + \tau_k \gamma_k \langle (K+B)^T (\tilde{y}^{k+1}-\tilde{y}^k), (K+A)^T (\tilde{y}^{k+1}-y) \rangle + \bar{\mathcal{D}}_k(z, \tilde{z}_v^{[k]}) + U_k.
\end{align*}
Applying \eqref{eqn:normid2} and following the steps of Lemma \ref{lem:unbddgen} results in \eqref{eqn:lem7b}. 
\end{proof}

\begin{proof}[Proof of Theorem \ref{thm:stocunbdd}]
Note that \eqref{eqn:Ubdd} holds by Lemma \ref{lem:deltabound}.
By the definition of $S$ in \eqref{eqn:Sdef} and \eqref{eqn:Ubdd}, we have 
\[
\expect [U_k] \le \frac{\gamma_k}{2 \tau_k} S^2. 
\]
By the above, \eqref{eqn:diffbdd2}, and \eqref{eqn:Rdefgen}, we have 
\[
\expect [\|\hat{x} - \tilde{x}^{k+1}\|^2] \le \frac{2R^2 + S^2}{1-q} \text{ and } \expect[\|\hat{y} - \tilde{y}^{k+1}\|^2] \le \frac{(2R^2+S^2) \sigma_1}{\tau_1 (1/2-r)}.
\]
By Jensen's inequality, this leads to 
\[
\expect [\|\hat{x} - \tilde{x}^{k+1}\|] \le \sqrt{\frac{2R^2 + S^2}{1-q}} \text{ and } \expect[\|\hat{y} - \tilde{y}^{k+1}\|] \le \sqrt{\frac{(2R^2+S^2) \sigma_1}{\tau_1 (1/2-r)}}.
\]
Similarly, we have 
\[
\expect [\|\hat{x} - \tilde{x}^{k+1}_v\|] \le \sqrt{2R^2 + S^2} \text{ and } \expect[\|\hat{y} - \tilde{y}^{k+1}_v\|] \le \sqrt{\frac{(2R^2+S^2) \sigma_1}{\tau_1}}.
\]
Thus 
\begin{align*}
\expect[\|v^{k+1}\|] & \le  \rho_k \expect[\frac{1}{\tau_k} (2 \|\hat{x}-\tilde{x}^1\| + \|\hat{x}-\tilde{x}^{k+1}\| + \|\hat{x} - \tilde{x}^{k+1}_v \|)\\
&\quad + \frac{1}{\sigma_k} (2 \|\hat{y} - \tilde{y}^1\| + \|\hat{y} - \tilde{y}^{k+1}\| + \|\hat{y} - \tilde{y}^{k+1}_v\|) \\
&\quad + \|A\|_2 (\|\hat{x} - \tilde{x}^{k+1}\| + \|\hat{x} - \tilde{x}^k\|) \\
&\quad + (\|B\|_2+ \|K+A\|_2\|K+B\|_2 \tau_k) (\|\hat{y} - \tilde{y}^{k+1}\| + \|\hat{y} - \tilde{y}^k\|)] \\
&\le  \frac{2 \rho_k \|\hat{x} - \tilde{x}^1\|}{ \tau_k} + \frac{2 \rho_k \|\hat{y} - \tilde{y}^1\|}{ \sigma_k}\\ 
&\quad+ \sqrt{2R^2 + S^2} [\frac{\rho_k}{ \tau_k} (1 + \mu') + \frac{\rho_k}{ \sigma_k}\sqrt{\frac{\sigma_1}{\tau_1}} (1 + \nu') \\
&\quad + \rho_k (2 \|A\|_2 \mu'+ 2\|B\|_2\nu' \sqrt{\frac{\sigma_1}{\tau_1}}) \\
&\quad + 2 \rho_k \tau_k \|K+A\|_2\|K+B\|_2 \nu' \sqrt{\frac{\sigma_1}{\tau_1}}]
\end{align*}

where $\mu'^{-1} = \sqrt{1-q}$ and $\nu'^{-1} = \sqrt{1/2-r}$.  Now we find an upper bound of $\expect [\delta_{k+1}]$. 
\begin{align*}
\expect[\delta_{k+1}] & \le  \expect[ \frac{2 \rho_k}{ \tau_k} (\|\hat{x} - x^{k+1}\|^2 + \|\hat{x} - \tilde{x}^1\|^2) + \frac{2 \rho_k}{ \sigma_k} (\|\hat{y} - y^{k+1}\|^2 + \|\hat{y} - \tilde{y}^1\|^2) + \frac{\rho_k}{2 \tau_k}S^2] \\
&= \expect [ \frac{\rho_k}{ \tau_k}(2R^2 + 2(1-q) \|\hat{x} - x^{k+1}\|^2 + \frac{2 \tau_k (r-1/2)}{\sigma_k}\|\hat{y} - y^{k+1}\|^2) \\
&\quad+ 2 q \|\hat{x} - x^{k+1}\|^2 + \frac{2 \tau_k (r+1/2)}{\sigma_k} \|\hat{y} - y^{k+1}\|^2 + \frac{\rho_k}{2  \tau_k}S^2 \\
& \le \frac{\rho_k}{ \tau_k} 2 R^2 +\frac{2 \rho_k}{ \tau_k}\sum_{i=1}^k \gamma_i[(2R^2 + S^2) + q \mu'^2 (2R^2+S^2) + (r+1/2)\nu'^2 (2R^2 + S^2)]+ \frac{S^2}{2}] \\
&= \frac{\rho_k}{ \tau_k} [6R^2 + \frac{5}{2}S^2 +\frac{2q}{1-q} (2R^2 + S^2) +\frac{2(r+1/2)}{1/2-r} (2R^2 + S^2)] \\
&= \frac{\rho_k}{ \tau_k} \left[\left(6 + \frac{4q}{1-q}+\frac{4(r+1/2)}{1/2-r}\right) R^2 + \left(\frac{5}{2} + \frac{2q}{1-q}+\frac{2(r+1/2)}{1/2-r}\right) S^2\right]
\end{align*}
\end{proof}

\begin{proof}[Proof of Corollary \ref{cor:stocunbdd}]
First we check \eqref{eqn:stoccond1} and \eqref{eqn:stoccond2}.
\begin{align*}
\frac{s-q}{\tau_k} - \rho_k L_f - \frac{\|A\|_2^2 \sigma_k}{r} 
& \ge  \|K\|_2 \left( (s-q) Q - \frac{1}{r} \right) \ge 0, 
\end{align*}
\begin{align*}
\frac{t-r}{\sigma_k} - \tau_k\frac{b^2\|K\|_2^2}{q} 
& \ge  \|K\|_2 \left( Q(t-r) - \frac{b^2}{q} \right) \ge 0,
\end{align*}
by \eqref{eqn:stocunbddparamcond}. 

Now note that 
\begin{align}\label{eqn:cbound}
\begin{split}
S &= \sqrt{\sum_{i=1}^{N-1} \frac{(2-s)\chi_x^2 i^2}{(1-s) \tau^2} + \sum_{i=1}^{N-1} \frac{(2-t) \chi_y^2 i^2}{(1-t) \tau^2}} \\
& \le  \sqrt{\frac{N^2(N-1)}{3 \tau^2} \left(\frac{(2-s) \chi_x^2}{1-s} + \frac{(2-t)\chi_y^2}{1-t}\right)} = \frac{\chi N \sqrt{N-1}}{\sqrt{3}\tau} \\
& \le  \frac{\chi N \sqrt{N-1}}{\sqrt{3}  N \sqrt{N-1} \chi/\tilde{R}} = \frac{\tilde{R}}{\sqrt{3}}.
\end{split}
\end{align}

Thus $\epsilon_N$ is bounded above by 
\begin{align*}
\epsilon_N &\le \frac{\rho_{N-1}}{ \tau_{N-1}} (\zeta R^2 + \xi S^2)  \le  \frac{\rho_{N-1}}{\tau_{N-1}} (\zeta R^2 + \xi \frac{\tilde{R}^2}{3} ),
\end{align*}
where $\zeta = 6 + \frac{4q}{1-q} + \frac{4(r+1/2)}{1/2-r}$ and $\xi = \frac{5}{2} + \frac{2q}{1-q} + \frac{2(r+1/2)}{1/2-r}$. 

Note that 
\begin{align*}
\frac{\rho_{N-1}}{ \tau_{N-1}} \|\hat{x} - \tilde{x}^1\| & \le \frac{\rho_{N-1}}{\tau_{N-1}}R, \;\; \frac{\rho_{N-1}}{ \sigma_{N-1}} \|\hat{y} - \tilde{y}^1\| \le \frac{\rho_{N-1}}{\tau_{N-1}}R
\end{align*}
and that 
\begin{align}\label{eqn:stocparambound}
\begin{split}
\rho_{N-1} \|K\|_2 &\le \frac{2 \|K\|_2}{N}, \\
\frac{\rho_{N-1}}{\tau_{N-1}} & \le \frac{2 \tau}{N(N-1)} = \frac{4P L_f + 2 Q \|K\|_2 (N-1) + 2 N \sqrt{N-1} \chi/\tilde{R}}{N(N-1)} \\
& = \frac{4PL_f}{N(N-1)} + \frac{2Q\|K\|_2}{N} + \frac{2 \chi/\tilde{R}}{\sqrt{N-1}}
\end{split}
\end{align}
Thus 
\begin{align*}
\epsilon_N & \le \frac{\rho_{N-1}}{\tau_{N-1}} (\zeta R^2 + \xi S^2) \\
& \le \left(\frac{4PL_f}{N(N-1)} + \frac{2 Q \|K\|_2}{N}+ \frac{2 \chi /\tilde{R}}{\sqrt{N-1}} \right) \left(\zeta R^2 + \frac{\xi \tilde{R}^2}{3}\right). 
\end{align*}
Now note that $\sqrt{2R^2 + S^2} \le \sqrt{2}R + S$. 
\begin{align*}
\expect[\|v^N\|]& \le \frac{2 \rho_{N-1}}{\tau_{N-1}} 2R + (\sqrt{2}R + S) \left[\frac{\rho_{N-1}}{\tau_{N-1}} (2+\mu' +\nu') + \rho_{N-1} \|K\|_2 (2 \mu' + 2 b \nu')\right] \\
&= \frac{\rho_{N-1}}{\tau_{N-1}} \left( 4R + (\sqrt{2}R + S) (2 + \mu' + \nu')\right) \\
& \quad + \rho_{N-1} \|K\|_2 (\sqrt{2}R +S) (2 \mu' + 2 b \nu') \\
& \le \left(\frac{4PL_f}{N(N-1)} + \frac{2Q\|K\|_2}{N} + \frac{2 \chi/\tilde{R}}{\sqrt{N-1}} \right) \left(4R + \left(\sqrt{2}R + \frac{\tilde{R}}{\sqrt{3}}\right) (2 + \mu' + \nu')\right)\\ 
& \quad + \frac{2\|K\|_2}{N} (\sqrt{2}R + \tilde{R}/\sqrt{3}) (2 \mu' + 2b \nu')
\end{align*}
we obtain the desired order for both $\epsilon_N$ and $\expect[\|v_N\|]$. 
\end{proof}